\pdfoutput=1

\documentclass[dvipsnames]{article}

\PassOptionsToPackage{numbers, compress}{natbib}

\usepackage{amsmath}
\usepackage{mathtools}
\usepackage[ruled, norelsize]{algorithm2e}
\usepackage{amsthm}
\usepackage{amssymb}
\usepackage{esint}
\usepackage{nicefrac}
\usepackage{arydshln}
\usepackage{bm}

\usepackage{booktabs}
\usepackage{microtype}
\usepackage{graphicx}       
\usepackage{fancyhdr}   
\usepackage{verbatim}
\usepackage[]{hyperref}
\usepackage{url}
\usepackage{varwidth}
\usepackage{graphicx}
\usepackage{epstopdf}
\usepackage{float}
\usepackage{tikz}
\usepackage{setspace}
\usepackage{subcaption}
\usepackage{cleveref}
\usepackage{soul}      
\usepackage{etoolbox}                        
\usepackage{dsfont}

\usepackage{natbib}

\Crefname{algocf}{Algorithm}{Algorithms}

\usepackage[final]{neurips_2020}

\def \SVD{\operatorname{SVD}}
\def \PCA{\operatorname{PCA}}

\def \RSPCA{\operatorname{SSVD}}

\def \FPCA{\operatorname{Federated-PCA}}

\def \energy{\operatorname{\mathcal{E}}}
\def \rankupdate{\operatorname{Rank}}
\def \rankupdate{\operatorname{AdjustRank}}
\def \QR{\operatorname{QR}}

\def \FPCAC{\operatorname{FPCA}}
\def \FPCAEC{\operatorname{FPCA-Edge}}

\def \BasicMerge{\operatorname{BasicMerge}}
\def \FasterMerge{\operatorname{FasterMerge}}
\def \Merge{\operatorname{Merge}}

\def \SU {\mathcal{S}}
\def \A {\mathcVal{A}}

\def \modsulq {\operatorname{MOD-SuLQ}}

\def \sulq {\operatorname{SuLQ}}

\def \Q {\mathbf{Q}}
\def \A {\mathbf{A}}
\def \BLK {\mathbf{B}}

\def \E {\mathbb{E}}
\def \h {\widehat}
\def \U {\mathcal{U}}

\def \Y {\mathbf{Y}}
\def \R {\mathbb{R}}
\def \GR {\mathbb{G}}

\def \X {\mathbf{X}}
\def \y{\mathbf{y}}
\def \P{\mathbf{P}}

\def \G {\mathbf{G}}
\def \Ghat {\hat{\mathbf{G}}}
\def \H {\mathbf{H}}

\newcommand{\B}[1]{\mathbf{#1}}

\def \P{\mathbf{P}}
\def \rate{\eta}

\def \rank {\operatorname{rank}}
 
\def \l{\left}
\def \r{\right}
\def \G{\B{G}}

\def \N{\mathbb{N}}

\def \wh{\widehat}

\DeclareMathOperator{\ind}{\mathds{1}}

\makeatletter
\newcommand{\removelatexerror}{\let\@latex@error\@gobble}
\makeatother

\newtheorem{thm}{Theorem}
\newtheorem{lemma}{Lemma}

\newtoggle{ONE_COLUMN}
\toggletrue{ONE_COLUMN}

\newtoggle{WITH_SOURCE}
\toggletrue{WITH_SOURCE}

\title{Federated Principal Component Analysis}

\author{{\bf Andreas Grammenos\textsuperscript{1,3}\thanks{Correspondence to: Andreas Grammenos $<$\href{mailto://ag926@cl.cam.ac.uk}{ag926@cl.cam.ac.uk}$>$}\hskip1em  
  Rodrigo Mendoza-Smith\textsuperscript{2}\hskip.5em
  Jon Crowcroft\textsuperscript{1,3}}\hskip.5em
  Cecilia Mascolo\textsuperscript{1}\\  
  \\
  \textsuperscript{1}Computer Lab, University of Cambridge\\
  \textsuperscript{2}Quine Technologies\\
  \textsuperscript{3}Alan Turing Institute
}

\begin{document}

\maketitle

\begin{abstract}
We present a federated, asynchronous, and $(\varepsilon, \delta)$-differentially private algorithm for $\PCA$ in the memory-limited setting.
Our algorithm incrementally computes local model updates using a streaming procedure and adaptively estimates its $r$ leading principal components when only $\mathcal{O}(dr)$ memory is available with $d$ being the dimensionality of the data.
We guarantee differential privacy via an  input-perturbation scheme in which  the covariance matrix of a dataset $\B{X} \in \R^{d \times n}$ is perturbed with a non-symmetric random Gaussian matrix with variance in $\mathcal{O}\left(\left(\frac{d}{n}\right)^2 \log d \right)$, thus improving upon the state-of-the-art.
Furthermore, contrary to previous federated or distributed algorithms for $\PCA$, our algorithm is also invariant to permutations in the incoming data, which provides robustness against straggler or failed nodes.  
Numerical simulations show that, while using limited-memory, our algorithm exhibits performance that closely matches or outperforms traditional non-federated algorithms, and in the absence of communication latency, it exhibits attractive horizontal scalability.
\end{abstract}

\section{Introduction}
\label{label:intro}

In recent years, the advent of edge computing in smartphones, IoT and cryptocurrencies has induced a paradigm shift in distributed model training and large-scale data analysis.
Under this new paradigm, data is generated by commodity devices with hardware limitations and severe restrictions on data-sharing and communication, which makes the centralisation of the data extremely difficult.
This has brought new computational challenges since algorithms do not only have to deal with the sheer volume of data generated by networks of devices, but also leverage the algorithm's voracity, accuracy, and complexity with constraints on hardware capacity, data access, and device-device communication.
Moreover, concerns regarding data ownership and privacy have been growing in applications where sensitive datasets are crowd-sourced and then aggregated by {\em trusted} central parties to train machine learning models. In such situations, mathematical and computational frameworks to ensure data ownership and guarantee that trained models will not expose private client information are highly desirable.
In light of this, the necessity of
being able to {\em privately} analyse large-scale decentralised datasets and extract useful insights out of 
them is becoming more prevalent than ever before.
A number of frameworks have been put forward to train machine-learning models while preserving data ownership and privacy like Federated Learning \citep{mcmahan2016communication, konevcny2016federated}, Multi-party computation \citep{mohassel2017secureml, liu2017oblivious, rouhani2018deepsecure}, Homomorphic encryption \citep{gilad2016cryptonets}, and Differential Privacy \citep{dwork2006calibrating, dwork2014algorithmic}.
In this work we pursue a combined {\em federated learning and differential privacy} framework to compute $\PCA$ in a decentralised way and provide precise guarantees on the privacy budget.
Seminal work in federated learning has been made, but mainly in the context of deep neural networks, see \citep{mcmahan2016communication, konevcny2016federated}.
Specifically, in~\citep{konevcny2016federated} a {\em federated} method for training of neural networks was proposed.
In this setting one assumes that each of a large number of independent {\em clients} can contribute to the training of a centralised 
model by computing local updates with their own data and sending them to
the client holding the centralised model for aggregation. 
Ever since the publication of this seminal work, 
interest in federated algorithms 
for training neural networks has surged, see~\citep{smith2017federated, he2018cola, geyer2017differentially}.
Despite of this, federated adaptations of classical data analysis 
techniques are still largely missing.
Out of the many techniques available, Principal Component Analysis 
($\PCA$)~\citep{pearson1901liii, jolliffe2011principal} is arguably 
the most ubiquitous one for discovering linear structure or reducing dimensionality in data, so has become an essential component in inference, machine-learning, and data-science pipelines. 
In a nutshell, given a matrix 
$\Y \in \R^{d \times n}$ of $n$ feature vectors of dimension $d$, $\PCA$ aims to build a 
low-dimensional subspace of $\R^{d}$ that captures
the directions of maximum variance in the data contained in $\Y$. 
Apart from being a fundamental tool for data analysis, $\PCA$ is often used to reduce the dimensionality of the data
in order to minimise the cost of computationally expensive operations. For instance, before applying t-SNE~\citep{maaten2008visualizing} or 
UMAP~\citep{mcinnes2018umap}. 
Hence, a federated algorithm for $\PCA$ is not only desired when data-ownership is sought to be preserved, but also from a computational viewpoint.

Herein, we propose a federated {and differentially private} algorithm for PCA  (Alg.~\ref{algorithm:federated_pca}). The computation of $\PCA$ is related to the Singular Value Decomposition ($\SVD$)~\citep{eckart,mirsky} which can decompose any matrix into a linear combination of orthonormal rank-1 matrices weighted by positive scalars.
In the context of high-dimensional data, the main  limitation stems from the fact that, in the absence of structure, performing $\PCA$ on a matrix $\Y \in \R^{d \times n}$ requires $\mathcal{O}(d^2n+ d^3)$ computation time and $\mathcal{O}(d^2)$ memory. 
This cubic computational complexity and quadratic storage dependency on $d$ makes 
the cost of $\PCA$ computation prohibitive for high-dimensional data, 
though it can often be circumvented when the data is sparse or has other type of exploitable structure.
Moreover, in some decentralised applications, the computation has to be done in commodity devices with $\mathcal{O}(d)$ storage capabilities, 
so a $\PCA$ algorithm with $\mathcal{O}(d)$ memory dependency is highly desirable.
On this front, there have 
been numerous recent works in the streaming setting
that try to tackle this problem, see~\citep{mitliagkas2014streaming,mitliagkas2013memory,marinov2018streaming,arora2012stochastic,arora2016stochastic,boutsidis2015online}.
However, most of these methods
do not naturally scale well nor can they be parallelised efficiently despite 
their widespread use, e.g.~\citep{bouwmans2014robust,boutsidis2015online}.
To overcome 
these issues a reliable and federated scheme for large decentralised datasets is highly desirable.
Distributed algorithms for $\PCA$ have been studied previously in \citep{kannan2014principal, liang2014improved, qu2002principal}.
Similar to this line of work in~\cite{narayanamurthy2020federated} proposed a federated subspace tracking algorithm in the presence of missing values.
However, the focus in this line of work is in obtaining high-quality guarantees in communication complexity and approximation accuracy and do not implement differential privacy.
A number of papers in non-distributed, but differentially private algorithms for $\PCA$ have been proposed.
These can be roughly divided in two main groups: (i) those which are {\em model free} and provide guarantees for unstructured data matrices, (ii) those that are specifically tailored for instances where specific structure is assumed.
In the model-free $\PCA$ we have (SuLQ) \citep{blum2005practical}, (PPCA) and ($\modsulq$) \citep{chaudhuri2012near}, Analyze Gauss \citep{dwork2014analyze}. 
In the structured case, \citep{hardt2012beating, hardt2013beyond, hardt2014noisy} studies approaches under the assumption of high-dimensional data, \citep{zhou2009differential} considers the case of achieving differential privacy by compressing the database with a random affine transformation, while \citep{ge2018minimax} proposes a distributed privacy-preserving version for sparse $\PCA$, but with a strong sparsity assumption in the underlying subspaces.
To the best of our knowledge, the combined federated, model free, and differential private setting for $\PCA$ has not been previously addressed in literature. This is not surprising as this case is especially difficult to address. In the one hand, distributed algorithms for computing principal directions are not generally {\em time-independent}. That is, the principal components are not invariant to permutations the data.
On the other hand, guaranteeing $(\varepsilon,\delta)$-differential privacy imposes an $\mathcal{O}(d^2)$ overhead in storage complexity, which might render the distributed procedure infeasible in limited-memory scenarios.

\textbf{Summary of contributions}:
Our main contribution is {\em Federated-PCA} (Alg.~\ref{algorithm:federated_pca}) a federated, asynchronous, and $(\varepsilon, \delta)$-differentially private algorithm for $\PCA$.
Our algorithm is comprised out of two independent 
components: (1) An algorithm for the incremental, private, and decentralised computation 
of local updates to $\PCA$, (2) a low-complexity merging procedure to privately aggregate these incremental updates together.
By design Federated-PCA is only allowed to do \textit{one pass} through each column of the dataset $\Y \in \R^{d \times n}$ using an $\mathcal{O}(d)$-memory device which results in a $\mathcal{O}(dr)$ storage complexity.
Federated-PCA achieves $(\varepsilon, \delta)$-differential privacy by extending the symmetric input-perturbation scheme put forward in \cite{chaudhuri2012near} to the non-symmetric case.
In doing so, we improve the noise-variance complexity with respect to the state-of-the-art for non-symmetric matrices.

\section{Notation \& Preliminaries}
\label{label:prelim_notation}

This section introduces the notational conventions used throughout the paper. 
We use lowercase letters $y$ for scalars, bold lowercase letters $\mathbf{y}$ for vectors, bold capitals $\Y$ for matrices, and calligraphic capitals $\mathcal{Y}$ for subspaces.
If $\Y \in \R^{d \times n}$ and $S \subset \{1, \dots,  m\}$, then $\Y_S$ is the block composed of columns indexed by $S$.
We reserve $\B{0}_{m\times n}$ for the zero matrix in $\R^{m \times n}$ and $\B{I}_{n}$ for the identity matrix in $\R^{n \times n}$.
Additionally, we use $\| \cdot \|_{\rm F}$ to denote the Frobenius norm operator and $\| \cdot \|$ to denote the $\ell_{2}$ norm.
If $\Y \in \R^{d \times n}$ we let $\mathbf{Y}=\mathbf{U}\mathbf{\Sigma} \mathbf{V}^T$ be its full $\SVD$ formed from unitary $\mathbf{U}\in \R^{d \times d}$ and $\mathbf{V} \in \R^{n\times n}$ and diagonal  $\mathbf{\Sigma} \in \R^{d \times n}$.
The values $\mathbf{\Sigma}_{i,i} = \sigma_{i}(\Y) \geq 0$ are the singular values of $\Y$.
If $1 \leq r\leq \min(d,n)$, we let 
$[\mathbf{U}_r, \mathbf{\Sigma}_r, \mathbf{V}_r^T] = \SVD_r(\mathbf{Y})$ be the singular value decomposition of its {\em best rank-$r$ approximation}. That is, the solution of $ \min \{\|\mathbf{Z}-\Y\|_F :  \rank{(\mathbf{Z})} \le r\}$.
 Using this notation, we define $[\mathbf{U}_r, \mathbf{\Sigma}_r]$ be the rank-$r$ {\em principal subspace} of $\Y$.
When there is no risk of confusion, we will abuse notation and use $\SVD_r(\mathbf{Y})$ to denote the rank-$r$ left principal subspace with the $r$ leading singular values $[\mathbf{U}_r, \mathbf{\Sigma}_r]$
We also let $\lambda_{1}({\Y})\ge \cdots \ge \lambda_{k}(\Y)$ be the eigenvalues of $\Y$ when $d=n$.
Finally, we let $\vec{\mathbf{e}}_{k} \in \R^d$ be the $k$-th canonical vector in $\R^d$.

\textbf{Streaming Model:}  A data stream is a vector sequence $\mathbf{y}_{t_0}, \mathbf{y}_{t_1}, \mathbf{y}_{t_2}, \dots $ such that
$t_{i+1} > t_{i}$ for all $i \in \N$.
We assume that $\mathbf{y}_{t_j} \in \R^d$ and $t_j \in \mathbb{N}$ for all $j$.
At time $n$, the data stream $\y_1, \dots, \y_n$ can be arranged in a matrix $\Y \in \R^{d \times n}$.
Streaming models assume that, at each timestep, algorithms observe sub-sequences $\y_{t_1}, \dots, \y_{t_b}$ of the data rather than the full dataset $\Y$.

\textbf{Federated learning:}  
Federated Learning~\citep{konevcny2016federated} is a machine-learning paradigm that 
considers how a large number of {\em clients} 
owning different data-points can contribute to the training of a {\em centralised model} by 
locally computing updates with their own data and merging them to the centralised model without sharing data between each other.
Our method resembles the 
{\em distributed agglomerative summary model} (DASM)~\citep{tanenbaum2007distributed}
in which updates are aggregated
in a ``bottom-up'' approach following a tree-structure.
That is, by
arranging the nodes in a
tree-like hierarchy
such that, for any sub-tree,
the leaves 
compute and propagate intermediate results the their roots for merging or summarisation.

\textbf{Differential-Privacy:} Differential privacy~\citep{dwork2014algorithmic} is a mathematical framework that measures to what extent the parameters or predictions of a trained machine learning model reveal information about any individual points in the training dataset.
 Formally, we say that a randomised algorithm $\mathcal{A}(\cdot)$ taking values in a set $\mathcal{T}$ provides $(\varepsilon, \delta)$-differential privacy if
 \begin{equation}
 \mathbb{P}\left[\mathcal{A}(\mathcal{D}) \in \mathcal{S}\right] \leq e^{\varepsilon} \mathbb{P}\left[\mathcal{A}(\mathcal{D}') \in \mathcal{S}\right] + \delta
 \end{equation}
 for all measurable $\mathcal{S} \subset \mathcal{T}$ and all datasets $\mathcal{D}$ and $\mathcal{D}'$ differing in a single entry.
Our algorithm extends $\modsulq$~\citep{chaudhuri2013near} to the streaming and {\em non-symmetric} setting and guarantees $(\varepsilon, \delta)$-differential privacy. Our extension 
only requires {\em one pass} over the data and preserves the nearly-optimal variance rate $\modsulq$.

\section{Federated PCA}
\label{label:fed_pca}

We consider a decentralised dataset $\mathcal{D} = \{\y_{1}, \dots, \y_{n}\} \subset \R^d$ distributed across $M$ clients.
The dataset $\mathcal{D}$ can be stored in a matrix
$\Y=\left[\Y^{1}|\Y^{2}|\cdots|\Y^{M}\right] \in \R^{d \times n}$ with $n \gg d$ and such that $\Y^{i} \in \R^{d \times n_i}$ is {\em owned} by client $i \in \{1, \dots, M\}$.
We assume that each $\Y^i$ is generated in a streaming fashion and that due to resource limitations it cannot be stored in full.
Furthermore, under the DASM we assume that the $M$ clients in the network can be arranged in a tree-like structure with $q > 1$ levels and approximately $\ell > 1$ leaves per node. Without loss of generality, in this paper we assume that $M = \ell^q$.
An example of such tree-like structure is given in~\autoref{fig:computation_graph}.
We note that such structure can be generated easily and efficiently using various schemes~\citep{wohwe2019optimized}.
Our procedure is presented in Alg.~\ref{algorithm:federated_pca}.

\begin{algorithm}[htb!]
\footnotesize{
\DontPrintSemicolon
\SetNoFillComment

\SetKwProg{Fn}{Function}{ is}{end}
\SetKwProg{Each}{Each client}{ :}{end}
\SetKwProg{At}{At time}{ , each client $i \in \{1, \dots, M\}$ }{end}
    \KwData{
    $\Y=\left[\Y^{1}|\cdots|\Y^{M}\right] \in \R^{d \times n}$: {\em Data for network with $M$ nodes} // 
    $(\varepsilon, \delta)$: {\em DP parameters} //
    $(\alpha, \beta)$: {\em Bounds on energy, see \eqref{eq:sv_energy} } // 
    $\B{B}$: {\em Batch size for clients} //
    $r$: {\em Initial rank} ;\;
    }
    \KwResult{$[\B{U}', \B{\Sigma}'] \approx \SVD_r(\Y) \in \R^{d \times r} \times \R^{r \times r}$}
    
    \Fn{$\FPCA_{\varepsilon,  \delta, \alpha, \beta, r}(\Y, B)$}{
    
    Compute $T_{\varepsilon, \delta, d, n}$ minimum batch size to ensure differential privacy, see Lemma \ref{lemma:diffprivacy}\;
    \Each(\tcp*[h]{1. Initialise clients}){ $i\in [M]$}{
        Initialises PC estimate to $(\B{U}^i, \B{\Sigma}^i) \leftarrow (0,0)$, batch $\B{B}^i \leftarrow \left[\;\right]$, and batch size $b^i \leftarrow T_{\varepsilon, \delta, d, n}$\;
    } 
    \At(\tcp*[h]{2. Computation of local updates}){ $t \in \{1, \dots, n\}$}{
        Observes data-point $\y^i_{t} \in \R^d$ and add it to batch $\B{B}^i \leftarrow [\B{B}^i, \y^i_{t}]$\;
        \If{$\B{B}^i$ has $b^i$ columns}{
            $(\B{U}^i, \B{\Sigma}^i) \leftarrow \FPCAEC_{\varepsilon, \delta, \alpha, \beta, r}(\B{B}^i, \B{U}^i, \B{\Sigma}^i)$\; Reset the batch $\B{B}^i \leftarrow \left[\;\right]$, and set the batch size $b^i \leftarrow B$\;
        }
    }
    \tcc{3. Recursive subspace merge}
    Arrange clients' subspaces in a tree-like data structure and merge them recursively with Alg.~\ref{algorithm:fastest_subspace} (Fig.~\ref{fig:computation_graph})\;
    }
    }
    \caption{Federated PCA ($\FPCAC$)}
    \label{algorithm:federated_pca}
\end{algorithm}

Note that Alg.~\ref{algorithm:federated_pca}, invokes $\FPCAEC$ (Alg.~\ref{alg:fpca_edge}) to privately compute local updates to the centralised model and Alg.~\ref{algorithm:fastest_subspace} to recursively merge the local subspaces in the tree.
To simplify the exposition we assume, without loss of generality, that every client $i \in [T]$ observes a vector $\y_t^i \in \R^d$ at time $t \in [T]$, but remark that this uniformity in data sampling need not hold in the general case.
We also assume that clients accumulate observations in {\em batches} and that these are not merged until their size grows to $b^i$.
However, we point out that in real-world device networks the batch size might vary from client to client due to heterogeneity in storage capacity and could indeed be merged earlier in the process.
Finally, it is important to note that the network does not need to wait for all clients to compute a global estimation, so that subspace merging can be initiated when a new local estimation has been computed without perturbing the global estimation.
This {\em time independence} property enables {\em federation} as it guarantees that the principal-component estimations after merging are invariant to permutations in the data, see Lemma~\ref{lemma:fpca-time-independence}.
Merge and $\FPCAEC$ are described in Algs.~\ref{algorithm:fastest_subspace} and \ref{alg:fpca_edge}.
\begin{figure}[htb!]
    \centering
    \includegraphics[scale=0.5]{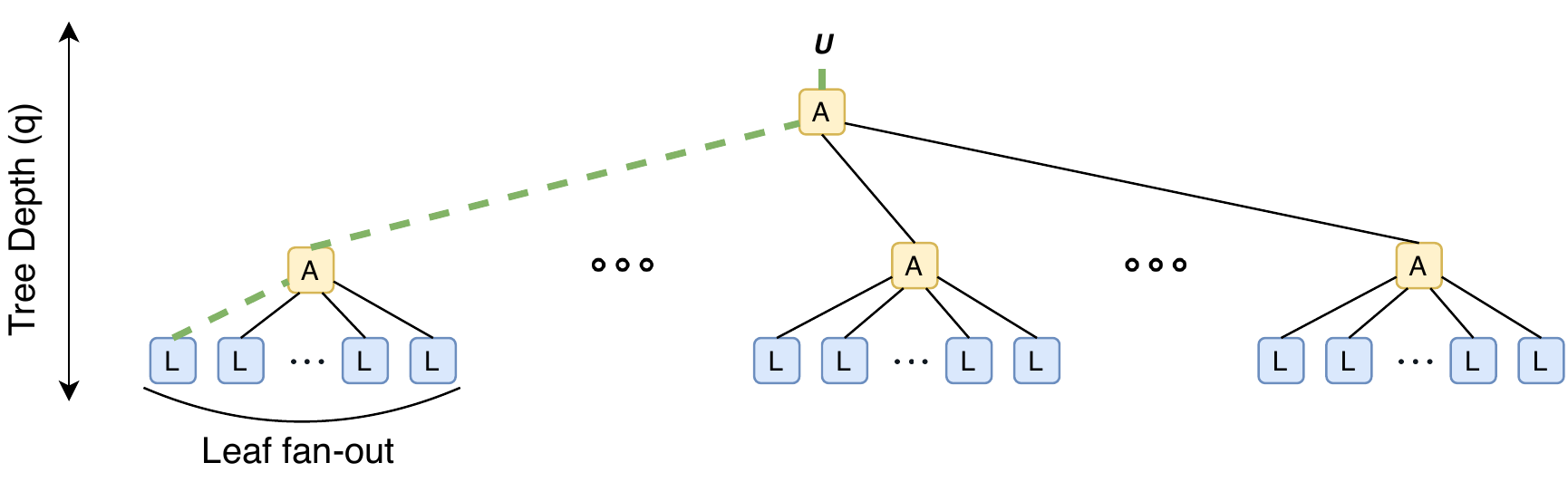}
    \caption{Federated model: (1) Leaf nodes ($\mathbf{L}$) independently compute local updates asynchronously, (2) The subspace updates are propagated upwards to aggregator nodes ($\mathbf{A}$), (3) The process is repeated recursively until the root node is reached, (4) $\FPCAC$ returns the global PCA estimate.}
    \label{fig:computation_graph}
\end{figure}
\subsection{Merging}
Our algorithmic constructions are built upon the concept of {\em subspace merging} in which two subspaces $\mathcal{S}_1 =(\B{U}_{1}, \B{\Sigma}_1) \in \R^{r_{1} \times d} \times \R^{r_1 \times r_1}$ and $\mathcal{S}_2=(\B{U}_{2}, \B{\Sigma_2}) \in \R^{r_{2} \times d} \times \R^{r_2 \times r_2}$ are {\em merged} together to produce a subspace $\mathcal{S}=(\B{U}, \B{\Sigma}) \in \R^{r \times d} \times \R^{r \times r}$ describing the combined $r$ principal directions of $\mathcal{S}_1$ and $\mathcal{S}_2$.
One can merge two sub-spaces  by computing a truncated $\SVD$ on a concatenation of their bases. 
Namely,
\begin{equation}
    \label{eq:basic_subspace}
    [\B{U},\B{\Sigma},\B{V^{T}}] \leftarrow \SVD_{r}([\lambda \B{U}_{1}\B{\Sigma}_{1},\B{U}_{2}\B{\Sigma}_{2}]),
\end{equation}
where $\lambda \in (0,1]$ a {\em forgetting factor} that allocates less weight to the previous subspace $\B{U}_{1}$.
In~\citep{vrehuuvrek2011subspace,eftekhari2019moses}, it is shown how \eqref{eq:basic_subspace} can be further optimised when $\mathbf{V}^{T}$ is not required and we have knowledge that $\mathbf{U}_{1}$ and $\mathbf{U}_{2}$ are already orthonormal. 
An efficient version of \eqref{eq:basic_subspace} is presented in Alg.~\ref{algorithm:fastest_subspace}.
\begin{algorithm}[htb!]
\footnotesize{
    \DontPrintSemicolon
    \SetNoFillComment
    \SetKwProg{Fn}{Function}{ is}{end}
    \KwData{$(\B{U}_{1}, \B{\Sigma}_1) \in \R^{d \times r_{1}} \times \R^{r_1 \times r_1}$: {\em First subspace} // $(\B{U}_{2}, \B{\Sigma}_2) \in \R^{d \times r_{2}} \times \R^{r_2 \times r_2}$: {\em Second subspace};\;
    }
    \KwResult{
        $(\B{U}'', \B{\Sigma}'') \in \R^{d \times r} \times \R^{r \times r}$ merged subspace\;
    }
    \Fn{$\Merge_{r}( \B{U}_1, \B{\Sigma}_1, \B{U}_2, \B{\Sigma}_2)$}{
        $\B{Z} \leftarrow \B{U}^{T}_{1}\B{U}_{2}$\;
        $[\B{Q}, \B{R}] \leftarrow \QR(\B{U}_{2} - \B{U}_{1}\B{Z})$, the QR factorisation\;
        $[\B{U}',\B{\Sigma}'', \thicksim] \leftarrow \SVD_{r}\bigg(
        \begin{bmatrix} 
            \B{\Sigma}_{1} & \B{Z}\B{\Sigma}_{2} \\
            0 & \B{R}\B{\Sigma}_{2}
        \end{bmatrix}
        \bigg)$\;
        $\B{U}'' \leftarrow [\B{U}_{1}, \B{Q}]\B{U}'$\;
        }
    }
    \caption{$\Merge_{r}$ \citep{vrehuuvrek2011subspace, eftekhari2019moses}}
    \label{algorithm:fastest_subspace}
\end{algorithm}
Alg.~\ref{algorithm:fastest_subspace} is generalised in~\citep{iwen2016distributed} to multiple subspaces when the computation is incremental, but not streaming.
That is, when every subspace has to be computed in full in order to be processed, merged, and propagated synchronously, which is not ideal for use in a federated approach.
Hence, in Lemma \ref{lemma:svd_partial_lemma} we extend the result in~\citep{iwen2016distributed} to the case of {\em streaming} data.
Lemma \ref{lemma:svd_partial_lemma} is proved in the Appendix.
\begin{lemma}[Federated $\SVD$ uniqueness]
  \label{lemma:svd_partial_lemma}
  Consider a network with $M$ nodes where, at each timestep $t \in \mathbb{N}$, node $i \in \{1, \dots, M\}$ processes a dataset $\B{D}_t^i \in \R^{d \times b}$.
  At time $t$, let $\Y_t^i = [\B{D}_1^i \mid \cdots \mid \B{D}_t^i] \in \R^{d \times tb}$ be the dataset observed by node $i$ and 
  $\Y_{t} = \left[\Y^{1}_{t}|\Y^{2}_{t}|\cdots|\Y^{M}_{t}\right]\in \R^{d\times tMb}$ be the dataset observed by the network.
  Moreover, let 
  $\B{Z}_{t} := [{\B{U}}^1_{t} {\B{\Sigma}}^1_{t} \mid\cdots\mid {\B{U}}^M_{t}{\B{\Sigma}}^M_{t} ]$ where $[\B{U}^i_t, \B{\Sigma}^i_t, (\B{V}^i_t)^T] = \SVD(\Y_t^i)$. 
  If $\left[{\B{U}_{t}}, {\B{\Sigma}_{t}}, {\B{V}}^{T}_{t}\right] = \SVD(\Y_{t})$ and $[\hat{\B{U}_{t}}, \hat{\B{\Sigma}_{t}}, (\hat{\B{V}_{t}})^{T} ] = \SVD(\B{Z}_t)$, 
  then  ${\B{\Sigma}} = \hat{\B{\Sigma}_{t}}$, and ${\B{U}_{t}} = \hat{\B{U}_{t}}\B{B}_{t}$, where $\B{B}_{t}\in\R^{r \times r}$ is a unitary block diagonal matrix with $r=\rank(\Y_t)$ columns.  
  If none of the nonzero singular values are repeated then $\B{B}_{t}=\B{I}_{r}$. 
  A similar result holds if $b$ differs for each worker as long as $b \geq \min \rank(\Y_t^i)$ $\forall i \in [M]$.
\end{lemma}

\subsection{Local update estimation: Subspace tracking}
\label{streaming_pca}

Consider a sequence $\{\mathbf{y}_1, \dots, \mathbf{y}_n\} \subset \R^{d}$ of feature vectors.
A block of size $b \in \N$ is formed by taking $b$ contiguous columns of $\{\mathbf{y}_1, \dots, \mathbf{y}_n\}$.
Assume $r \leq b \leq \tau \leq n$. If $\widehat{\Y}_{0}$ is the empty matrix, the $r$ principal components of $\Y_{\tau} := [\mathbf{y}_1, \cdots, \mathbf{y}_{\tau}]$ can be estimated by running the following iteration for $k = \{1, \dots, \lceil \tau/b \rceil\}$,
\begin{equation}
\label{eq:svdr_1}
[\widehat{\B{U}}, \widehat{\B{\Sigma}},\widehat{\B{V}}^T]\leftarrow\text{SVD}_{r}\left(\left[\begin{array}{cccc}
\widehat{\Y}_{(k-1)b} & \mathbf{y}_{(k-1)b+1} & \cdots & \mathbf{y}_{kb}\end{array}\right]\right), \;\;\;\; \widehat{\Y}_{kb}\leftarrow \widehat{\B{U}} \widehat{\B{\Sigma}}\widehat{\B{V}}^T \in \R^{d \times kb}.
\end{equation}
Its output after $K = \lceil \tau /b \rceil$ iterations contains an estimate $\widehat{\B{U}}$ of the leading $r$ principal components of $\Y_{\tau}$ and the projection $\widehat{\Y}_{\tau}=\widehat{\B{U}} \widehat{\B{\Sigma}}\widehat{\B{V}}^T$  of $\Y_{\tau}$ onto this estimate.
The local subspace estimation
in \eqref{eq:svdr_1} was initially analysed in~\citep{eftekhari2019moses}. 
$\FPCAEC$ adapts \eqref{eq:svdr_1} to the federated setting by implementing an adaptive rank-estimation procedure which allows clients to adjust, independently of each other, their rank estimate based on the distribution of the data seen so far.
That is, by enforcing,
\begin{equation}
\label{eq:sv_energy}
     \energy_r({\Y_{\tau}}) = \frac{\sigma_{r}(\Y_{\tau})}{\sum_{i=1}^{r}\sigma_{i}(\Y_{\tau})}
     \in [ 
     \alpha, \beta],
\end{equation}
and increasing $r$ whenever $\energy_r(\Y_{\tau}) >\beta$ or decreasing it when $\energy_r(\Y_{\tau}) < \alpha$.
In our algorithm, this adjustment happens only once per block, though a number of variations to this strategy are possible. 
Further, typical values for $\alpha$ and $\beta$ are $1$ and $10$ respectively; note for best results the ratio $\alpha/\beta$ should be kept below $0.3$.
Letting $[r+1] = \{1, \dots, r+1\}$, $[r-1] = \{1, \dots, r-1\}$, and $\ind\{\cdot\}\in \{0, 1\}$ be the indicator function, the subspace tracking and rank-estimation procedures in Alg.~\ref{alg:fpca_edge} depend on the following functions: 
\begin{footnotesize}
\begin{equation*}
\begin{array}{rl}
\RSPCA_r(\mathbf{D}, \mathbf{U}, \mathbf{\Sigma}) =& 
        \SVD_r (\mathbf{D}) \ind\{ \mathbf{U}\mathbf{\Sigma} = 0\} + \mbox{Merge}_{r}(\mathbf{U}, \mathbf{\Sigma}, \mathbf{D}, \mathbf{I}) \ind\{\mathbf{U}\mathbf{\Sigma} \neq 0 \}\\
\rankupdate_r^{\alpha, \beta}({\mathbf{U}}, {\mathbf{\Sigma}}) =& 
        \left([{\mathbf{U}}, \vec{\mathbf{e}}_{r+1}], \B{\Sigma}_{[r+1]} \right) \ind\{\energy_r({\mathbf{\Sigma}}) > \beta\} + 
        ({\mathbf{U}}_{[r-1]}, {\mathbf{\Sigma}}_{[r-1]}) \ind\{\energy_r({\mathbf{\Sigma}}) < \alpha\} \\
        &+ 
        (\mathbf{U}, \mathbf{\Sigma}) \ind\{\energy_r({\mathbf{\Sigma}}) \in [\alpha, \beta]\}
\end{array}
\end{equation*}
\end{footnotesize}
Note that the storage and computational requirements of the {\em Subspace tracking} procedure of Alg.~\ref{alg:fpca_edge} are nearly optimal for the given objective since, at iteration $k$, only requires $\mathcal{O}(r(d+kr))$ bits of memory and $\mathcal{O}(r^2(d+kr))$ flops.
However, in the presence of perturbation masks, the computational complexity is  $\mathcal{O}(d^2)$ due to the incremental covariance expansion per block, see Sec.~\ref{sec:private-estimation}.
\subsection{Differential Privacy: Streaming $\modsulq$}
\label{sec:private-estimation}

Given a data matrix $\X \in \R^{d \times n}$ and differential privacy parameters $(\varepsilon, \delta)$, the $\modsulq$ algorithm \citep{chaudhuri2012near} privately computes the $k$-leading principal components of
\begin{equation}
    \label{eq:A}
    \mathbf{A} = \frac{1}{n} \X\X^T + \mathbf{N}_{\varepsilon, \delta, d, n} \in \R^{d \times d},
\end{equation}
the covariance matrix of $\X$ perturbed with a {\em symmetric} random Gaussian matrix $\mathbf{N}_{\varepsilon, \delta, d, n}\in \R^{d \times d}$.
This {\em symmetric} perturbation mask is such that $\left(\mathbf{N}_{\varepsilon, \delta, d, n}\right)_{i,j} \sim \mathcal{N}(0, \omega^2)$ for $i \geq j$ where
\begin{equation}
     \label{eq:omega}
     \omega:=\omega(\varepsilon, \delta, d, n) = \frac{d+1}{n \varepsilon}\sqrt{2 \log\left(\frac{d^2 + d}{2\delta\sqrt{2\pi}}\right)} + \frac{1}{n \sqrt{\varepsilon}}.
\end{equation}
Materialising \eqref{eq:A} requires $\mathcal{O}(d^2)$ memory which is prohibitive given our complexity budgets.
We can reduce the memory requirements to $\mathcal{O}(cdn)$ by computing $\X\X^T$ incrementally in batches of size $c \leq d$.
That is, by drawing $\B{N}_{\varepsilon, \delta, d, n}^{d\times c} \in \R^{d \times c}$ and merging the {\em non-symmetric} updates 
\begin{equation}
\label{eq:streaming-privacy}
\B{A}_{k,c} = \frac{1}{b}\X
    \begin{bmatrix}
        (\X^T)_{(k-1)c+1} & \cdots & (\X^T)_{ck} 
    \end{bmatrix}
    + \B{N}_{\varepsilon, \delta, d, n}^{d \times c}
\end{equation}
using Alg.~\ref{algorithm:fastest_subspace}.
In Lemma \ref{lemma:diffprivacy} we extend the results in \citep{chaudhuri2012near} to guarantee $(\varepsilon, \delta)$-differential privacy in \eqref{eq:streaming-privacy}.
While the $\sulq$ algorithm~\citep{blum2005practical}, guarantees $(\varepsilon, \delta)$-differential privacy with non-symmetric noise matrices, it requires a variance rate of $\omega^2=\frac{8d^2\log^2(d/\delta)}{n^2 \varepsilon^2}$, which is sub-optimal with respect to the $\mathcal{O}(\frac{d^2 \log(d/\delta)}{n^2 \varepsilon^2})$ guaranteed by Lemma \ref{lemma:diffprivacy}. Lemma \ref{lemma:diffprivacy} is proved in the Appendix.
\begin{lemma}[Streaming Differential Privacy]
\label{lemma:diffprivacy}
Let $\B{X}=\left[\mathbf{x}_1 \cdots \mathbf{x}_n\right] \in \R^{d \times n}$ be a dataset with $\|\mathbf{x}_i\|\leq 1$, $\mathbf{N}_{\varepsilon, \delta, d, n} \in \R^{d \times d}$ and $\B{A}= \frac{1}{n}\B{X}\B{X}^T +\mathbf{N}_{\varepsilon, \delta, d, n}$. Let $\{\mathbf{v}_1, \dots, \mathbf{v}_d\}$ be the eigenvectors of $\frac{1}{n}\B{X}\B{X}^T$ and $\{\hat{\mathbf{v}}_1, \dots, \hat{\mathbf{v}}_d\}$ be the eigenvectors of $\B{A}$. Let
\begin{equation}
\label{eq:omega-streaming}
\omega(\varepsilon, \delta, d, n) = \frac{4d}{\varepsilon n} \sqrt{2 \log \left(\frac{d^2}{\delta \sqrt{2\pi}}\right)} + \frac{\sqrt{2}}{\sqrt{\varepsilon} n}.
\end{equation}
\begin{enumerate}
    \item If $\left(\mathbf{N}_{\varepsilon, \delta, d, n}\right)_{i,j} \sim \mathcal{N}(0, \omega^2)$ independently, then \eqref{eq:streaming-privacy}
    is $(\varepsilon, \delta)$-differentially private. 
    \item If $n \geq T_{\varepsilon, \delta, d, n} := {\omega_0^{-1}} \left[4d\varepsilon^{-1}\sqrt{2 \log\left(d^2\delta^{-1} (2\pi)^{-1/2}\right)} + \sqrt{2\varepsilon^{-1}}\right]^{-1}$, then \eqref{eq:streaming-privacy} is $(\varepsilon, \delta)$-differentially private for a noise mask with variance $\omega_0^2$.
    
    \item Iteration \eqref{eq:streaming-privacy} inherits $\modsulq$'s sample complexity guarantees, and asymptotic utility bounds on $\mathbb{E}\left[|\langle \mathbf{v}_1, \hat{\mathbf{v}}_1 \rangle|\right]$ and  $\mathbb{E}\left[\|\mathbf{v}_1 - \hat{\mathbf{v}}_1\|\right]$.
\end{enumerate}

\end{lemma}

Alg.~\ref{alg:fpca_edge} uses the result in Lemma \ref{lemma:diffprivacy} for $\X = \B{B} \in \R^{d \times b}$ and computes an input-perturbation in a streaming way in batches of size $c$. Therefore, the utility bounds for Alg.~\ref{alg:fpca_edge} can be obtained by setting $n = b$ in \eqref{eq:omega-streaming}.
If $c$ is taken as a fixed small constant the memory complexity of this procedure reduces to $\mathcal{O}(db)$, which is linear in the dimension.
\begin{algorithm}[htb!]
\footnotesize{
    \DontPrintSemicolon
    \SetNoFillComment
    \SetKwProg{Fn}{Function}{ is}{end}
    \KwData{$\B{B}\in \R^{d \times b}$: {\em Batch $\Y_{\{(k-1)b+1, \dots, kb\}}$} // 
    $(\wh{\mathbf{U}}_{k-1}, \wh{\mathbf{\Sigma}}_{k-1})$: {\em SVD estimate for $\Y_{\{1, \dots, (k-1)b\}}$} //
    $r$: {\em Initial rank estimate} //
    $(\alpha, \beta)$: {\em Bounds on energy, see \eqref{eq:sv_energy}} // $(\varepsilon, \delta)$: { \em DP parameters} // $r$: {\em Initial rank estimate}
    }
    \KwResult{ $(\wh{\mathbf{U}},\wh{\mathbf{\Sigma}})$, principal $r$-subspace of $\Y_{\{1, \dots, kb\}}$.}
    \Fn{$\FPCAEC_{\varepsilon, \delta, \alpha, \beta, r}(\B{B}, \wh{\mathbf{U}}_{k-1}, \wh{\mathbf{\Sigma}}_{k-1})$}{
    \tcc{Streaming $\modsulq$}
    $(\B{U}, \B{\Sigma}) \leftarrow (0,0)$\;
    \For{$\ell \in \{1, \dots, d/c\}$}{
        $\B{B}_{s} \leftarrow \frac{1}{b}\BLK(\BLK_{\{(\ell-1)c + 1, \dots, \ell c\}})^{T} + \B{N}_{\varepsilon, \delta, d, b}^{d \times c}$ such that $\left(\B{N}_{\varepsilon, \delta, d, b}^{d\times c}\right)_{i,j} \sim \mathcal{N}(0, \omega^{2})$ and $\omega$ as in \eqref{eq:omega-streaming}\;
        $(\B{U}, \B{\Sigma}) \leftarrow  \RSPCA_r(\B{B}_{s}, \B{U}, \B{\Sigma})$\;
    }
    \tcc{Subspace tracking}
    $(\wh{\mathbf{U}}', \wh{\mathbf{\Sigma}}') \leftarrow \Merge_{r}(\mathbf{U}, \mathbf{\Sigma}, \wh{\mathbf{U}}_{k-1}, \wh{\mathbf{\Sigma}}_{k-1})$\;
    $(\wh{\mathbf{U}}, \wh{\mathbf{\Sigma}}) \leftarrow \rankupdate_r^{\alpha, \beta}(\wh{\mathbf{U}}',\wh{\mathbf{\Sigma}}')$
        }
    }
    \caption{Federated PCA Edge ($\FPCAEC$)}
    \label{alg:fpca_edge} 
\end{algorithm}
A value for $\varepsilon$ can be obtained from Apple's differential privacy guidelines~\cite{applediifpriv}.
However, in our experiments, we benchmark across a wider spectrum of values.

\section{Experimental Evaluation}
\label{label:eval}
All our experiments were computed on a workstation using an AMD 1950X CPU with $16$ cores at $4.0$GHz, $128$ GB $3200$ MHz DDR4 RAM, and Matlab R2020a (build 9.8.0.1380330). 
To foster reproducibility both code and datasets used for our numerical evaluation are made publicly available at: \url{https://www.github.com/andylamp/federated\_pca}.

\subsection{Differential Privacy empirical evaluation}
To quantify the loss with the application of differential private that our scheme has we compare the quality of the projections using the MNIST standard test set~\citep{lecun2010mnist} and Wine~\cite{cortez2009modeling} datasets which contain, respectively, 10000 labelled images of handwritten digits and physicochemical data for 6498 variants of red and white wine.
To retrieve our baseline we performed the full-rank $\PCA$ on the MNIST and (red) Wine datasets and retrieved the first and second principal components, see Figs.~\ref{fig:pca_mnist}~and~\ref{fig:pca_wine}. 
Then, on the same datasets, we applied $\FPCAC$ with rank estimate $r=6$, block size $b=25$, and DP budget $(\varepsilon, \delta)=(0.1, 0.1)$.
The projections for Offline PCA, $\FPCAC$ with no DP mask, $\FPCAC$ with DP mask, and vanilla $\modsulq$ for the MNIST and (red) Wine datasets are shown in Fig.~\ref{fig:mnist_projections}.
We note that for a fair comparison with $\modsulq$, the rank estimation was disabled in this first round of experiments.
It can be seen from Fig.~\ref{fig:mnist_projections} 
that in all cases $\FPCAC$ learnt the principal subspace of Offline $\PCA$ (up to a rotation) and managed to preserve the underlying structure of the data.
In fact, in most instances it even performed better than $\modsulq$.
We note that rotations are expected
as the guarantees for our algorithm hold up to a unitary transform, see Appendix C.  
\begin{center}
    \begin{figure*}[ht]
    \centering
    \begin{subfigure}{.24\linewidth}
        \centering
        \includegraphics[scale=0.25]{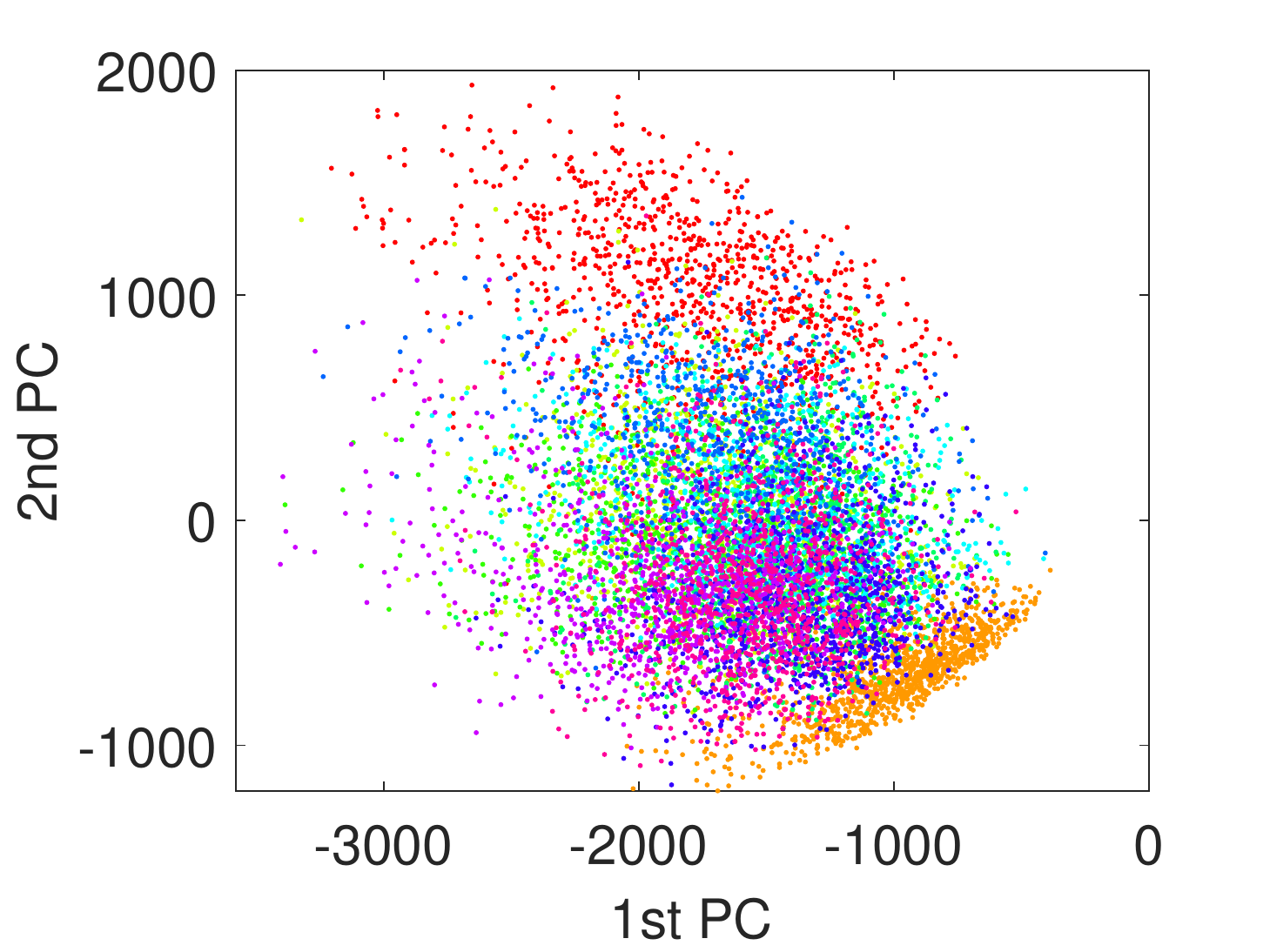}
        \caption{Offline PCA}
        \label{fig:pca_mnist}
    \end{subfigure}
    \begin{subfigure}{.24\linewidth}
        \centering
        \includegraphics[scale=0.25]{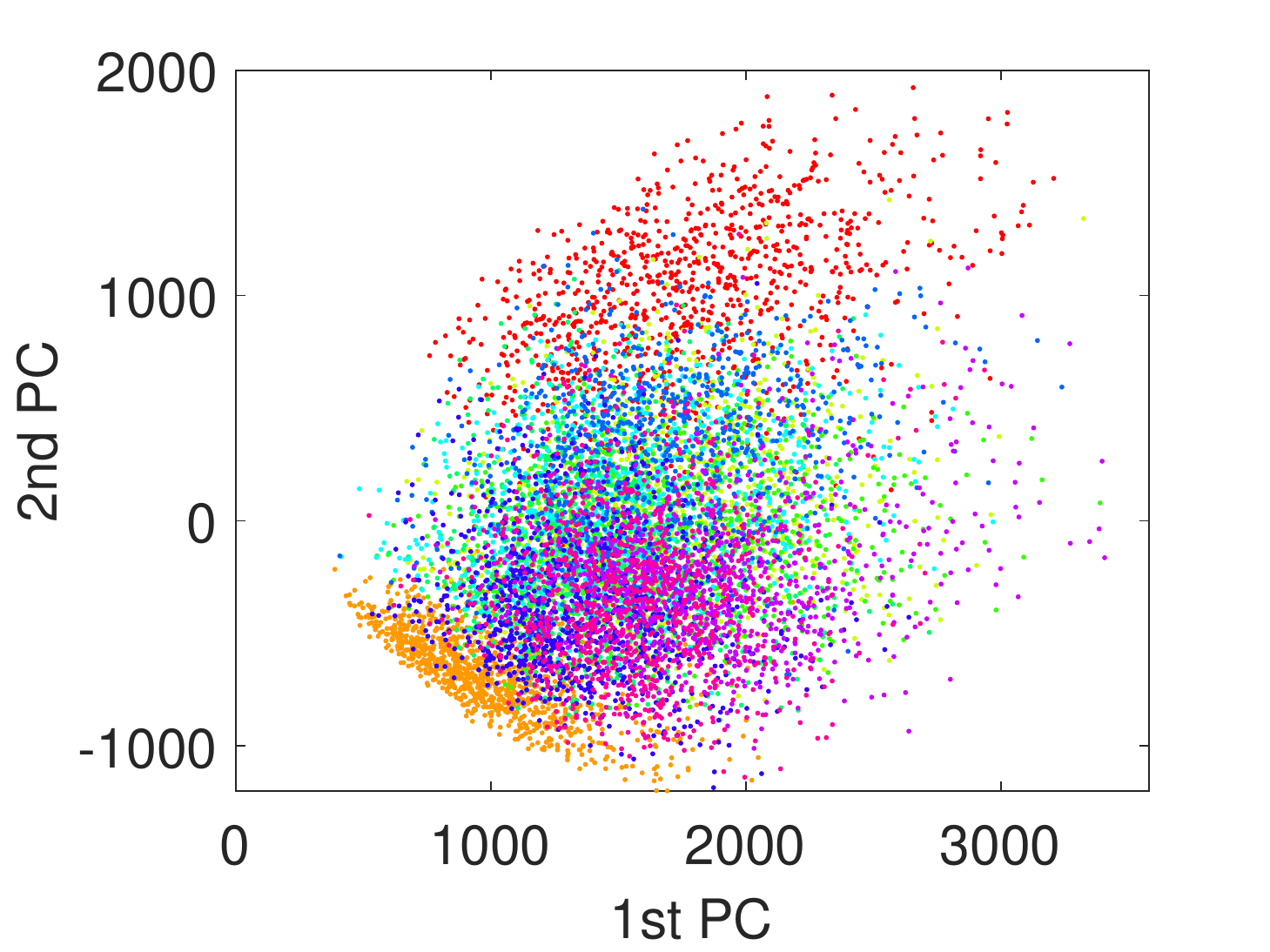}
        \caption{F-PCA (no mask)}
        \label{fig:mnist_fpca}
    \end{subfigure}
    \begin{subfigure}{.24\linewidth}
        \centering
        \includegraphics[scale=0.25]{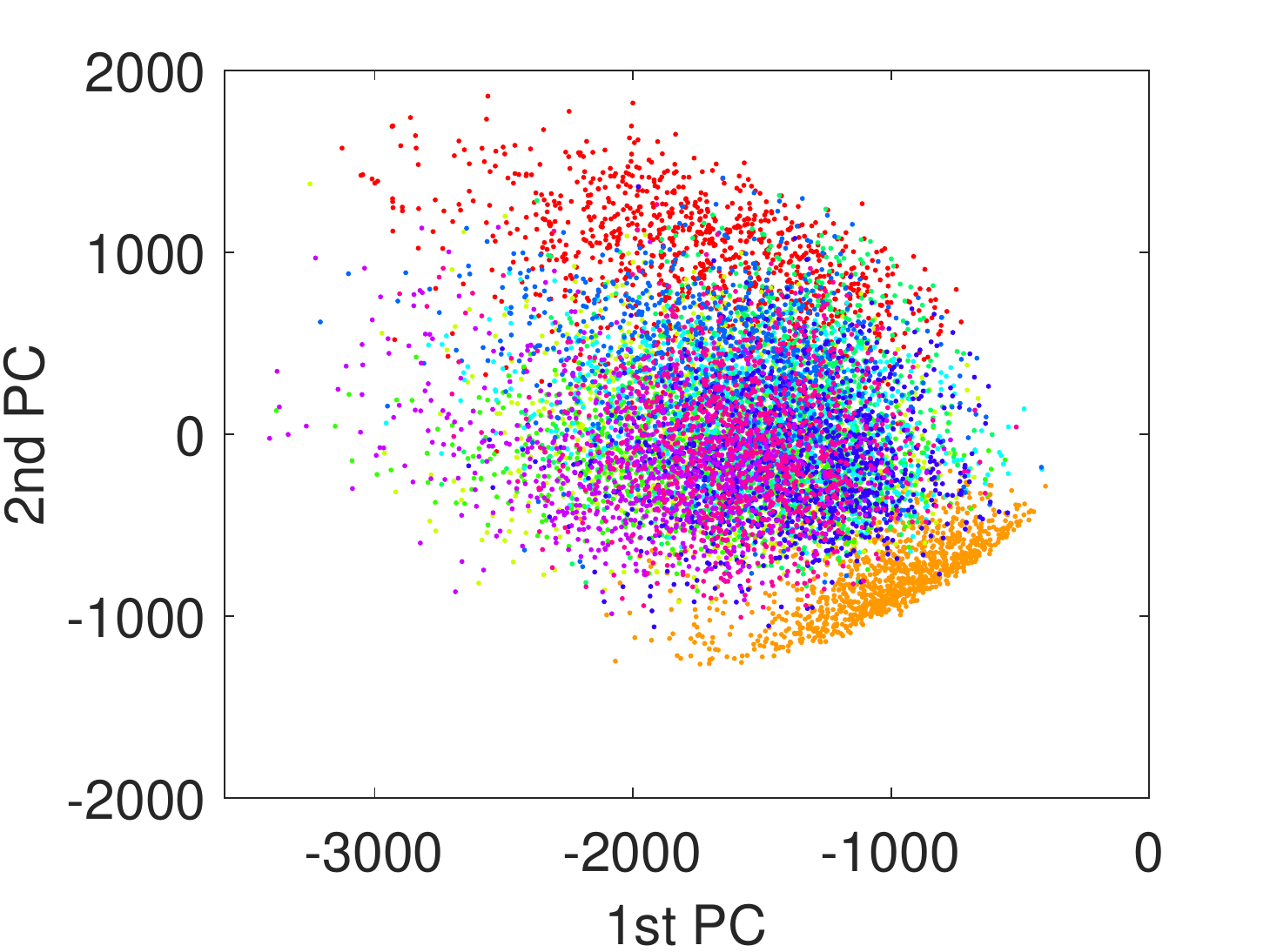}
        \caption{F-PCA (with mask)}
        \label{fig:mnist_fpca_private}
    \end{subfigure}
    \begin{subfigure}{.24\linewidth}
        \centering
        \includegraphics[scale=0.25]{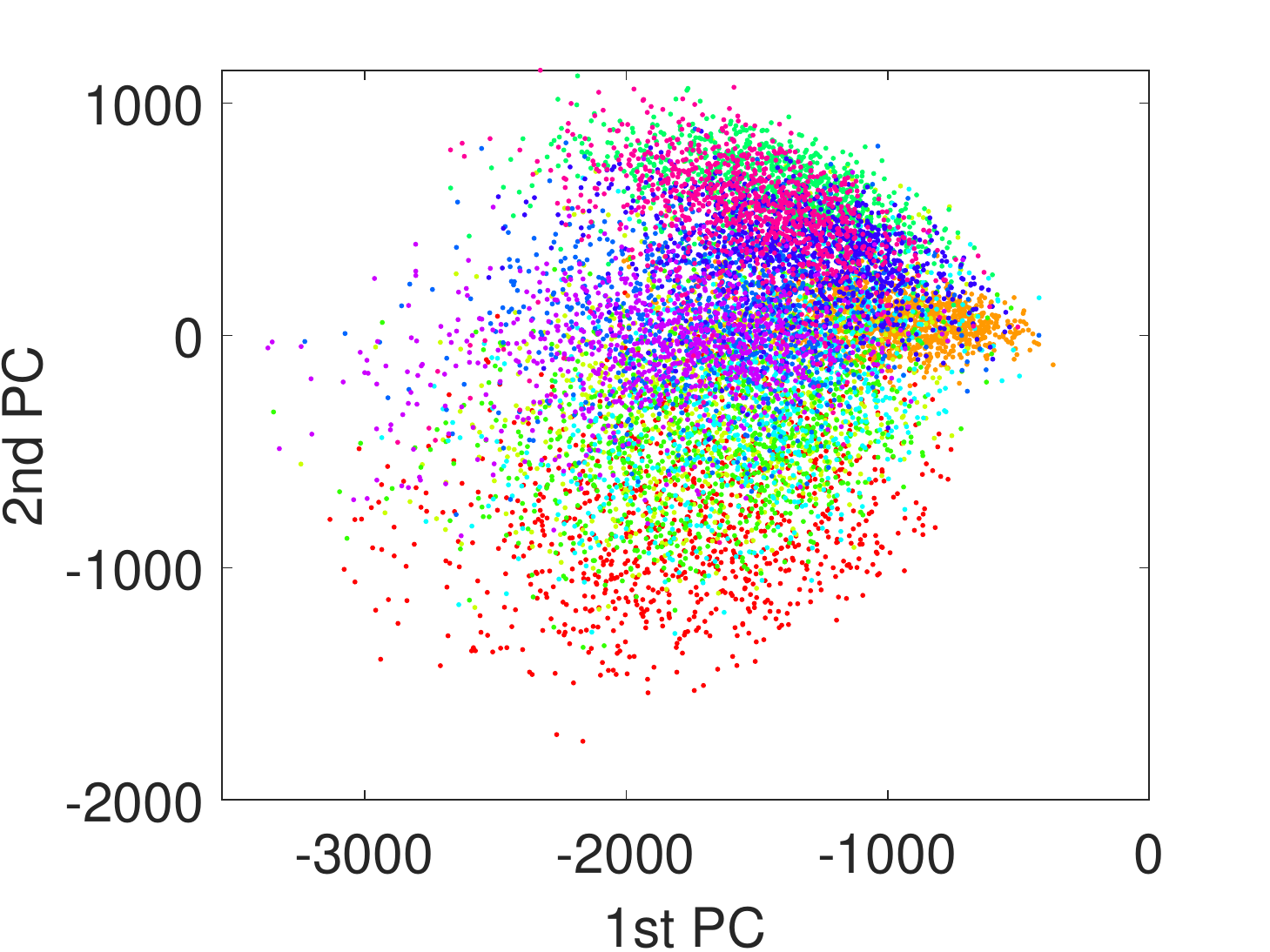}
        \caption{MOD-SuLQ}
        \label{fig:mnist_mod_sulq}
    \end{subfigure} \\
    \begin{subfigure}{.24\linewidth}
        \centering
        \includegraphics[scale=0.25]{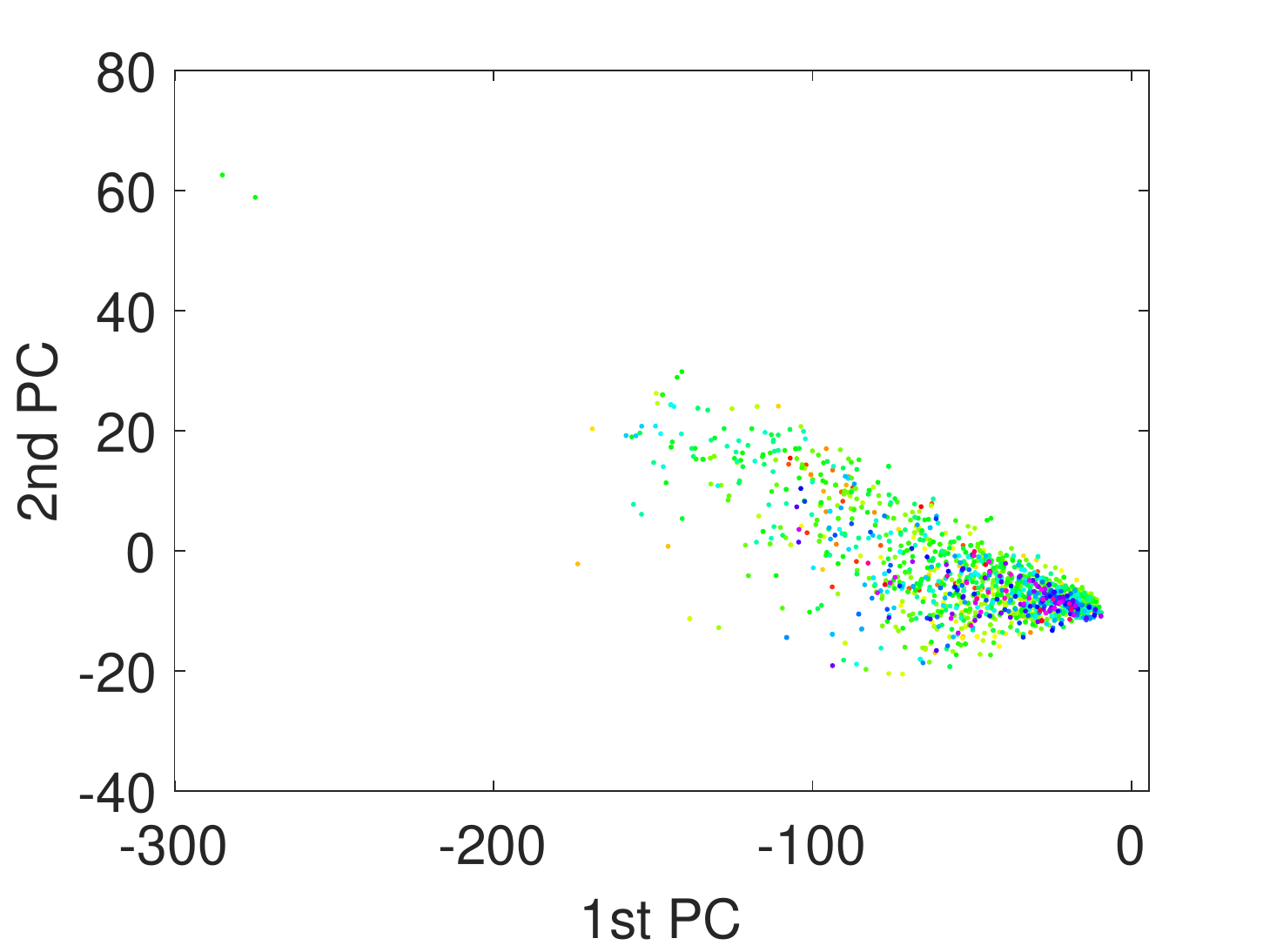}
        \caption{Offline PCA}
        \label{fig:pca_wine}
    \end{subfigure}
    \begin{subfigure}{.24\linewidth}
        \centering
        \includegraphics[scale=0.25]{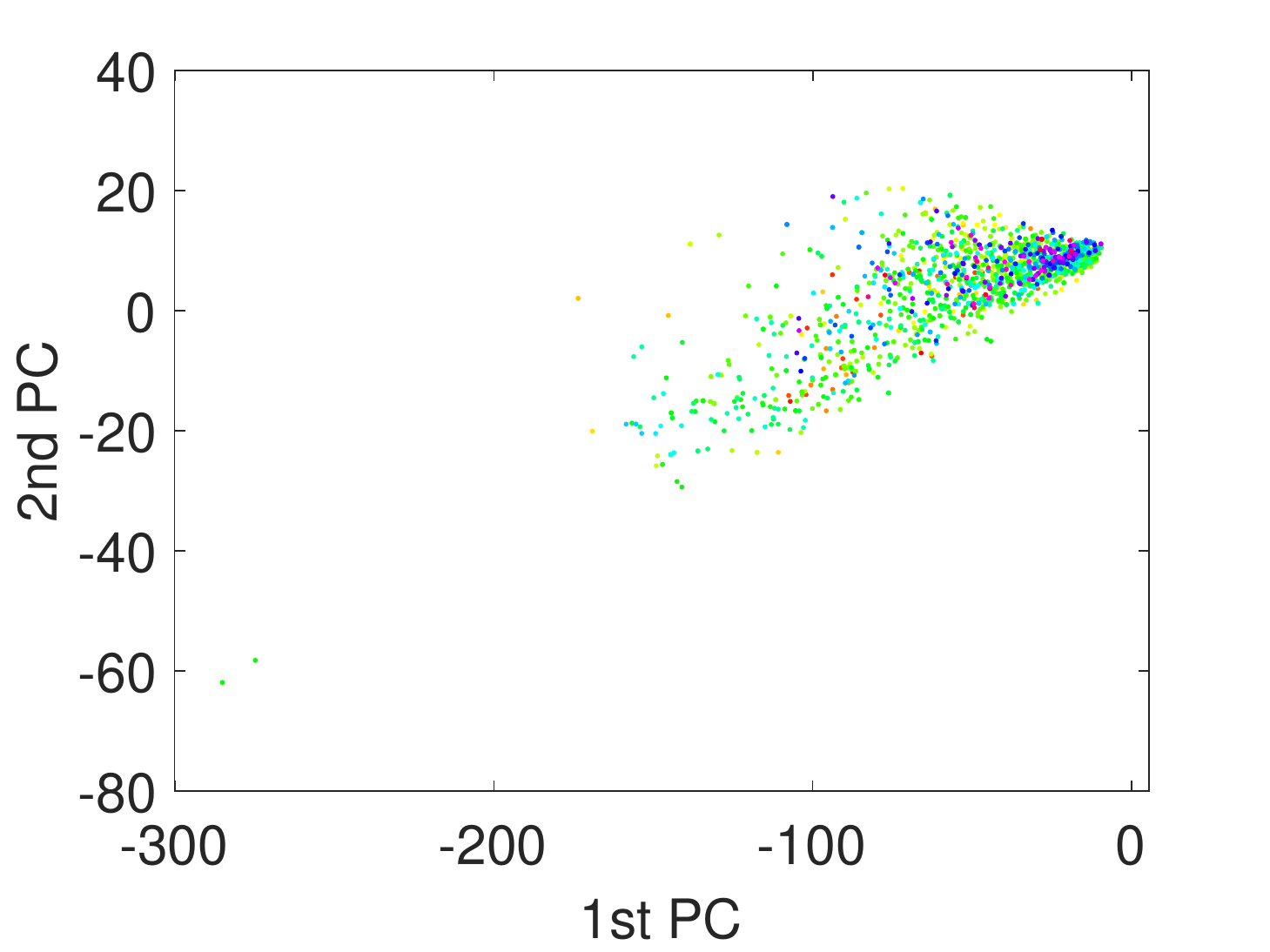}
        \caption{F-PCA (no mask)}
        \label{fig:fpca_wine_no_mask}
    \end{subfigure}
    \begin{subfigure}{.24\linewidth}
        \centering
        \includegraphics[scale=0.25]{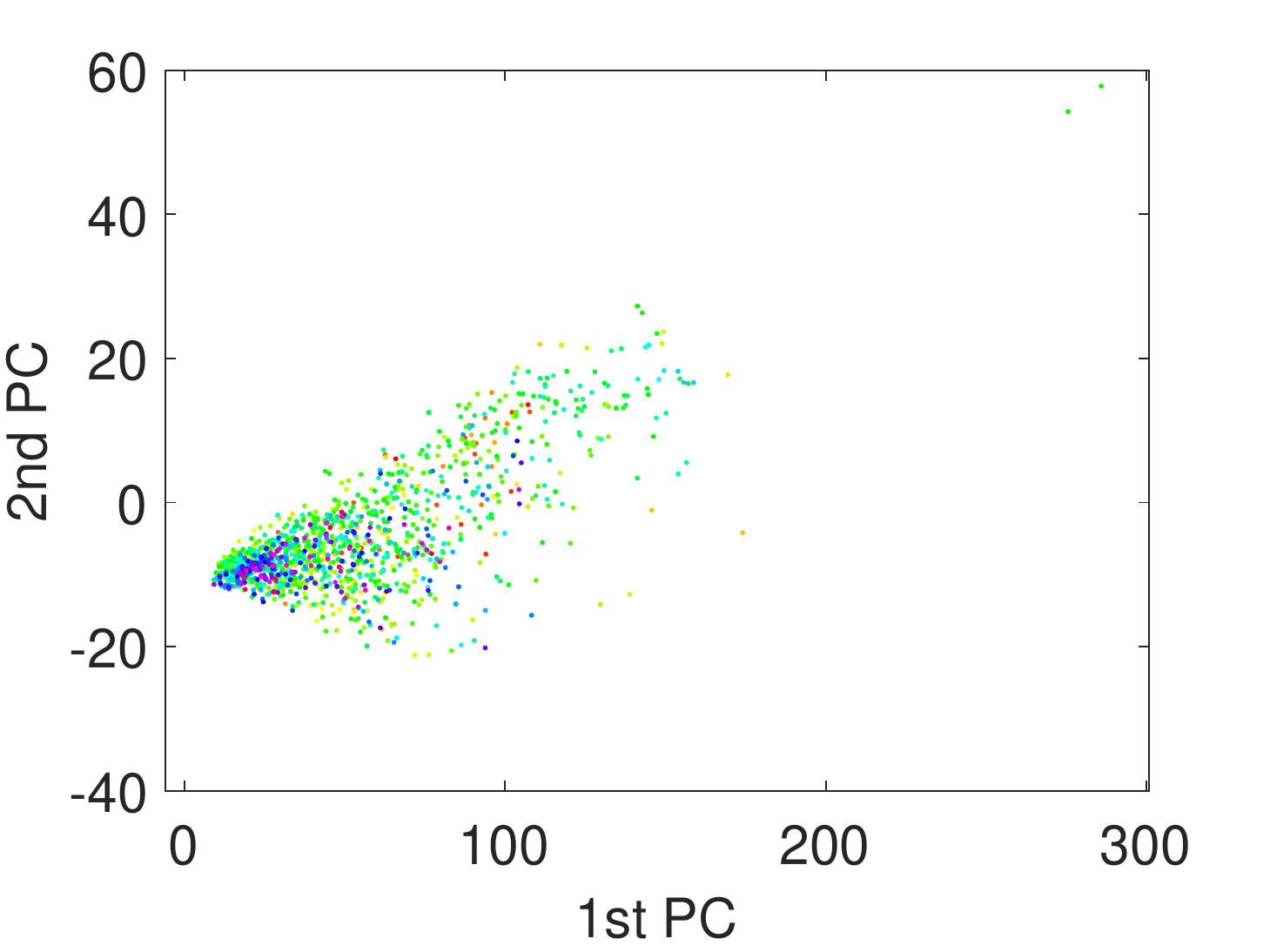}
        \caption{F-PCA (with mask)}
        \label{fig:fpca_wine_with_mask}
    \end{subfigure}
    \begin{subfigure}{.24\linewidth}
        \centering
        \includegraphics[scale=0.25]{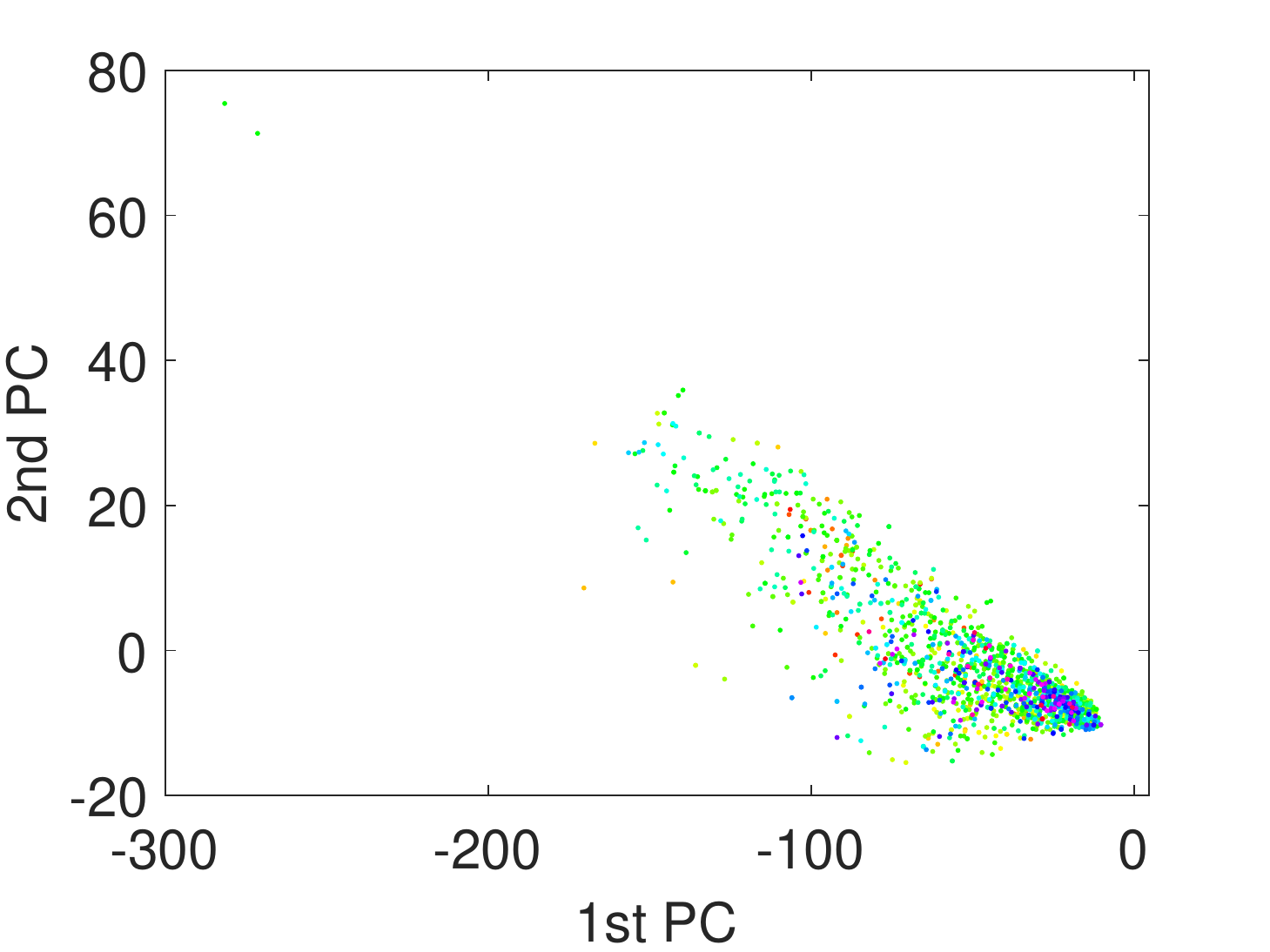}
        \caption{MOD-SuLQ}
        \label{fig:pca_wine_mod_sulq}
    \end{subfigure}
    \caption{MNIST and Wine projections, for (a,e) Offline PCA, (b,f) F-PCA without DP mask, (c,g) F-PCA with DP mask, (d,h) (symmetric) MOD-SuLQ. Computed with DP budget of $(\varepsilon, \delta)=(0.1, 0.1)$.}
    \label{fig:mnist_projections}
    \end{figure*}
\end{center}

To evaluate the utility loss with respect to the privacy-accuracy trade-off we fix $\delta=0.01$ and plot $q_{A}= \langle \mathbf{v}_{1}, \hat{\mathbf{v}}_{1} \rangle$ for $\varepsilon\in\{0.1k : k \in \{1, \dots, 40\}\}$ where $\mathbf{v}_1$ and $\hat{\mathbf{v}}_1$ are defined as in Lemma \ref{lemma:diffprivacy}.
Synthetic data was generated from a power-law spectrum\footnote{If $\Y \sim \text{Synth}(\alpha)^{d \times n}$ iff $\Y = \B{U} \B{\Sigma} \B{V}^T$ with $[\B{U}, \sim] =\mbox{QR}(\B{N}^{d\times d})$, $[\B{V}, \sim] = \mbox{QR}(\B{N}^{d \times n})$, and $\B{\Sigma}_{i,i} = i^{- \alpha}$, and $\B{N}^{m \times n}$ is an $m \times n$ matrix with i.i.d. entries drawn from $\mathcal{N}(0,1)$.} $\Y_{\alpha} \sim \text{Synth}(\alpha)^{d \times n}\subset \R^{d\times n}$ using $\alpha\in\{0.01, 0.1, .5, 1\}$.
The results are shown in~\autoref{fig:utility_loss} where we see that a larger $\varepsilon$ increases the utility, but at the cost of lower DP.
Quantitatively speaking, our experiments suggest that
the more uniform the spectrum is, the harder it is to guarantee DP and preserve the utility. 
\begin{figure}[htb!]
    \centering
    \begin{subfigure}{.32\linewidth}
        \centering
        \includegraphics[scale=0.25]{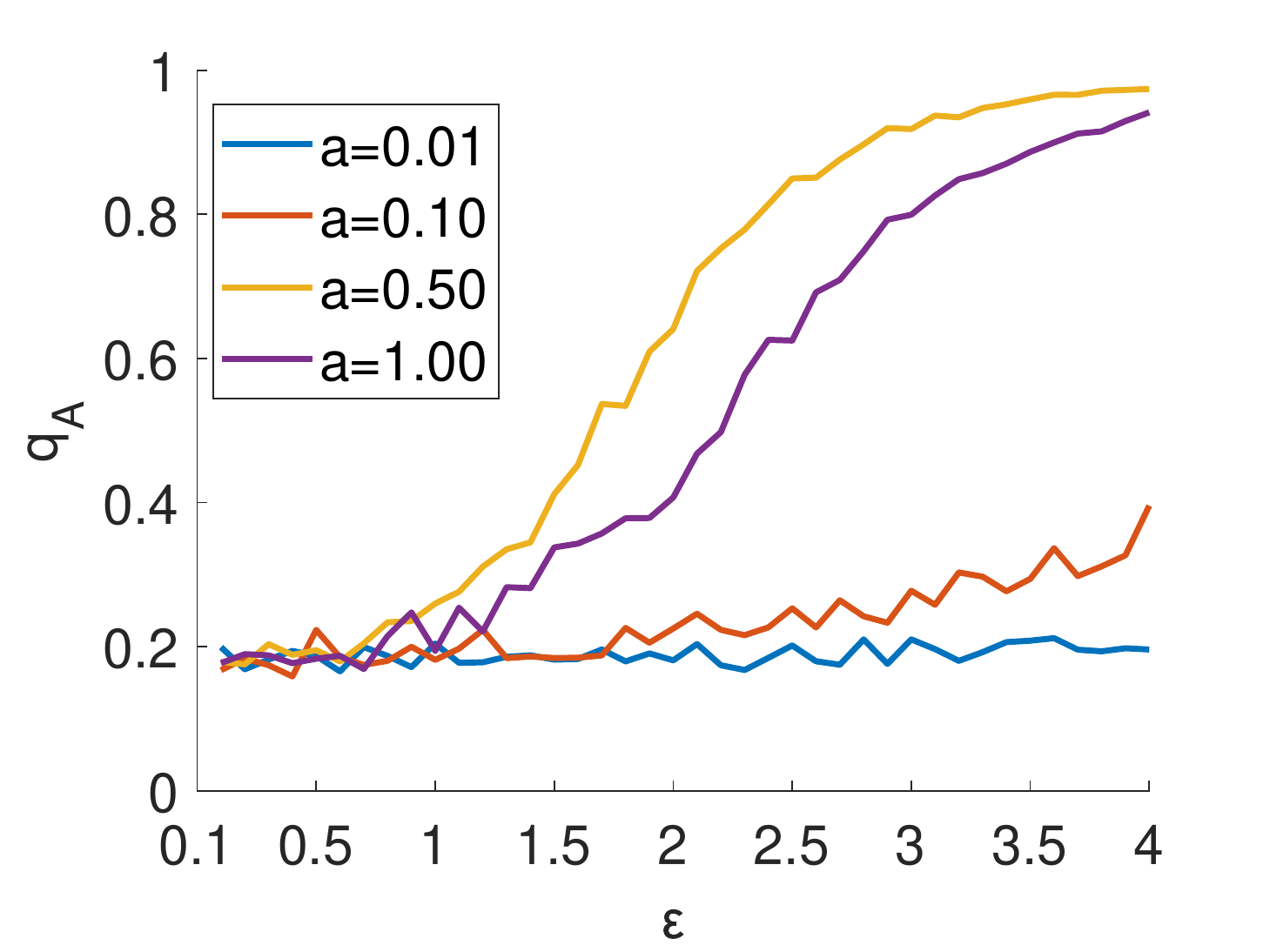}
        \caption{F-PCA (with mask).}
        \label{fig:fed_utility_loss}
    \end{subfigure}
    \begin{subfigure}{.32\linewidth}
        \centering
        \includegraphics[scale=0.25]{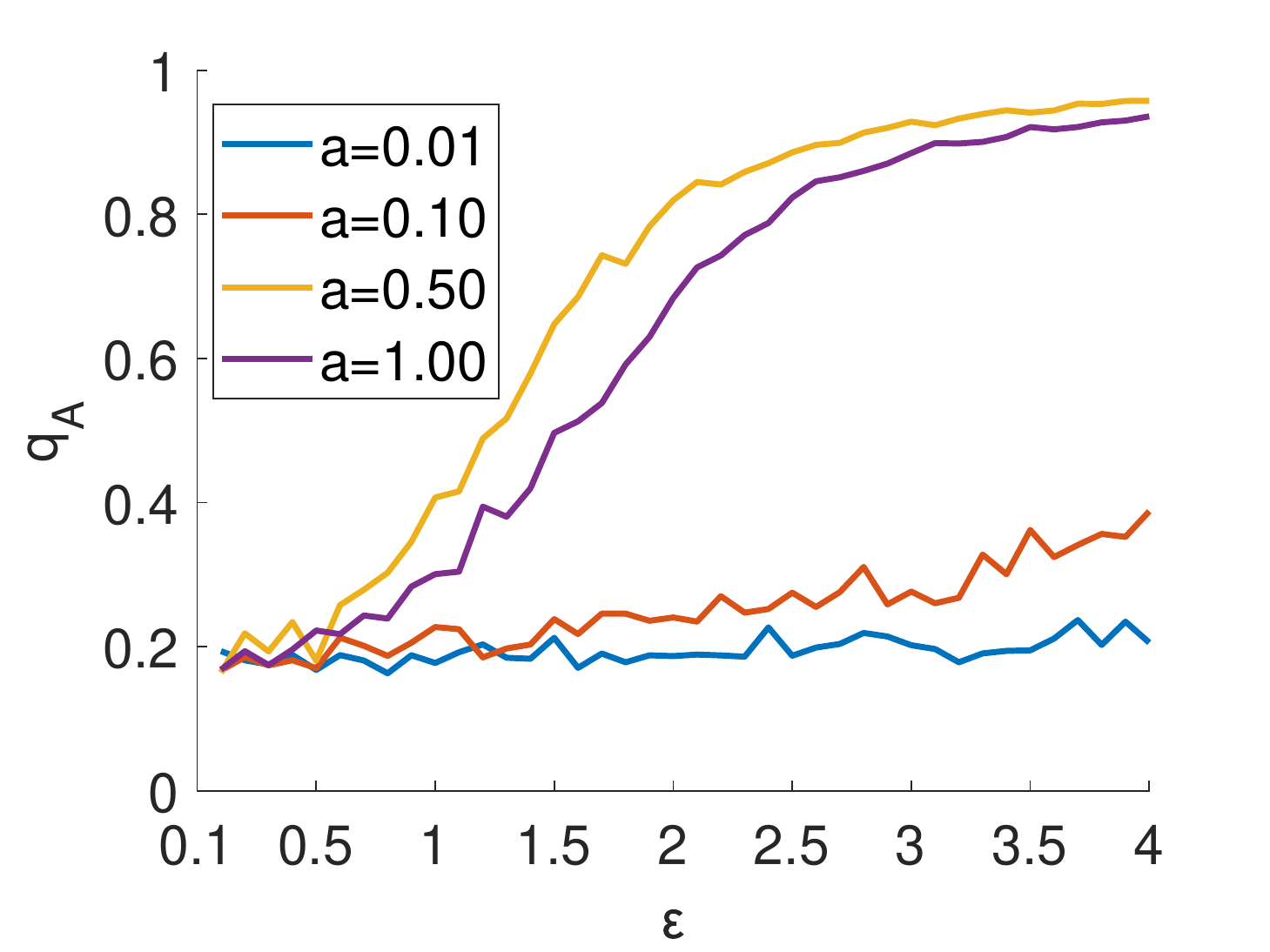}
        \caption{MOD-SuLQ (non-symmetric).}
        \label{fig:mod_sulq_non_symmetric_utility_loss}
    \end{subfigure}
    \begin{subfigure}{.32\linewidth}
        \centering
        \includegraphics[scale=0.25]{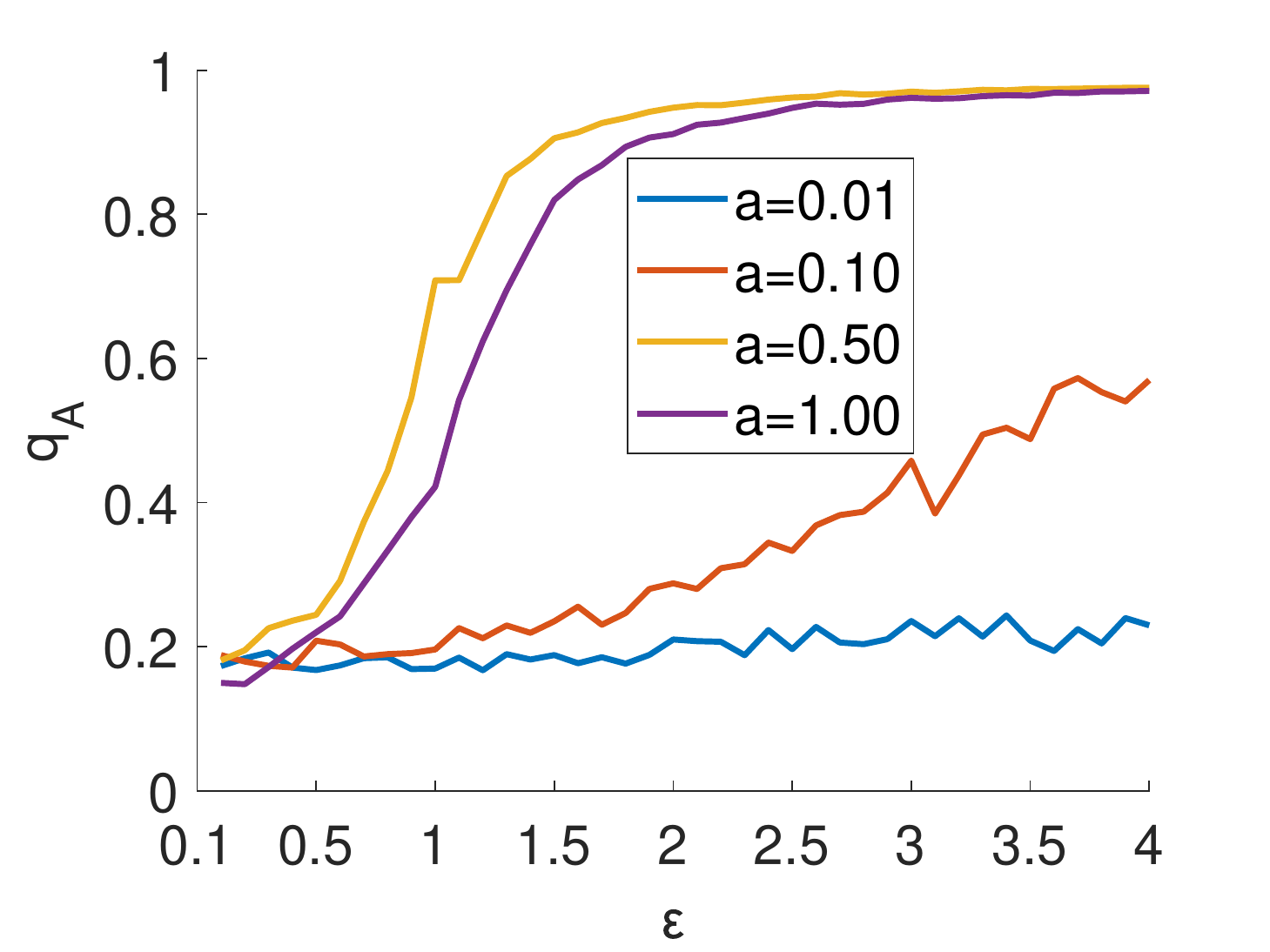}
        \caption{MOD-SuLQ (symmetric).}
        \label{fig:mod_sulq_symmetric_utility_loss}
    \end{subfigure}
    \caption{Utility loss of $q_{A}$ for (a) F-PCA, (b) non-symmetric MOD-SuLQ, and (c) symmetric MOD-SuLQ using $\delta=0.05$, $N=5$k, and $d=20$ across different $\varepsilon$ and $\Y_{\alpha} \sim \text{Synth}(\alpha)^{d \times n}$.
    }
    \label{fig:utility_loss}
\end{figure}

\subsection{Computational performance evaluation}
Figs.~\ref{fig:final_subspace_fro_mse_errors_real},~\ref{fig:final_subspace_fro_mse_errors},~\ref{fig:fpca_speed_exec_eval} evaluate the performance of $\FPCAEC$ against other streaming algorithms. 
The algorithms considered in this instance are: $\FPCAEC$ (on a single node network), GROUSE~\citep{balzano2013grouse}, Frequent Directions (FD)~\citep{desai2016improved,luo2017robust}, the Power Method (PM)~\citep{mitliagkas2014streaming}, and a variant of Projection Approximation Subspace Tracking (PAST)~\citep{yang1995projection},  named SPIRIT (SP)~\citep{papadimitriou2005streaming}.
In the spirit of a fair comparison, we run $\FPCAEC$ without its DP features, given that no other streaming algorithm implements DP.
The algorithms are tested on: (1) synthetic datasets, (2) the {\em humidity}, {\em voltage}, {\em temperature}, and {\em light} datasets of readings from \texttt{Berkeley Mote} sensors~\citep{deshpande2004model}, (3) the MNIST and Wine datasets used in the previous section.
Figs.~\ref{fig:final_subspace_fro_mse_errors_real} and \ref{fig:final_subspace_fro_mse_errors} report $\log(\text{RMSE})$ errors with respect to the offline full-rank $\PCA$ and show that $\FPCAC$ exhibits state-of-the-art performance across all datasets.
On the other hand, Fig.~\ref{fig:fpca_speed_exec_eval}
shows that the computation time of $\FPCAC$ scales gracefully as the ambient dimension $d$ grows, and even outperforms SPIRIT.
\begin{figure}[h]
    \centering
    \begin{subfigure}{.32\linewidth}
        \centering
        \includegraphics
        [scale=0.3]
        {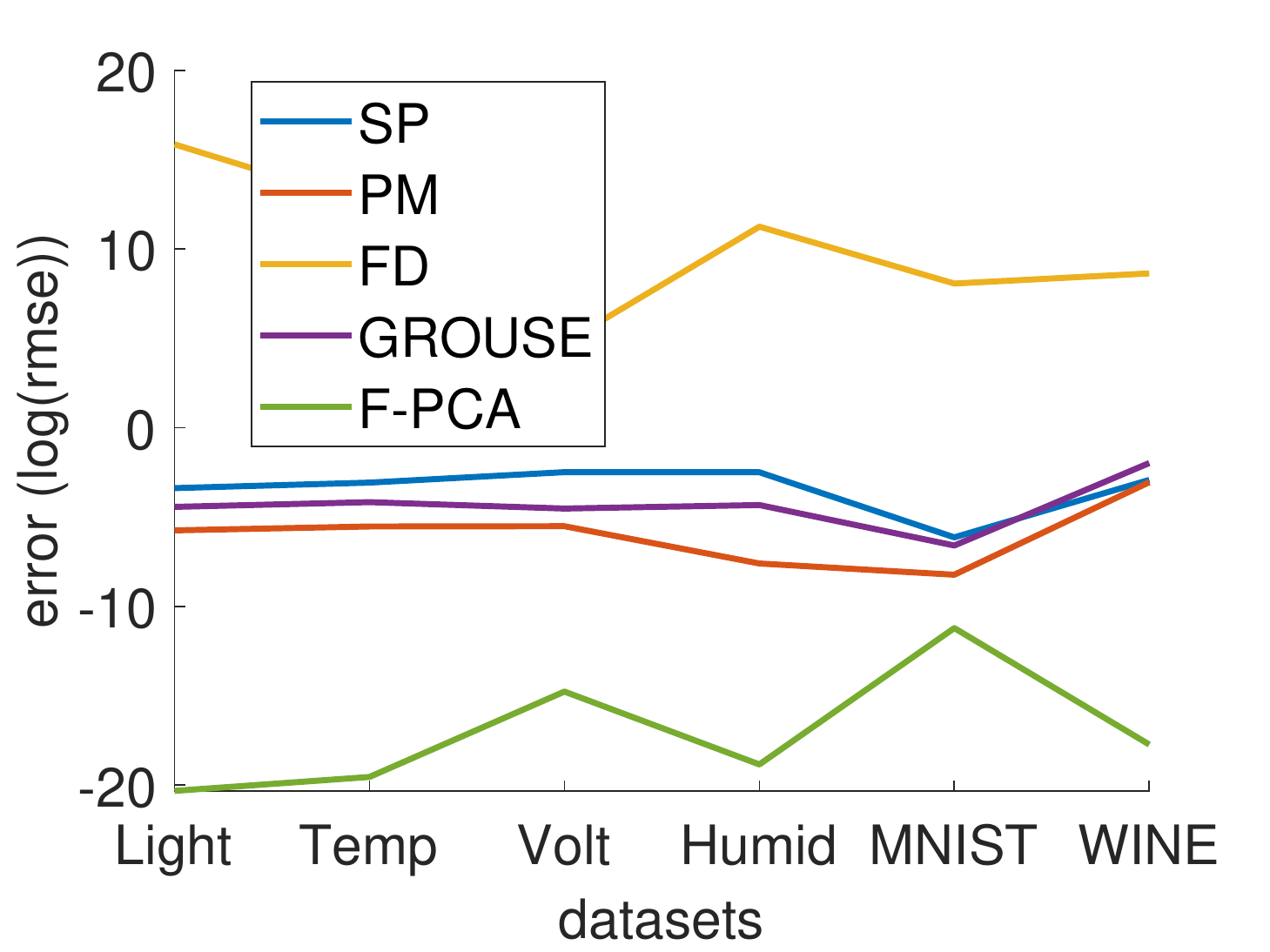}
        \caption{Errors on real datasets}
        \label{fig:final_subspace_fro_mse_errors_real}
    \end{subfigure}
    \begin{subfigure}{.32\linewidth}
        \centering
        \includegraphics
        [scale=0.3]
        {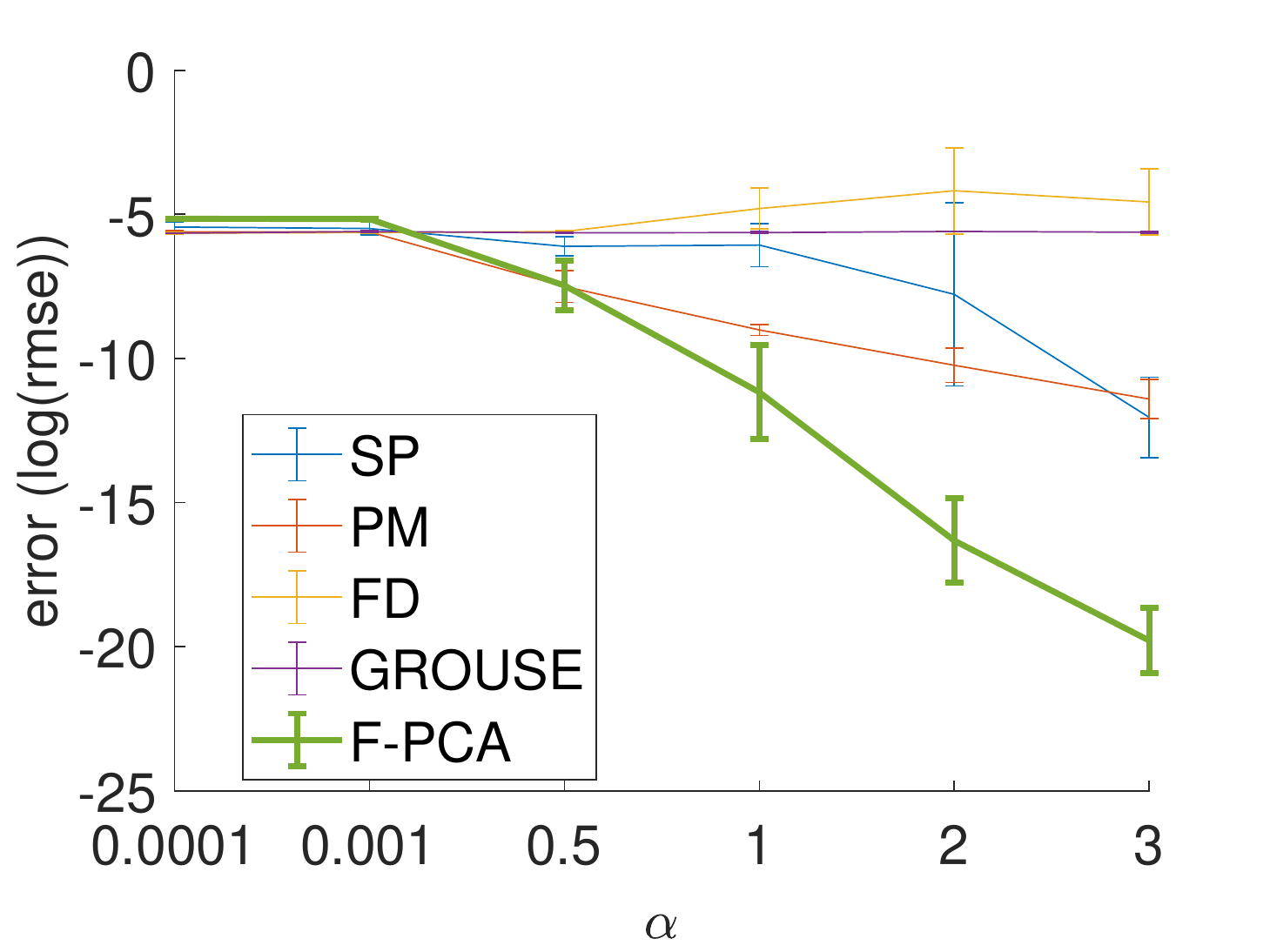}
        \caption{Errors on synthetic datasets}
        \label{fig:final_subspace_fro_mse_errors}
    \end{subfigure}
    \begin{subfigure}{.32\linewidth}
        \centering
        \includegraphics
        [scale=0.3]
        {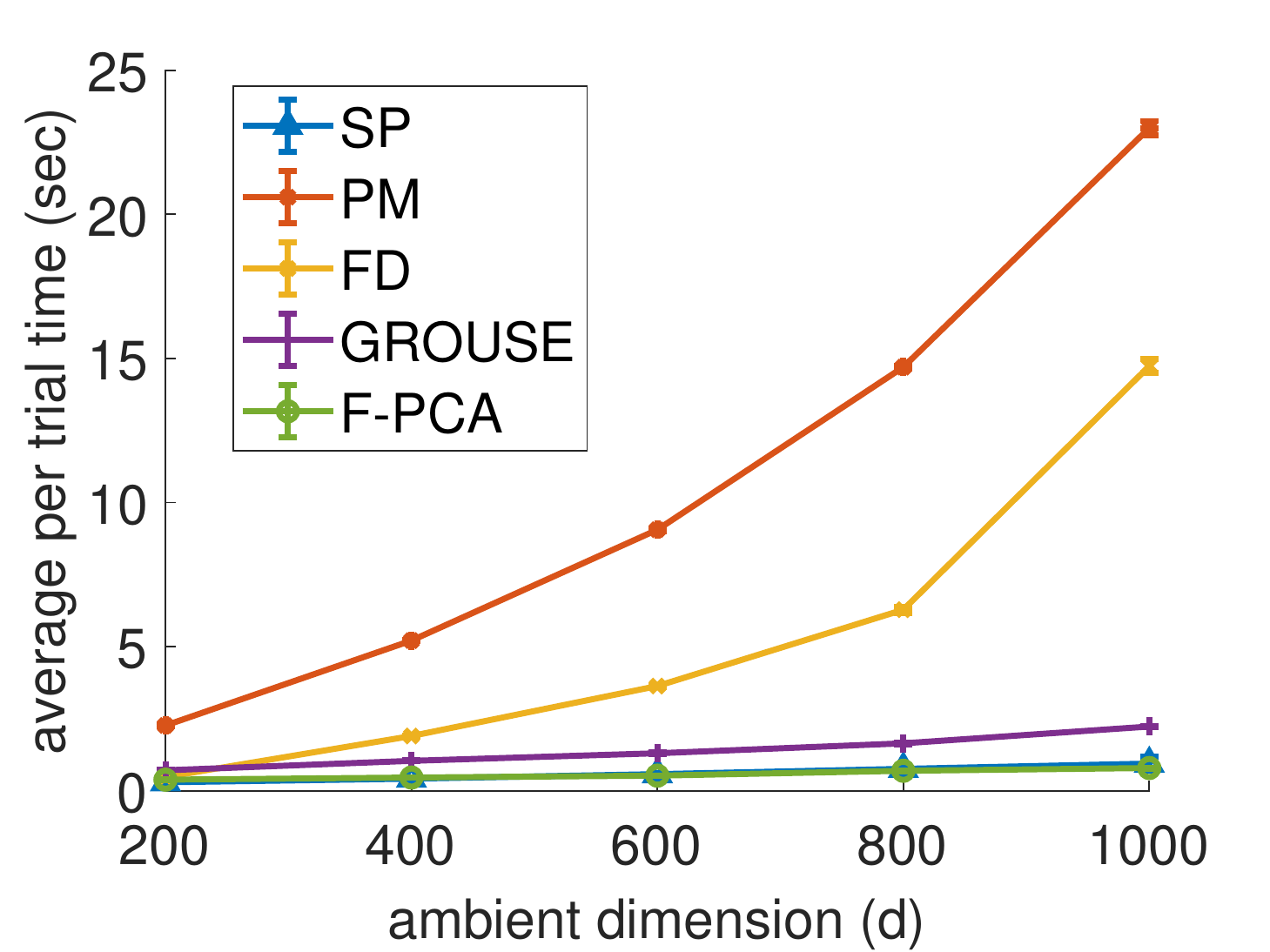}
        \caption{Average execution time}
        \label{fig:fpca_speed_exec_eval}
    \end{subfigure}\\
    \begin{subfigure}{.32\linewidth}
        \centering
        \includegraphics[scale=0.25]{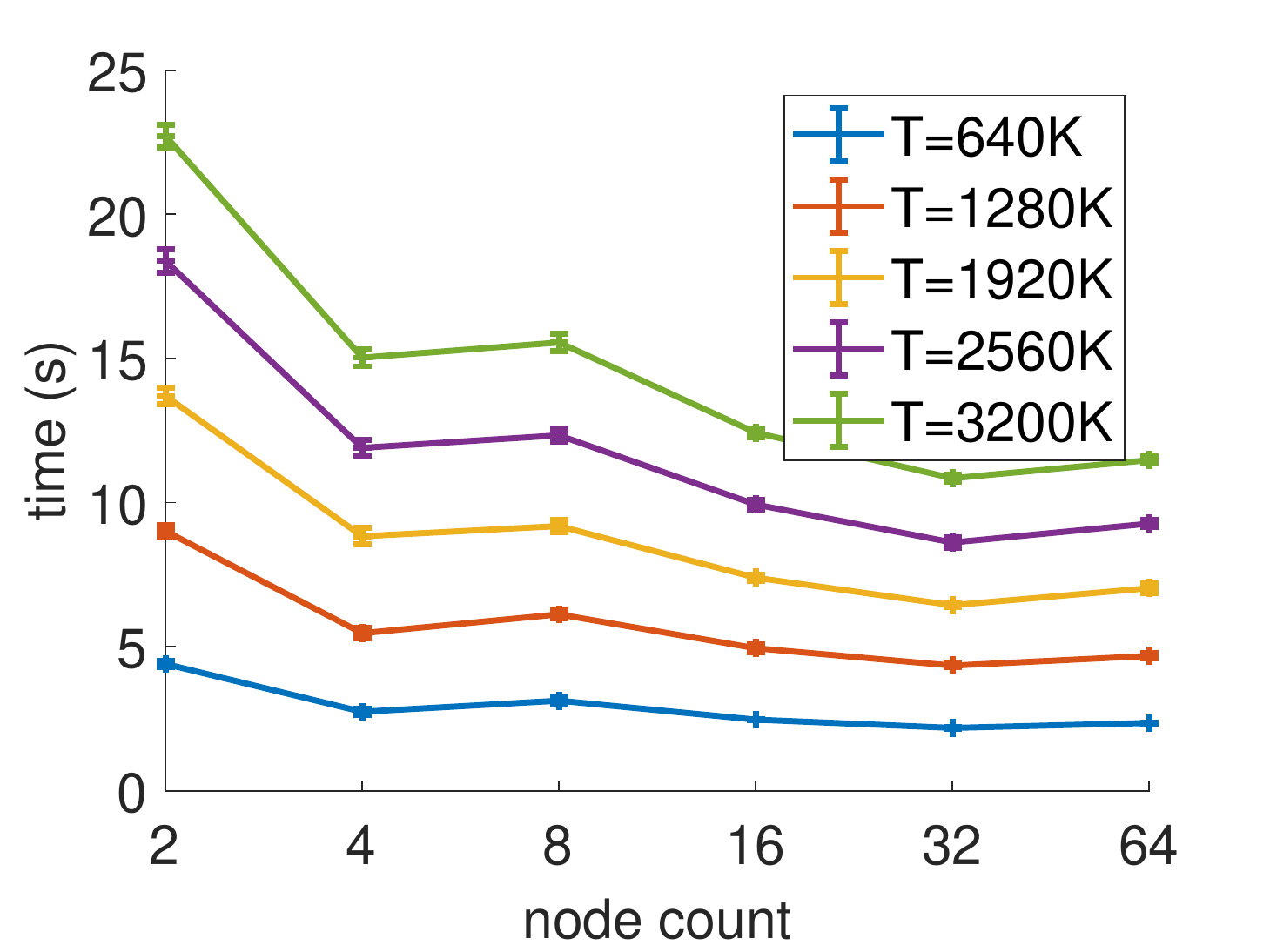}
        \caption{FPCA: Total time}
        \label{fig:fed_scaling}
    \end{subfigure}
    \begin{subfigure}{.32\linewidth}
        \centering
        \includegraphics[scale=0.25]{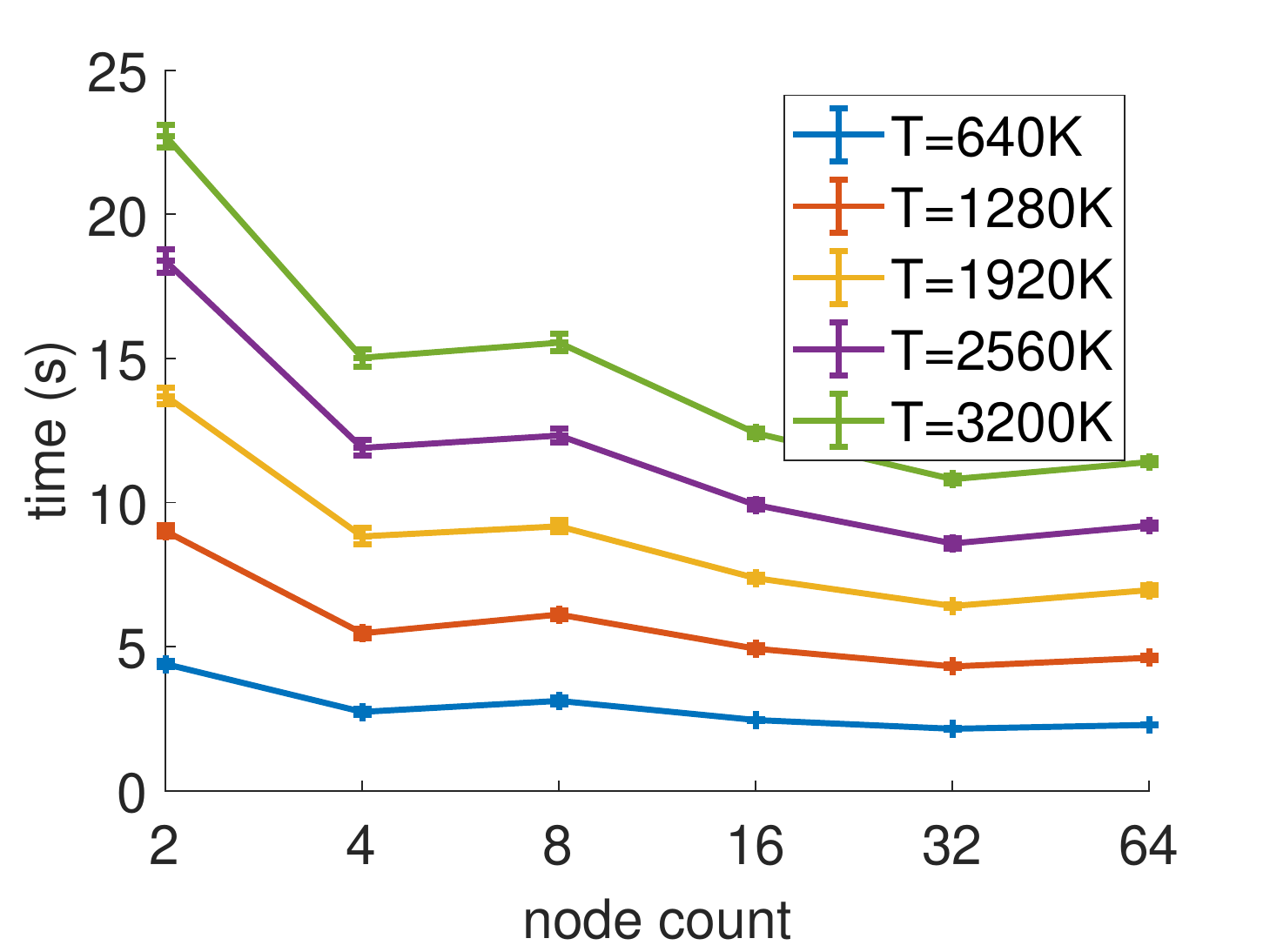}
        \caption{FPCA: PCA computation time}
        \label{fig:fed_pca_time}
    \end{subfigure}
    \begin{subfigure}{.32\linewidth}
        \centering
        \includegraphics[scale=0.25]{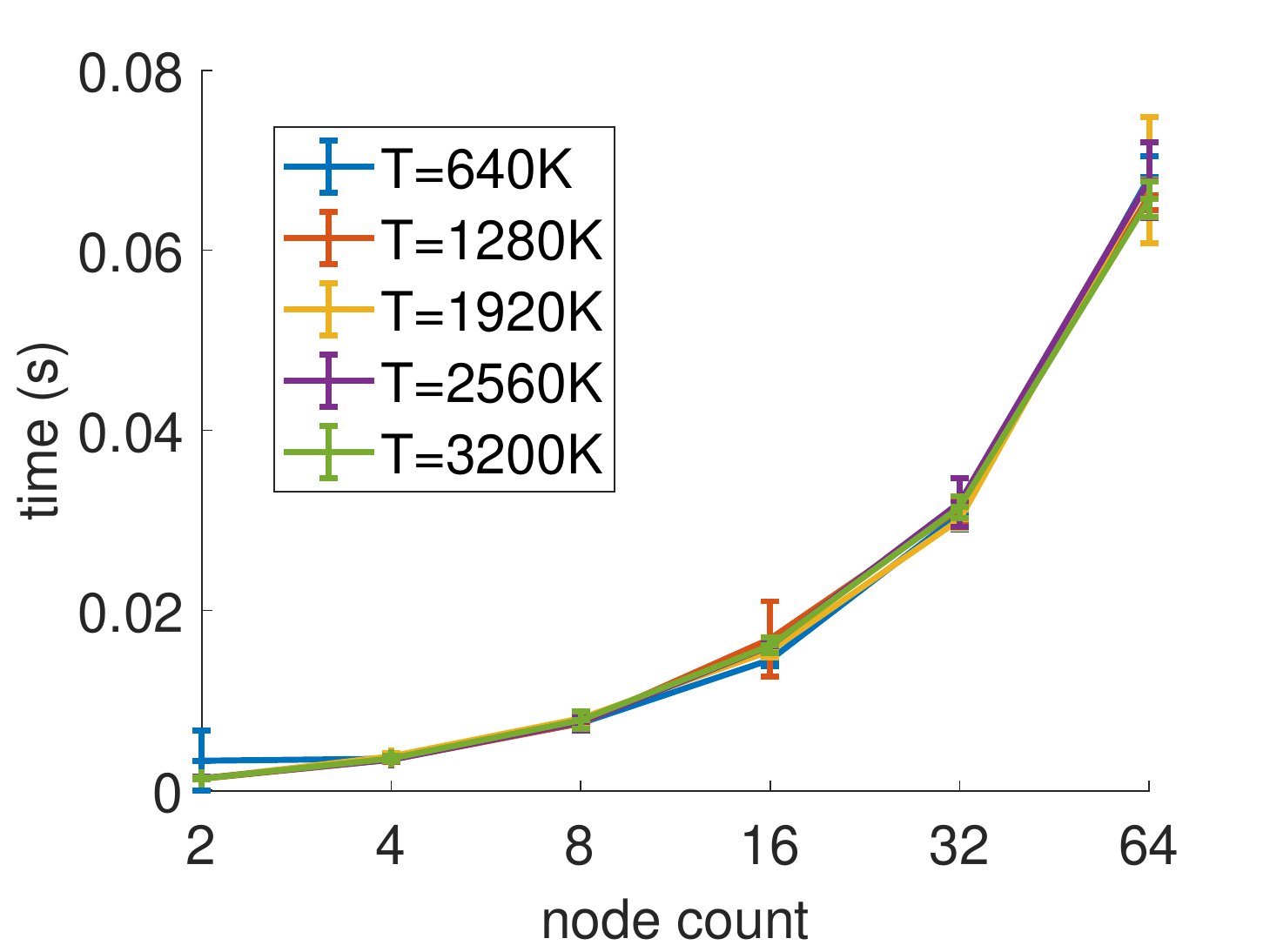}
        \caption{FPCA: Merging time}
        \label{fig:fed_merge_time}
    \end{subfigure}
    \caption{(a)-(c) Approximation and execution benchmarks against other streaming algorithms for a single-node network and without DP masks, (d)-(f) Computational scaling of $\FPCAC$ on multi-node networks with binary-trees of depth $\ell=\log_2(\mbox{node count})$.}
    \label{fig:fed_pca_main_text_performance}
\end{figure}

Figs.~\ref{fig:fed_scaling}, \ref{fig:fed_pca_time}, \ref{fig:fed_merge_time}
show the evaluation of $\FPCAC$ in a simulated federated computation environment. 
Specifically, they show the average execution times required to compute $\PCA$ on a dataset $\Y_{\alpha} \sim \text{Synth}(\alpha)^{d \times N}$
when fixing $d=10^3$ and varying $n\in\{640\text{k},1.28\text{M},1.92\text{M},2.56\text{M},3.2\text{M}\}$.
Fig.~\ref{fig:fed_scaling} shows the total computation time of the federated computation, while Figs.~\ref{fig:fed_pca_time} and~\ref{fig:fed_merge_time} show respectively the time spent computing PCA, and merging subspaces.
Fig.~\ref{fig:fed_scaling} shows a regression after exceeding the number of physical cores in our machine.
However, the amortised cost shows that with sufficient resources the federation can scale horizontally. 
More details can be found in Appendix~\ref{apx:fed_eval_details}.

\section{Discussion \& Conclusions}
\label{discussion_conclusions}
In this work, we introduced a federated streaming and differentially private algorithm for computing $\PCA$.
Our algorithm advances the state-of-the-art from several fronts: It is time-independent, asynchronous, and differentially-private.
DP is guaranteed by extending the results in \citep{chaudhuri2012near} to the streaming and non-symmetric setting. We do this while preserving the same nearly-optimal asymptotic guarantees provided by $\modsulq$.
Our algorithm is complemented with several theoretical results that guarantee bounded estimation errors and robustness to permutations in the data.
We have supplemented our work with a wealth of numerical experiments that show that shows that $\FPCA$ compares favourably against other methods in terms of convergence, bounded estimation errors, and low memory requirements.
An interesting avenue for future work is to study Federated $\PCA$ in the setting of missing values while preserving differential privacy.

\section{Broader Impact}

PCA is an ubiquitous and fundamental tool in data analysis and machine learning pipelines and also has important societal applications like {\em poverty measurement}.
Computing PCA on large-scale data is not only challenging from the computational point of view, but also from the {\em public policy} point of view. Indeed, new regulations around data ownership and privacy like GDPR have imposed restrictions in data collection and storage.
Our work allows for large-scale decentralised computation of PCA in settings where each compute node
- be it large (servers), thin (mobile phones), or super-thin (cryptocurrency blocks) - contributes in an independent an asynchronous way to the training of a global model, while ensuring the ownership and privacy of the data.
However, we note that our algorithmic framework is a {\em tool} and, like all tools, is subject to
misuse.
For example, our framework could allow malicious users to extract embeddings out of user data to be used for surveillance, user fingerprinting, and many others not so desirable use-cases.
We firmly believe, however, that the positives outweigh the negatives and this work has the potential to unlock information from decentralised datasets for the benefit of society, all while guaranteeing high-quality outputs and stringent privacy properties.

\section{Acknowledgements}

This work was supported by The Alan Turing Institute under grants: TU/C/000003, {TU/B/000069}, and {EP/N510129/1}.

\bibliographystyle{plain}
\bibliography{refs}

\clearpage
\newpage

\appendix

\section*{Supplementary Material}

This comes as supplementary material to the paper \textit{Federated Principal Component Analysis}. %
The appendix is structured as follows: 
\begin{enumerate}
    \item $\FPCA$'s local update guarantees,
    \item $\FPCA$'s differential privacy properties,
    \item In-depth analysis of algorithm's federation,
    \item Additional evaluation and discussion.
\end{enumerate}
Furthermore, we complement our theoretical analysis with additional empirical evaluation on synthetic and real datasets which include details on memory consumption. 

\section{Local Update Guarantees}
\label{apx:local_update_guarantees}

We note that the local updating procedure in Algorithm \ref{alg:fpca_edge} inherits some theoretical guarantees from~\citep{eftekhari2019moses}.
We leverage on these to provide a bound for the adaptive case.
Specifically, let $\mu$ be an {\em unknown} probability distribution supported on $\R^d$ with zero mean.
The informal objective is to find an $r$-dimensional subspace $\U$ that provides the {\em best approximation} with respect to the mass of $\mu$. 
That is, provided that $y$ is drawn from $\mu$, the target is to find an 
$r$-dimensional subspace $\U$ that minimises the \emph{population risk}.
This is done by solving 
\begin{equation}
\label{eq:direct}
\min_{\U\in\GR(d,r)} \underset{\mathbf{y}\sim \mu}{\E}\l\| \mathbf{y} - \P_{\U} \mathbf{y}\r\|_2^2
\end{equation} 
where the Grassmanian $\mathbb{G}(d,r)$ is the manifold of all $r$-dimensional subspaces in $\R^d$ and $\P_{\U} \in \R^{d \times d}$ is the orthogonal projection onto $\U$.
Unfortunately, the value of $\mu$ is 
unknown and cannot be used to directly solve \eqref{eq:direct}, but provided we have access 
to a block of samples $\{ \mathbf{y}_{t} \}^{\tau}_{t=1} \in \R^{d}$ 
that are independently drawn from $\mu$, then \eqref{eq:direct} can be
reformulated using the \textit{empirical risk} by
\begin{equation}
\label{eq:empirical}
 \min_{\U\in\GR(d,r)} \frac{1}{\tau}\sum_{t=1}^{\tau} \l\| \mathbf{y}_t - \P_{\U}\mathbf{y}_t \r\|_2^2.
\end{equation}
Given that $\sum_{t=1}^{\tau}\|\mathbf{y}_t - \P_{\U}\mathbf{y}_t\|_2^2 = \|\Y_{\tau} - \P_{\U}\Y_{\tau}\|_F^2$, it follows by the EYM Theorem \citep{eckart,mirsky}, that $\P_{\U}\Y_{\tau}$ is the {\em best} rank-$r$ approximation to $\Y_{\tau}$ which is given by $\hat{\Y}_{\tau} = \SVD_r(\Y_{\tau})$.
Therefore, $\mathcal{U} =  \operatorname{span}(\hat{\mathbf{Y}}_{\tau})$, which implies that $\|\Y_{\tau} - \P_{\mathcal{U}}\Y_{\tau}\|_F^2 = \|\Y_{\tau} - \hat{\Y}_{\tau}\|_F^2 = \rho_r^2(\Y_{\tau})$, so the solution of \eqref{eq:empirical} equals $\rho_r^2(\Y_{\tau})/\tau$. For
completeness the theorem is shown below.

\begin{thm}[\citep{eftekhari2019moses}]
\label{thm:stochastic-result}
Suppose $\{\mathbf{y}_t\}_{t=1}^{\tau} \subset\R^{d}$  are 
independently drawn from a zero-mean Gaussian distribution 
with covariance matrix $\B{\Xi}\in\R^{d \times d}$
and form $\Y_{\tau} = \left[\mathbf{y}_1 \cdots \mathbf{y}_{\tau}\right]\in \R^{d \times \tau}$.
Let $\lambda_{1}\ge \cdots \ge \lambda_{d}$ be the eigenvalues of $\B{\Xi}$ and $\rho_r^2=\rho_r^2(\Xi)$ be its residual.
Define 
{\footnotesize\begin{equation}
\rate_r = \frac{\lambda_1}{\lambda_r} +  \sqrt{\frac{2\alpha \rho_r^2}{p^{\frac{1}{3}}\lambda_r}},
\label{eq:spect of guassian thrm}
\end{equation}}
Let $\wh{\Y}_{\tau}$ be defined as in \eqref{eq:svdr_1}, $\mathcal{U} = \operatorname{span}(\wh{\Y}_{\tau})$ and $\alpha,p, c$ be constants such that $1\le \alpha\le \sqrt{\tau/\log \tau}$, $p>1$ and $c>0$.
Then, if $b \geq \max(\alpha p^{\frac{1}{3}}r({p^{\frac{1}{6}}}-1)^{-2}, c\alpha r)$ and $\tau \geq p \eta_r^2 b$, it holds, with probability at most $\tau^{-c\alpha^2}+ e^{-c\alpha r}$ that 
\begin{equation*}
\frac{\| \Y_{\tau} - \wh{\Y}_{\tau} \|_{F}^{2}}{\tau} \;\lesssim \;G_{\alpha, b, p, r, \tau}
\end{equation*}
\begin{equation*}
\underset{\mathbf{y}\sim \mu}{\E} \|\mathbf{y} - \P_{{\mathcal{U}}} \mathbf{y}\|_2^2 \;\lesssim \; G_{\alpha, b, p, r, \tau} + \alpha (d-r)\lambda_1 \sqrt{\frac{ {\log \tau} }{\tau}}
\end{equation*}

where 

{\footnotesize
\[
G_{\alpha, b, p, r, \tau}
=
\frac{\alpha p^{\frac{1}{3}} 4^{p\eta_r^2} }{(p^{\frac{1}{3}}-1)^2} \; \min\l (   \frac{\lambda_1}{\lambda_r} \rho_r^2,  r\lambda_1+\rho_r^2\r) \; \l(\frac{\tau}{p\eta^2_r b} \r)^{p\eta_r^2-1}
\]
}
\end{thm}
The condition $\tau\ge p\eta_r^2 b$ is only required to obtain a tidy bound and is not necessary in the general case. 
When considering only asymptotic dominant terms Theorem~\ref{thm:stochastic-result} reduces to,
\begin{equation}
\|\Y_{\tau} - \P_{\U}{\Y}_{\tau}\|_{F}^{2} \propto
\l(\frac{\tau}{b}\r)^{p\eta_{r}^{2}-1} \|\Y_{\tau} - \hat{\Y}_{\tau}\|_{F}^{2}
\label{eq:fpca-local-error}
\end{equation}
Practically speaking, assuming $\rank(\B{\Xi})\le r$ and $\rho_r^2(\Xi)  = \sum_{i=r+1}^d \lambda_i(\B{\Xi})$ we can read that, $\P_{\U}{\Y}_{\tau}=\hat{\Y}_{\tau}=\Y_{\tau}$ meaning that the outputs of offline truncated $\SVD{}$ and~\citep{eftekhari2019moses} coincide.

\subsection{Interpretation of each local worker as a streaming, stochastic solver for PCA}
It is easy to interpret each solver as a streaming, stochastic algorithm for Principal Component Analysis ($\PCA$). 
To see this, note that \eqref{eq:direct} is equivalent to maximising ${\E}_{y\sim\mu} \| \mathbf{U} \mathbf{U}^T \mathbf{y} \|_F^2$ over $\mathcal{Z} = \{\mathbf{U} \in \R^{d \times r} : \mathbf{U}^T\mathbf{U} = \mathbf{I}_{r \times r} \}$ 
The restriction $\mathbf{U}^T \mathbf{U} = \mathbf{I}_{r \times r}$ can be relaxed to $\mathbf{U}^T\mathbf{U}\preccurlyeq \mathbf{I}_r$, where $\mathbf{A} \preccurlyeq \mathbf{B}$ denotes that $\mathbf{B}-\mathbf{A}$ is a positive semi-definite matrix. 
Using the Schur's complement, we can formulate this program as
\begin{align}
\max \,\,\, \underset{y\sim\mu}{\E} \langle \mathbf{U} \mathbf{U}^T, \mathbf{y} \mathbf{y}^T \rangle  \nonumber \\
\operatorname{s.t.} 
\l[
\begin{array}{cc}
\mathbf{I}_n & \mathbf{U} \\
\mathbf{U}^T & \mathbf{I}_r
\end{array}
\r]
\succcurlyeq \mathbf{0}%
\label{eq:sdp formulation}
\end{align}
Note that, \eqref{eq:sdp formulation} has an objective function that is convex and that the feasible set is also conic and convex.
However, its gradient can only be computed when the probability measure $\mu$ is known, since otherwise $\B{\Xi}=\E [\mathbf{y}\mathbf{y}^T] \in\R^{d \times d}$ is unknown.
If $\mu$ is known, and an iterate of the form $\h{\B{S}}_{t}$ is provided, we could draw a random vector $\mathbf{y}_{t+1}\in \R^{d}$ from the probability measure $\mu$ while moving along the direction of $2\mathbf{y}_{t+1}\mathbf{y}_{t+1}^T\h{\B{S}}_t$.
This is because $\E [2\mathbf{y}_{t+1}\mathbf{y}_{t+1}^T\h{\B{S}}_t ]= 2\B{\Xi}\widehat{\B{S}}_t$ which is then followed by back-projection onto the feasible set $\mathcal{Z}$.
Namely,
\begin{equation}
\h{\B{S}}_{t+1} = \mathcal{P}\l(\B{S}_t + 2\alpha_{t+1} \mathbf{y}_{t+1}\mathbf{y}_{t+1}^T \h{\B{S}}_t  \r),
\label{eq:proj ball}
\end{equation}
One can see that in~\eqref{eq:proj ball}, $\mathcal{P}(\mathbf{A})$ 
projects onto the unitary ball of the spectral norm by clipping at one all of $\mathbf{A}$'s singular values exceeding one.

\subsection{Adaptive Rank Estimation}

Our algorithm provides a scheme to \textit{adaptively} adjust the rank of each individual estimation based on the distribution seen so far. This can be helpful when there are distribution shifts and/or changes in the data over time.
The scheme uses a thresholding procedure that consists in bounding the minimum and maximum contributions of $\sigma_r(\Y_{\tau})$ to the variance $\sum_{i=1}^r \sigma_{i}(\Y_{\tau})$ of the dataset. That is, by enforcing
\begin{equation}
     \energy_r^{\Y_{\tau}} = \frac{\sigma_{r}(\Y_{\tau})}{\sum_{i=1}^{r}\sigma_{i}(\Y_{\tau})}
     \in [ 
     \alpha, \beta],
     \label{eq:variance_bound_each_client}
\end{equation}
for some $\alpha, \beta > 0 $ and increasing $r$ whenever $\energy_r(\Y_{\tau}) >\beta$ or decreasing it when $\energy_r(\Y_{\tau}) < \alpha$. As a guideline, from our experiments a typical ratio of $\alpha/\beta$ should be less or equal to $0.2$ which could be used as an reference point when picking their values.
This ensure that each client will have a bounded Frobenius norm at any given point in time.
With this procedure, we are able to bound the global error as
\begin{equation}
 \rho_{r_{\max}(\alpha, \beta)}(\Y_{kb}) \leq \Y_{\text{err}} \leq \rho_{r_{\min}(\alpha, \beta)}(\Y_{kb}).
 \end{equation}
\begin{proof}
At iteration $k \in \{1, \dots, K\}$, each node computes $\hat \Y_{kb}^{\text{local}}$, the best rank-$r$ approximation of $\Y_{kb}$ using iteration \eqref{eq:svdr_1}.
 Hence, for each $k \in \{1, \dots, K\}$, the error of the approximation is given by ${\|\Y_{kb} - \hat \Y_{kb}^{\text{local}}\|_F = \rho_r(\Y_{kb})}$.
 Let $r_{\min} = r_{\min}(\alpha, \beta)$ and $r_{\max}=r_{\max}(\alpha, \beta) >0$ be the minimum and maximum rank estimates in when running $\FPCAC$. The result follows from

 \begin{equation*}
 \rho_{r_{\max}(\alpha, \beta)}(\Y_{kb}) \leq \Y_{\text{err}} \leq \rho_{r_{\min}(\alpha, \beta)}(\Y_{kb}).
 \end{equation*}

 Where $\Y_{\text{err}}=\|\Y_{kb} - \hat \Y_{kb}^{\text{local}}\|_F$
\end{proof}

Furthermore, we can express the global bound in a different form which can give us a more descriptive overall bound. 
To this end we know that for each local worker its $\| \cdot \|_{F}$ accumulated error any given time is bounded by the ratio of the summation of its singular values.
\vspace{12pt}

\begin{lemma}
  Let $\| \cdot \|_{F}^{M}\in \{1, \dots, M\}$ be the error accumulated for each of the $M$ clients at block $\tau$; then, after merging operations the global error will be $\sum_{i=1}^{M} \energy_M^{\Y_{\tau}}$. 
\end{lemma}

\begin{proof}
    By Equation~\eqref{eq:variance_bound_each_client} we know that the error is deterministically bounded for each of the $M$ clients at any given block $\tau$. 
    Further, we also know that the merging as in (\Cref{algorithm:fastest_subspace}) is able to merge the target subspaces with minimal error and thus at any given block $\tau$ we can claim that $\sum_{i=1}^{M} \energy_M^{\Y_{\tau}} + c_{m}$ where $c_{m}$ is a small constant depicting the error accumulated during the merging procedure of the subspaces, thus when asymptotically eliminating the constant factors the final error is $\sum_{i=1}^{M} \energy_M^{\Y_{\tau}}$. 
\end{proof}

\section{Privacy Preserving Properties of Federated PCA}

In this section we prove Lemma \ref{lemma:diffprivacy}, which summarises the differential privacy properties of our method.
The arguments are based on the proofs given by \cite{chaudhuri2012near}.
Lemma \ref{lemma:dpours} proves the first part of Lemma \ref{lemma:diffprivacy} by extending $\modsulq$ to the case of  non-symmetric noise matrices.
The second part of Lemma \ref{lemma:diffprivacy} is a direct corollary of Lemma \ref{lemma:dpours}. The third part follows directly from Lemmas \ref{lemma:utility} and \ref{lemma:samplecomplexity}.
\begin{lemma}[Differential privacy]
\label{lemma:dpours}
Let $\B{X} \in \R^{d \times n}$ be a dataset with orthonormal columns and $\B{A}= \frac{1}{n}\B{X}\B{X}^T$. Let
\begin{equation}
\omega(\varepsilon, \delta, d, n) = \frac{4d}{\varepsilon n} \sqrt{2 \log \left(\frac{d^2}{\delta \sqrt{2\pi}}\right)} + \frac{\sqrt{2}}{\sqrt{\varepsilon} n},
\end{equation}
and $\mathbf{N}_{\varepsilon, \delta, d, n} \in \R^{d \times d}$ be a non-symmetric random Gaussian matrix with i.i.d. entries drawn from $\mathcal{N}(0, \omega^2)$. Then, the principal components of $\frac{1}{n} \X\X^T + \mathbf{N}_{\varepsilon, \delta, d, n}$ are $(\varepsilon, \delta)$-differentially private.
\end{lemma}

\begin{proof}
Let $\B{N}, \hat{\B{N}} \in \R^{d \times d}$ be two random matrices such that $\B{N}_{i,j}$ and $\hat{\B{N}}_{i,j}$ are i.i.d. random variables drawn from  $\mathcal{N}(0, \omega^2)$.
Let $\mathcal{D} = \{\mathbf{x}_i : i \in [n]\} \subset \R^d$ be a dataset and let $\hat{\mathcal{D}} = \mathcal{D} \cup \{\hat{\mathbf{x}}_n\} \setminus \{\mathbf{x}_n\}$. Form the matrices
\begin{align}
\B{X} &= [\mathbf{x}_1, \dots, \mathbf{x}_{n-1}, \mathbf{x}_n]\\
\hat{\B{X}} &= [\mathbf{x}_1, \dots, \mathbf{x}_{n-1},\hat{\mathbf{x}}_n].
\end{align}
Let $\B{Y}= [\mathbf{x}_1, \dots \mathbf{x}_{n-1}]$. Then, the covariance matrices for these datasets are
\begin{align}
\B{A} &= \frac{1}{n}[\B{Y}\B{Y}^T + \mathbf{x}_n \mathbf{x}_n^T]\\
\hat{\B{A}} &= \frac{1}{n}[\B{Y}\B{Y}^T + \hat{\mathbf{x}}_n \hat{\mathbf{x}}_n^T].
\end{align}
Now, let $\B{G} = \B{A} + \B{B}$ and $\hat{\B{G}} = \hat{\B{A}} + \hat{\B{B}}$ and consider the log-ratio of their densities at point $\B{H} \in \R^{d \times d}$.
\begin{align}
\log\frac{f_{\G}(\H)}{f_{\Ghat}(\H)} &=  \frac{1}{2\omega^2}\sum_{i,j=1}^d \left(-(\H_{i,j} - \B{A}_{i,j})^2 + (\H_{i,j} - \hat{\B{A}}_{i,j})^2\right) \nonumber\\
&=  \frac{1}{2\omega^2}\sum_{i,j=1}^d \left(\frac{2}{n}(\B{A}_{i,j}-\H_{i,j})(\hat{\mathbf{x}}_n\hat{\mathbf{x}}_n^T - \mathbf{x}_n\mathbf{x}_n^T)_{i,j} + \frac{1}{n^2}(\hat{\mathbf{x}}_n\hat{\mathbf{x}}_n^T - \mathbf{x}_n\mathbf{x}_n^T)_{i,j}^2\right) \nonumber\\
&=  \frac{1}{2\omega^2}\sum_{i,j=1}^d \left(\frac{2}{n}(\B{A}_{i,j}-\H_{i,j})(\hat{\mathbf{x}}_{n,i}\hat{\mathbf{x}}_{n,j} - \mathbf{x}_{n,i}\mathbf{x}_{n,j}) + \frac{1}{n^2}(\hat{\mathbf{x}}_{n,i}\hat{\mathbf{x}}_{n,j} - \mathbf{x}_{n,i}\mathbf{x}_{n,j})^2\right). \label{eq:logdensity}
\end{align}
Note that if $\mathbf{x}, \mathbf{y} \in \R^d$ are such that $\|\mathbf{x}\|=\|\mathbf{y}\| = 1$ are unit vectors, then
\begin{equation}
\sum_{i,j=1}^d (\mathbf{x}_i\mathbf{x}_j - \mathbf{y}_i\mathbf{y}_j)^2 \leq 4.
\end{equation}
Moreover,
\begin{align}
\sum_{i,j=1}^d (\hat{\mathbf{x}}_{n,i}\hat{\mathbf{x}}_{n,j} - \mathbf{x}_{n,i}\mathbf{x}_{n,j}) &\leq \sum_{i,j=1}^d |\hat{\mathbf{x}}_{n,i}\hat{\mathbf{x}}_{n,j}|  + \sum_{i,j=1}^d |\mathbf{x}_{n,i}\mathbf{x}_{n,j}|  \\
&\leq 2 \max_{\mathbf{z} : \|\mathbf{z}\|\leq 1} \sum_{i,j=1}^d \mathbf{z}_i \mathbf{z}_j\\
&\leq 2 \max_{\mathbf{z} : \|\mathbf{z}\|\leq 1} \|\mathbf{z}\|_1^2\\
&\leq 2 \max_{\mathbf{z} : \|\mathbf{z}\|\leq 1} (\sqrt{d} \|\mathbf{z}\|_2)^2\\
&\leq 2 d.
\end{align}
Using these observations to bound \eqref{eq:logdensity}, and using the fact that 
for any $\gamma \in \R$ the events $\{\forall\;i,j :\B{N}_{i,j} \leq \gamma \}$ and $\{\exists\;i,j : \B{N}_{i,j} > \gamma\}$ are complementary, we obtain that for any measurable set $\mathcal{S}$ of matrices,
\begin{equation}
\label{eq:probS}
\mathbb{P}(\B{G} \in \mathcal{S}) \leq \exp\left(\frac{1}{2\omega^2}\left(\frac{4}{n}d\gamma + \frac{4}{n^2}\right)\right) + \mathbb{P}(\exists \; i,j : \B{N}_{i,j} > \gamma). 
\end{equation}
Moreover, if $\gamma > \omega$, we can use the union bound with a Gaussian tail bound to obtain %
\begin{align}
\delta:=\mathbb{P}(\exists \; i,j : \B{N}_{i,j} > \gamma) &= \mathbb{P}\left(\bigcup_{i,j=1}^d \left\{ \B{N}_{i,j} > \gamma\right\}\right) \nonumber\\
&\leq \sum_{i,j=1}^d \mathbb{P}\left(\B{N}_{i,j} > \gamma\right) \nonumber\\
&\leq \sum_{i,j=1}^d \left(\frac{1}{\sqrt{2\pi}} e^{-\frac{\gamma^2}{2\omega^2}}\right) \nonumber\\
&= \frac{d^2}{\sqrt{2\pi}} e^{-\frac{\gamma^2}{2\omega^2}} \label{eq:deltabound}
\end{align}
Now, solving for $\gamma$ in \eqref{eq:deltabound} we obtain,
\begin{equation}
\label{eq:gammadf}
\gamma = \omega \sqrt{2 \log \left(\frac{d^2}{\delta\sqrt{2\pi}}\right)}
\end{equation}
Substituting \eqref{eq:gammadf} in \eqref{eq:probS} we can give an expression for $(\varepsilon, \delta)$-differential privacy by letting
\begin{equation}
\varepsilon = \frac{1}{2\omega^2}\left(\frac{4}{n}d\left(\omega \sqrt{2 \log \left(\frac{d^2}{\delta\sqrt{2\pi}}\right)}\right) + \frac{4}{n^2}\right).
\end{equation}
This yields a quadratic equation on $\omega$, which we can rewrite as
\begin{equation}
\label{eq:quadratic}
2\varepsilon \omega^2 - \frac{4}{n}d\left(\omega \sqrt{2 \log \left(\frac{d^2}{\delta\sqrt{2\pi}}\right)}\right)\omega - \frac{4}{n^2} = 0.
\end{equation}
Using the quadratic formula to solve for $\omega$ in \eqref{eq:quadratic} yields,
\begin{align*}
\omega &= \frac{2d}{\varepsilon n} \sqrt{2 \log \left(\frac{d^2}{\delta \sqrt{2\pi}}\right)} \pm \frac{2}{\varepsilon n} \sqrt{2d^2 \log \left(\frac{d^2}{\delta \sqrt{2\pi}}\right) + \frac{\varepsilon}{2}}\\
&\leq \frac{2d}{\varepsilon n} \sqrt{2 \log \left(\frac{d^2}{\delta \sqrt{2\pi}}\right)} + \frac{2}{\varepsilon n} \left(\sqrt{2 d^2\log \left(\frac{d^2}{\delta \sqrt{2\pi}}\right)} + \sqrt{ \frac{\varepsilon}{2}}\right)\\
&= \frac{4d}{\varepsilon n} \sqrt{2 \log \left(\frac{d^2}{\delta \sqrt{2\pi}}\right)} + \frac{\sqrt{2}}{\sqrt{\varepsilon} n}.
\end{align*}

\end{proof}

To prove the utility bound in Lemma \ref{lemma:utility} of Streaming $\modsulq$, we will Lemmas \ref{lemma:packing}, \ref{lemma:kl-gaussian}, and \ref{lemma:fano}.
\begin{lemma}[Packing result \citep{chaudhuri2012near}]
\label{lemma:packing}
For $\phi \in [(2 \pi d)^{-1/2}, 1)$, there exists a set $\mathcal{C} \subset \mathbb{S}^{d-1}$ with
\begin{equation}
\label{eq:packing}
|\mathcal{C}| = \frac{1}{8}\exp\left((d-1) \log \frac{1}{\sqrt{1 - \phi^2}}\right)
\end{equation}
and such that $|\langle \boldsymbol{\mu}, \mathbf{v} \rangle| \leq \phi$ for all $\boldsymbol{\mu}, \mathbf{v} \in \mathcal{C}$. 
\end{lemma}

\begin{lemma}[Kullback-Leibler for Gaussian random variables]
\label{lemma:kl-gaussian}
Let $\B{\Sigma}$ be a positive definite matrix and let $f$ and $g$ denote, respectively, the densities $\mathcal{N}(\mathbf{a}, \B{\Sigma})$ and $\mathcal{N}(\mathbf{b}, \B{\Sigma})$. Then,
\begin{equation}
\mathbf{KL}(f \mid\mid g) = \frac{1}{2}(\mathbf{a} - \mathbf{b})^T \B{\Sigma} (\mathbf{a}-\mathbf{b}).
\end{equation}
\end{lemma}
\begin{proof}
The proof follows directly by using the definition of the Kullback-Leibler divergence and simplifying. 
\end{proof}

\begin{lemma}[Fano's inequality \citep{yu1997assouad}]
\label{lemma:fano}
Let $\mathcal{R}$ be a set and $\Theta$ be a parameter space with a pseudo-metric $d(\cdot)$. Let $\mathcal{F}$ be a set of $r$ densities $\{f_1, \dots, f_r\}$ on $\mathcal{R}$ corresponding to parameter values $\{\theta_1, \dots, \theta_r\}$ in $\Theta$. Let $X$ have a distribution $f \in \mathcal{F}$ with corresponding parameter $\theta$ and let $\hat \theta(X)$ be an estimate of $\theta$. If for all $i,j$, $d(\theta_i, \theta_j) \geq \tau$ and $\mathbf{KL}(f_i \mid \mid f_j) \geq \gamma$, then
\begin{equation}
\label{eq:fano}
\max_j \mathbb{E}_j\left[d(\hat \theta, \theta_j)\right] \geq \frac{\tau}{2}\left(1 - \frac{\gamma + \log 2}{\log r}\right).
\end{equation}
\end{lemma}

We are now ready to give a bound on the utility for Streaming $\modsulq$. We note that the proof for Lemma \ref{lemma:utility} is identical as the one given in \citep{chaudhuri2012near} except for a few equations where the dimension of the object considered changes from $\frac{d(d+1)}{2}$ to $d^2$.
We also note that while the utility bound has the same functional form, it is not identical to the one given in \citep{chaudhuri2012near} since it depends on the value of $\omega = \omega(\varepsilon, \delta, d, n)$ given in Lemma \ref{lemma:diffprivacy}.

\vspace{12pt}

\begin{lemma}[Utility bounds]
\label{lemma:utility}
Let $d, n \in \mathbb{N}$ and $\varepsilon > 0$ be given and let $\omega$ be given as in Lemma \ref{lemma:diffprivacy}, so that the output of Streaming $\modsulq$ is $(\varepsilon, \delta)$ differentially private for all datasets $\X \in \R^{d \times n}$. 
Then, there exists a dataset with $n$ elements such that if $\hat{\mathbf{v}}_1$ denotes the output of the Streaming $\modsulq$ and $\mathbf{v}_1$ is the top eigenvector of the empirical covariance matrix of the dataset, the expected correlation $\langle \mathbf{v}_1, \hat{\mathbf{v}}_1\rangle$ is upper bounded,
\begin{equation}
\mathbb{E}\left[|\langle \mathbf{v}_1, \hat{\mathbf{v}}_1\rangle|\right] \leq \min_{\phi \in \Phi} \left(1 - \frac{1-\phi}{4}\left(1 - \frac{1/\omega^2 + \log 2}{ (d-1) \log \frac{1}{\sqrt{1 - \phi^2}} - \log 8}\right)^2\right)
\end{equation}
where
\begin{equation}
\Phi \in \left[ \max\left\{ \frac{1}{\sqrt{2\pi d}}, \sqrt{1 - \exp\left(- \frac{2 \log (8d)}{d-1}\right)}, \sqrt{1 - \exp\left( - \frac{2/\omega^2 + \log 256}{d-1}\right)}\right\}\right].
\end{equation}
\end{lemma}

\begin{proof}
Let $\mathcal{C}$ be an orthonormal basis in $\R^d$. Then, $|\mathcal{C}|=d$, so solving for $\phi$ in \eqref{eq:packing} yields
\begin{equation}
\phi = \sqrt{1 - \exp\left(-\frac{2 \log (8d)}{d-1}\right)}.
\end{equation}
For any unit vector $\boldsymbol{\mu}$ let $\A(\boldsymbol{\mu}) = \boldsymbol{\mu}\boldsymbol{\mu}^T + \mathbf{N}$ where $\mathbf{N}$ is a symmetric random matrix such that $\{\mathbf{N}_{i,j}: i \leq i \leq j \leq d\}$ are i.i.d. $\mathcal{N}(0, \omega^2)$ and $\omega^2$ is the noise variance used in the Streaming $\modsulq$ algorithm.
The matrix $\A(\boldsymbol{\mu})$ can be thought of as a jointly Gaussian random vector on $d^2$ variables. The mean and covariance of this vector is
\begin{align}
\mathbb{E}[\boldsymbol{\mu}] &= (\boldsymbol{\mu}^2_{1}, \dots, \boldsymbol{\mu}^2_{d}, \boldsymbol{\mu}_1 \boldsymbol{\mu}_2, \dots, \boldsymbol{\mu}_{d-1}\boldsymbol{\mu}_d, \boldsymbol{\mu}_2 \boldsymbol{\mu}_1, \dots, \boldsymbol{\mu}_d\boldsymbol{\mu}_{d-1}) \in \R^{d^2},\\
\mbox{Cov}[\boldsymbol{\mu}] &= \omega^2 \mathbf{I}_{d^2 \times d^2} \in \R^{d^2 \times d^2}.
\end{align}
For $\boldsymbol{\mu}, \boldsymbol{\nu} \in \mathcal{C}$, the divergence can be calculated using Lemma \ref{lemma:kl-gaussian} yielding
\begin{equation}
\label{eq:kl-utility}
\mathbf{KL}(f_{\boldsymbol{\mu}} \mid \mid f_{\boldsymbol{\nu}}) \leq \frac{1}{\omega^2}.
\end{equation}

For any two vectors $\boldsymbol{\mu}, \boldsymbol{\nu} \in \mathcal{C}$, we have that $|\langle \boldsymbol{\mu}, \boldsymbol{\nu} \rangle| \leq \phi$, so that $-\phi \leq - \langle \boldsymbol{\mu}, \boldsymbol{\nu}\rangle$. Therefore, 
\begin{align}
\|\boldsymbol{\mu} - \boldsymbol{\nu}\|^2 &= \langle \boldsymbol{\mu} - \boldsymbol{\nu}, \boldsymbol{\mu} - \boldsymbol{\nu} \rangle\\
&= \|\boldsymbol{\mu}\|^2 + \|\boldsymbol{\nu}\|^2 - 2\langle \boldsymbol{\mu}, \boldsymbol{\nu}\rangle\\
&= 2(1 - \langle \boldsymbol{\mu}, \boldsymbol{\nu} \rangle)\\
&\geq 2(1 - \phi) \label{eq:aux-eq-utility}.
\end{align}
From \eqref{eq:kl-utility} and \eqref{eq:aux-eq-utility}, the set $\mathcal{C}$ satisfies the conditions of Lemma \ref{lemma:fano} with $\mathcal{F} = \{f_{\boldsymbol{\mu}}: \boldsymbol{\mu} \in \mathcal{C}\}$, $r = K$ and $\tau = \sqrt{2(1 - \phi)}$, and $\gamma = 1/\omega^2$. Hence, this shows that for Streaming $\modsulq$,
\begin{equation}
\label{eq:25}
\max_{\boldsymbol{\mu} \in \mathcal{C}} \mathbb{E}_{f_{\boldsymbol{\mu}}}\left[\|\hat{\boldsymbol{v}}  - \boldsymbol{\mu}\|\right] \geq \frac{\sqrt{2(1 - \phi)}}{2}\left(1 - \frac{1/\omega^2 + \log 2}{\log K}\right)
\end{equation}
As mentioned in \citep{chaudhuri2012near} this bound is vacuous when the term inside the parentheses is negative which imposes further conditions on $\phi$. Setting $K = 1/\omega^2 + \log 2$, we can solve to find another lower bound on $\phi$:
\begin{equation}
\phi \geq \sqrt{1 - \exp\left(-\frac{2/\omega^2 + \log 256}{d-1}\right)}
\end{equation}
Using Jensen's inequality on the left hand side of \eqref{eq:25} yields
\begin{equation}
\max_{\boldsymbol{\mu} \in \mathcal{C}} \mathbb{E}_{f_{\boldsymbol{\mu}}}\left[2(1 - |\langle \hat{\mathbf{v}}, \boldsymbol{\mu}\rangle|)\right] \geq \frac{(1 - \phi)}{2}\left(1 - \frac{1/\omega^2 + \log 2}{\log K}\right)^2
\end{equation}
so there is a $\boldsymbol{\mu}$ such that
\begin{equation}
\label{eq:expect}
\mathbb{E}_{f_{\boldsymbol{\mu}}}\left[|\langle \hat{\mathbf{v}}, \boldsymbol{\mu}\rangle|\right] \leq 1 - \frac{(1 - \phi)}{4}\left(1 - \frac{1/\omega^2 + \log 2}{\log K}\right)^2.
\end{equation}
Now, consider the dataset $\mathbf{D} = [\boldsymbol{\mu} \cdots \boldsymbol{\mu}] \in \R^{d^2 \times n}$. 
This dataset has covariance matrix equal to $\boldsymbol{\mu} \boldsymbol{\mu}^T$ and has top eigenvector equal to $\mathbf{v}_1 = \boldsymbol{\mu}$.
The output of the algorithm Streaming $\modsulq$ applied to $\mathbf{D}$ approximates $\boldsymbol{\mu}$, so satisfies \eqref{eq:expect}.
Minimising this equation over $\phi$ yields the required result.
\end{proof}

\begin{lemma}[Sample complexity]
    \label{lemma:samplecomplexity}
    For $(\epsilon, \delta)$ and $d \in \mathbb{N}$, there are constants $C_1>0$ and $C_2>0$ such that with
    \begin{equation}
    n \geq C_1 \frac{d^{3/2}\sqrt{\log(d/\delta)}}{\varepsilon}\left(1 - C_2 \left(1 -  \mathbb{E}_{f_{\boldsymbol{\mu}}} \left[|\langle \hat{\mathbf{v}}, \boldsymbol{\mu}\rangle|\right]\right)\right),
    \end{equation}
    where $\boldsymbol{\mu}$ is the first principal component of the dataset $\mathbf{X} \in \R^{d \times n}$ and $\hat{\mathbf{v}}$ is the first principal component estimated by Streaming MOD-SULQ.
\end{lemma}
\begin{proof}
Using \eqref{eq:expect}, and letting $\mathbb{E}_{f_{\boldsymbol{\mu}}} \left[|\langle \hat{\mathbf{v}}, \boldsymbol{\mu}\rangle|\right] = \rho$, we obtain,
\begin{equation}
2 \sqrt{1 - \rho} \geq \min_{\phi \in \Phi} \sqrt{1 - \phi} \left(1 - \frac{1/\omega^2 + \log 2}{(d-1)\log \frac{1}{\sqrt{1 - \phi^2}} - \log 8}\right)
\end{equation}
Picking $\phi$ so that the fraction in the right-hand side becomes 0.5, we obtain,
\begin{equation}
\label{eq:aux-complx}
4 \sqrt{1 - \rho} \geq \sqrt{1 - \phi}.
\end{equation}
Moreover, as $d, n \rightarrow \infty$, this value of $\phi$ guarantee implies an asymptotic of the form
\begin{equation}
\log \frac{1}{\sqrt{1 - \phi^2}} \sim \frac{2}{\omega^2 d} + o(1).
\end{equation}
This implies that $\phi = \Theta(\omega^{-1} d^{-1/2})$, and by \eqref{eq:omega-streaming} that $\omega \gtrsim d^2(\varepsilon n)^{-2}\log(d/\delta)$.
Therefore, there exists $C > 0 $ such that $\omega^2 > C d^2 (n \varepsilon)^{-2} \log (d/\delta)$. Since $\phi = \Theta (\omega^{-1}d^{-1/2})$ we have that for some $D > 0$
\begin{equation}
\phi^2 \leq D \frac{n^2 \varepsilon^2}{d^3 \log (d/\delta)}.
\end{equation}
By \eqref{eq:aux-complx} we get
\begin{equation}
\label{eq:solveforn}
(1 - 16 (1 - \rho)) \leq D \frac{n^2 \varepsilon^2}{d^3 \log (d/\delta)}
\end{equation}
Solving for $n$ in \eqref{eq:solveforn} yields
\begin{equation}
n \geq C_1 \frac{d^{3/2}\sqrt{\log(d/\delta)}}{\varepsilon}(1 - C_2 (1 - \rho)),
\end{equation}
for some constants $C_1$ and $C_2$.
\end{proof}

\section{Federated PCA Analysis}
\label{apx:fed_pca_analysis}

In this section we will present a detailed analysis of $\FPCA$ in which we will describe the merging process in detail as well as  provide a detailed error analysis in the \textit{streaming} and \textit{federated} setting that is based is based on the mathematical tools introduced in~\citep{iwen2016distributed}.

\subsection{Asynchronous Independent Block based SVD}

We begin our proof by proving~\Cref{lemma:svd_partial_lemma} (Streaming partial $\SVD$ uniqueness) which applies in the absence of perturbation masks and is the cornerstone of our federated scheme.

\begin{proof}
Let the reduced $\SVD{}_{r}$ representation of each of the $M$ nodes at time $t$ be,

  \begin{align}
    \Y^{i}_{t} = \sum_{j=1}^r \B{u}_j^i \bm{\sigma}_j^i (\B{v}_j^i)^{T} = \hat{\B{U}}^{i}_{t}
    \hat{\B{\Sigma}}^{i}_{t} (\hat{\B{V}}^{i}_{t})^{T} , \quad i = 1,2,\ldots,M.
  \end{align}

We also know that each of the blocks $\Y^{i}_{t}\in [M]$ can be at most of rank $d$. 
Note that in this instance, the definition applies for only \textit{fully} materialised matrices; however, substituting each block of $\Y_{i}^{t}$ with our local updates procedure as in~\Cref{alg:fpca_edge} then will generate an estimation of the reduced $\SVD{}_{r}$ of that particular $\Y_{i}^{t}$ block with an error at most as in~\eqref{eq:fpca-local-error} subject to each update chunk being in $\R^{d\times b}$ with $b \ge \min \rank(\Y_t^i)$ $\forall i \in [M]$.

Now, let the singular values of $\Y_{t}$ be the positive square root of the eigenvalues of $\Y_{t} \Y_{t}^{T}$, where as defined previously $\Y_{t}$ is the data seen so far from the $M$ nodes; then, by using the previously defined streaming block decomposition of a matrix $\Y_{t}$ we have the following,
  \begin{align}
    \Y_{t} \Y^{T}_{t} = \sum_{i=1}^M \Y^{i}_{t} (\Y^{i}_{t})^{T}  
    = \sum_{i=1}^M  \hat{\B{U}}^{i}_{t} \hat{\B{\Sigma}}^{i}_{t} (\hat{\B{V}_{t}}^{i})^{T}   (\hat{\B{V}}_{t}^{i}) (\hat{\B{\Sigma}}^{i}_{t})^{T} (\hat{\B{U}}^{i}_{t})^{T}  
    = \sum_{i=1}^M  \hat{\B{U}}^{i}_{t} \hat{\B{\Sigma}}^{i}_{t} (\hat{\B{\Sigma}}^{i}_{t})^{T} (\B{U}^{i}_{t})^{T}
  \end{align}
Equivalently, the singular values of $\B{Z}_{t}$ are similarly defined as the square root of the eigenvalues of $\B{Z}_{t} \B{Z}^{T}_{t}$.
  \begin{align}
    \B{Z} \B{Z}^{T} &= \sum_{i=1}^M (\hat{\B{U}}^{i}_{t} \hat{\B{\Sigma}}^{i}_{t}) (\hat{\B{U}}^{i}_{t} \hat{\B{\Sigma}}^{i}_{t})^T
    = \sum_{i=1}^M  \hat{\B{U}}^{i}_{t} \hat{\B{\Sigma}}^{i}_{t} (\hat{\B{\Sigma}}^{i}_{t})^T (\hat{\B{U}}^{i}_{t})^T
  \end{align}
  
Thus $\Y_{t} \Y^{T}_{t} = \B{Z}_{t} \B{Z}^{T}_{t}$ at any $t$, hence the singular values of matrix $\B{Z}_{t}$ must surely equal to those of matrix $\Y_{t}$. 
Moreover, since the left singular vectors of both $\Y_{t}$ and $\B{Z}_{t}$ will be also eigenvectors of $\Y_{t} \Y^{T}_{t}$ and $\B{Z}_{t} \B{Z}^{T}_{t}$, respectively; then the eigenspaces associated with each - possibly repeated - eigenvalue will also be equal thus $\hat{\B{U}}_{t} = \hat{\B{U}}_{t}'\B{B}_{t}$.  
The block diagonal unitary matrix $\B{B}_{t}$ which has $p$ unitary blocks of size $p \times p$ for each repeated eigenvalue; this enables the singular vectors which are associated with each repeated singular value to be rotated in the desired matrix representation $\hat{\B{U}}_{t}$.
In case of different update chunk sizes per worker the result is unaffected as long as the requirement for their size ($b$) mentioned above is kept and their rank $r$ is the same.
\end{proof}

\subsection{Time Order Independence}
Further, a natural extension to Lemma 1 which is pivotal to a successful federated scheme is the ability to guarantee that our result will be the same regardless of the merging order in the case there are no input perturbation masks.  
\begin{lemma}[Time independence]
\label{lemma:fpca-time-independence}
Let $\Y \in \R^{d \times n}$. Then, if $\mathbf{P} \in \R^{n \times n}$ is a row permutation of the identity. Then, in the absence of input-perturbation masks, ${\FPCAC(\Y) = \FPCAC(\Y\mathbf{P})}$.
\end{lemma}
\begin{proof}
If $\Y = \mathbf{U}\mathbf{\Sigma}\mathbf{V}^T$ is the Singular Value Decomposition ($\SVD$) of $\Y$, then $\Y\mathbf{P} =\mathbf{U}\mathbf{\Sigma}\left(\mathbf{V}^T\mathbf{P}\right)$.
Since $\mathbf{V}' = \mathbf{P}^T\mathbf{V}$ is orthogonal, $\mathbf{U}\mathbf{\Sigma}(\mathbf{V}')^T$ is the $\SVD$ of $\Y\mathbf{P}$.
Hence, both $\Y$ and $\Y\mathbf{P}$ have the same singular values and left principal subspaces.
\end{proof}

Notably, by formally proving the above Lemmas we can now exploit the following important properties: i) that we can create a block decomposition of $\Y_{t}$ for every $t$ without fully materialising the block matrices while being able to obtain their $\SVD_{r}$ incrementally, and ii) that the result will hold regardless of the arrival order. 

\subsection{Subspace Merging}

In order to expand the result of Lemmas~\ref{lemma:svd_partial_lemma} and~\ref{lemma:fpca-time-independence} we must first present the full implementation of~\Cref{algorithm:basic_subspace}.
This algorithm is a direct consequence of Lemma~\ref{lemma:svd_partial_lemma}, with the addition of a forgetting factor $\lambda$ that only gives more weight to the {\em newer} subspace.

\begin{algorithm}
    \KwData{
        $\B{U}_{1} \in \R^{d \times r_{1}}$, first subspace\\
        $\B{\Sigma}_{1} \in \R^{r_{1} \times r_{1}}$, first subspace singular values\\
        $\B{U}_{2} \in \R^{d \times r_{2}}$, second subspace\\
        $\B{\Sigma}_{2} \in \R^{r_{2} \times r_{2}}$, second subspace singular values\\
        $r \in [r]$, , the desired rank $r$\\
        $\lambda_{1} \in (0, 1)$, forgetting factor\\
        $\lambda_{2} \geq 1$, enhancing factor\\
    }
    \KwResult{
        $\B{U}' \in \R^{d \times r}$, merged subspace,
        $\B{\Sigma}' \in \R^{r \times r}$, merged singular values
    }
    \SetKwProg{Fn}{Function}{ is}{end}
    \Fn{$\BasicMerge(\B{U}_{1}$, $\B{\Sigma}_{1}$, $\B{U}_{2}$, $\B{\Sigma}_{2}$, $\lambda_{1}$, $\lambda_{2})$}{
    $[\B{U'}, \B{\Sigma'}, \text{\textasciitilde}] \leftarrow \SVD_{r}([\lambda_{1} \B{U}_{1}\B{\Sigma}_{1}, \lambda_{2} \B{U}_{2}\B{\Sigma}_{2}])$\\
    }
    \caption{$\BasicMerge$ algorithm}
    \label{algorithm:basic_subspace}
\end{algorithm}

\subsubsection{Improving upon regular $\SVD$}
As per~\Cref{lemma:svd_partial_lemma} we are able to use this algorithm in order to merge two subspaces with ease, however there are a few things that we could improve in terms of speed.
Recall, that in our particular care we do not require $\B{V}^{T}$, which is computed by default when using $\SVD$; this incurs both computational and memory overheads.
We now show how we can do better in this regard. 

We start by deriving an improved version for merging, shown Algorithm~\ref{algorithm:faster_subspace}; notably, this algorithm improves upon the basic merge (Algorithm~\ref{algorithm:basic_subspace}) by exploiting the fact that the input subspaces are already \textit{orthonormal}.
In this case, we show how we can transform the Algorithm~\ref{algorithm:basic_subspace} to Algorithm~\ref{algorithm:faster_subspace}. 
The key intuition comes from the fact that we can incrementally update $\B{U}$ by using $\B{U}\leftarrow \B{Q}_{p} \B{U}_{R}$. 
To do this we need to first create a subspace basis which spans $\B{U_{1}}$ and  $\B{U_{2}}$, namely $\text{span}(\B{Q}_{p})=\text{span}([\B{U_{1}}, \B{U_{2}}])$. This is done by performing $[\B{Q}_p, \B{R}_p] = \QR([\lambda_{1} \B{U}_{1}\B{\Sigma}_{1}, \lambda_{2} \B{U}_{2}\B{\Sigma}_{2}])$ and use $\B{R}_{p}$ to perform an incremental update.
Additionally, it is often the case that the subspaces spanned by $\B{U_{1}}$ and $\B{U_{2}}$ to intersect; in which case the rank of $\B{Q}$ is less than the sum $r_{1}$ and $r_{2}$.
Typically, practical implementations of $\QR$ will permute $\B{R}$ pushing the diagonal zeros only after all non-zeros which preserves the intended diagonal shape in the upper left part of $\B{R}$.
However, this behaviour has no practical impact to our results; as in the event this occurs, $\B{Q}$ is always permuted accordingly to reflect this~\cite{strang2019linear}.
Continuing, we know that $\B{Q}_{p}$ is orthogonal but we are not finished yet since $\B{R}_{p}$ is not diagonal, so an extra $\SVD$ needs to be applied on it which yields the singular values in question and the rotation that $\B{Q}_p$ requires to represent the new subspace basis.
Unfortunately, even if this improvement, this technique only yields a marginally better algorithm since the $\SVD$ has to now be performed at a much smaller matrix, namely, $\B{R}_{p}$.

\begin{algorithm}[htb!]
    \KwData{
        $U_{1} \in \R^{d \times r_{1}}$, first subspace\\
        $\B{\Sigma}_{1} \in \R^{r_{1} \times r_{1}}$, first subspace singular values\\
        $\B{U}_{2} \in \R^{d \times r_{2}}$, second subspace\\
        $\B{\Sigma}_{2} \in \R^{r_{2} \times r_{2}}$, second subspace singular values\\
        $r \in [r]$, , the desired rank $r$\\
        $\lambda_{1} \in (0, 1)$, forgetting factor\\
        $\lambda_{2} \geq 1$, enhancing factor\\
    }
    \KwResult{
        $\B{U}' \in \R^{d \times r}$, merged subspace\\
        $\B{\Sigma}' \in \R^{r \times r}$, merged singular values
    }
    \SetKwProg{Fn}{Function}{ is}{end}
    \Fn{$\FasterMerge(\B{U}_{1}$, $\B{\Sigma}_{1}$, $\B{U}_{2}$, $\B{\Sigma}_{2}$, $\lambda_{1}$, $\lambda_{2}, $r$)$}{
    $[\B{Q}_{p}, \B{R}_{p}] \leftarrow \QR(\lambda_{1} \B{U}_{1}\B{\Sigma}_{1} ~ | ~ \lambda_{2}\B{U}_{2}\B{\Sigma}_{2})$\\
    $[\B{U}_{R}, \B{\Sigma}', \text{\textasciitilde}] \leftarrow \SVD_{r}(\B{R_{p}})$\\
    $\B{U}' \leftarrow \Q_{p} \B{U}_{R}$\\
    }
    \caption{$\FasterMerge$ algorithm}
    \label{algorithm:faster_subspace}
\end{algorithm}

Now we will derive our final merge algorithm by showing how \Cref{algorithm:faster_subspace} can be further improved when $\mathbf{V}^{T}$ is not needed and we have knowledge that $\mathbf{U}_{1}$ and $\mathbf{U}_{2}$ are already orthonormal. 
This is done by building a basis $\B{U}'$ for $\operatorname{span}((\B{I} - \B{U_1}\B{U_1}^T)\B{U_2})$ via the QR factorisation and then computing the $\SVD$ decomposition of a matrix $\B{X}$ such that  

\begin{equation}
    [\B{U_{1}}\B{\Sigma_{1}},\B{U_{2}}\B{\Sigma_{2}}] = [\B{U_{1}},\B{U}']\B{X}.
\end{equation}

It is shown in~\citep[Chapter 3]{vrehuuvrek2011subspace} in an analytical derivation that this yields an $\B{X}$ of the form

\begin{equation*}
    \B{X} = 
    \begin{bmatrix} 
        \B{U_{1}^{T}}\B{U_{1}}\B{\Sigma_{1}} &
        \B{U_{1}^{T}}\B{U_{2}}\B{\Sigma_{2}}\\
        \B{U'}^{T}\B{U_{1}} & \B{U'^{T}}\B{U_{2}}\B{\Sigma_{2}}
    \end{bmatrix}
    = 
    \begin{bmatrix} 
        \B{\Sigma_{1}} & \B{U_{1}^{T}}\B{U_{2}}\B{\Sigma_{2}}\\
        0 & \B{R_{p}}\B{\Sigma_{2}}
    \end{bmatrix}
\end{equation*}

The same technique appears to have been independently rediscovered in~\citep{eftekhari2019moses} as the merging procedure for each block is identical. 
The Algorithm~\ref{algorithm:fastest_subspace_apx} below shows the full implementation.
\begin{algorithm}[htb!]
\footnotesize{
    $\mbox{Merge}_r(\B{U}_1, \B{\Sigma}_1, \B{U}_2, \B{\Sigma}_2)$
    
    \KwData{ 
        
        $r \in [r]$, rank estimate\;
        $(\B{U}_{1}, \B{\Sigma}_1) \in \R^{d \times r_{1}} \times \R^{r_1 \times r_1}$, 1st subspace\;
        $(\B{U}_{2}, \B{\Sigma}_2) \in \R^{d \times r_{2}} \times \R^{r_2 \times r_2}$, 2nd subspace\;
    }
    \KwResult{
        $(\B{U}', \B{\Sigma}') \in \R^{d \times r} \times \R^{r \times r}$ merged subspace\;
    }
    \SetKwProg{Fn}{Function}{ is}{end}
        \Fn{$\Merge_{r}(\B{U}_{1}$, $\B{\Sigma}_{1}$, $\B{U}_{2}$, $\B{\Sigma}_{2})$}{
        $\B{Z} \leftarrow \B{U}^{T}_{1}\B{U}_{2}$\;
        $[\B{Q}, \B{R}] \leftarrow \QR(\B{U}_{2} - \B{U}_{1}\B{Z})$\;
        $[\B{U}_r,\B{\Sigma}', \thicksim] \leftarrow \SVD_{r}\bigg(
        \begin{bmatrix} 
            \B{\Sigma}_{1} & \B{Z}\B{\Sigma}_{2} \\
            0 & \B{R}\B{\Sigma}_{2}
        \end{bmatrix}
        \bigg)$\;
        $\B{U}' \leftarrow [\B{U}_{1}, \B{Q}]\B{U}_r$\;
    }
    }
    \caption{$\Merge_{r}$ \citep{vrehuuvrek2011subspace, eftekhari2019moses}} %
    \label{algorithm:fastest_subspace_apx}
\end{algorithm}

The algorithm shown above is the one of the essential components of our federated scheme, allowing us to quickly merge incoming subspaces as they are propagated upwards. 
To illustrate the practical benefits of the merging algorithm we conducted an experiment in order to evaluate if the algorithm performs as expected.
Concretely, we created synthetic data using $ \text{Synth}(1)^{d\times n}$ with $d=800$ and $n\in \{800, 1.6k, 2.4k, 3.2k, 4k\}$; then we split each dataset into two equal chunks each of which was processed using $\FPCA$ with a target rank of $100$. 
Then we proceeded to merge the two resulting subspaces with two different techniques, namely, with the~\Cref{eq:basic_subspace} and~\Cref{algorithm:fastest_subspace_apx} as well as find the offline subspace using traditionally $\SVD$. 
We then show in~\Cref{fig:subspace_merge_evaluation} the errors incurred with respect to the offline $\SVD$ against the resulting merged subspaces and singular values of the two techniques used, as well as their execution.
We can clearly see that the resulting subspaces are \textit{identical} in all cases and that the error penalty in the singular values is minimal when compared to~\cref{eq:basic_subspace}; as expected, we also observe that derived algorithm is faster while consuming less memory.
Critically speaking, the speed benefit is not significant in the single case as presented; however, these benefits can be additive in the presence of thousands of merges that would likely occur in a federated setting.

\begin{figure*}[htb!]
    \centering
    \begin{subfigure}{.31\linewidth}
        \centering
        \includegraphics[scale=0.28]{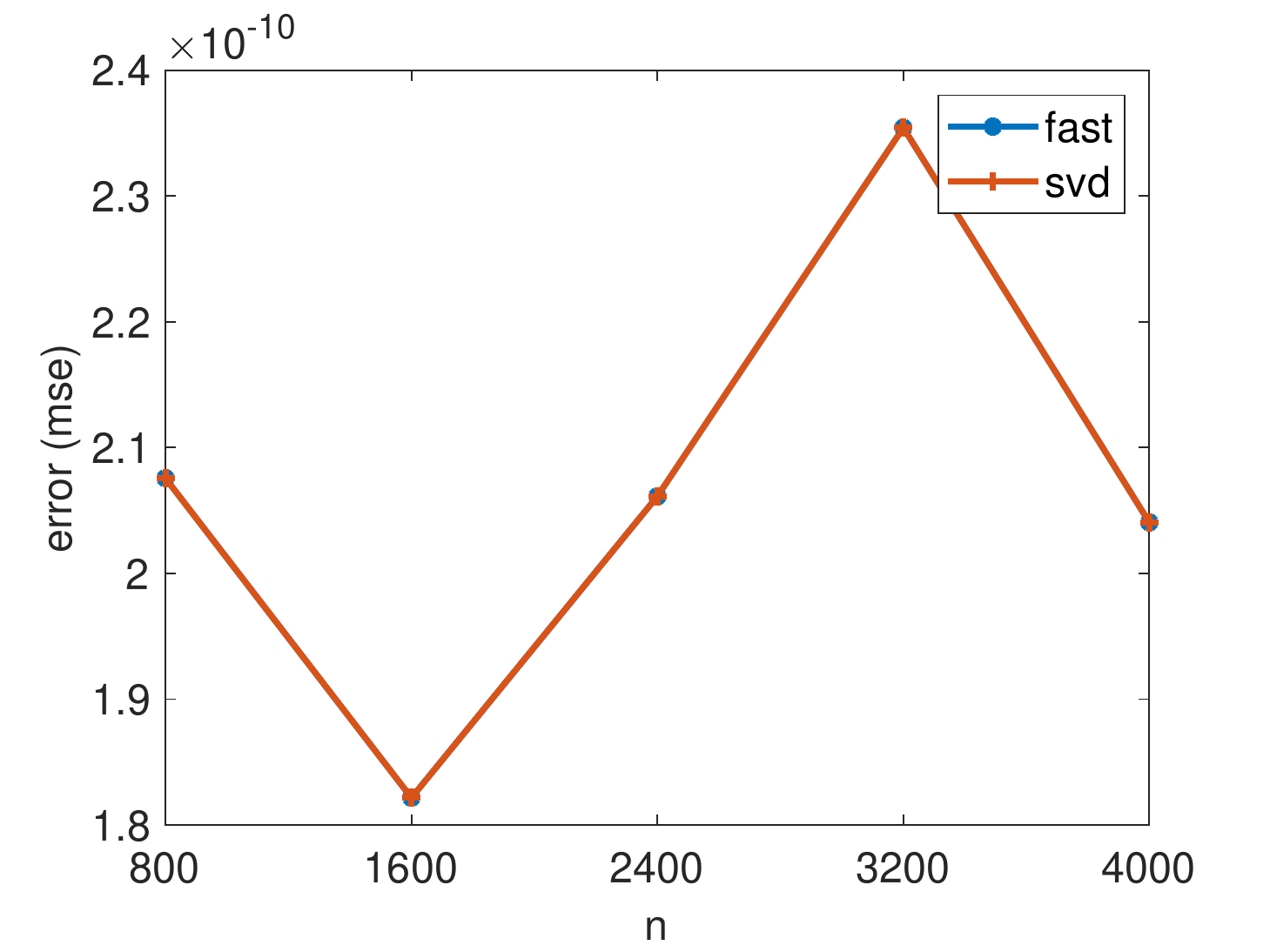}
        \caption{$\B{U}$ errors.}
        \label{fig:merge_u_errors}
    \end{subfigure}
    \begin{subfigure}{.31\linewidth}
        \centering
        \includegraphics[scale=0.28]{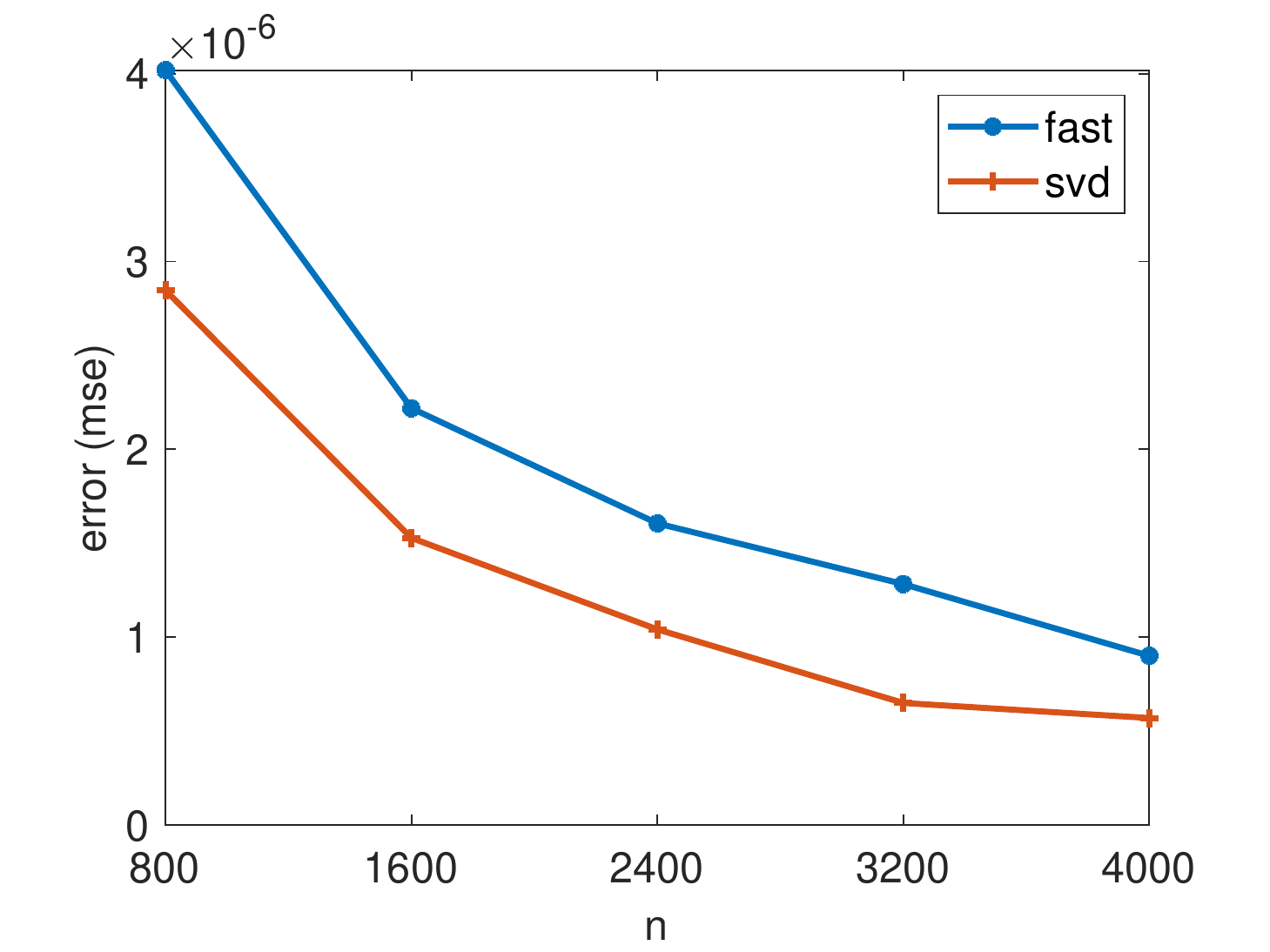}
        \caption{Singular Value errors.}
        \label{fig:merge_sigma_errors}
    \end{subfigure}
    \begin{subfigure}{.31\linewidth}
        \centering
        \includegraphics[scale=0.28]{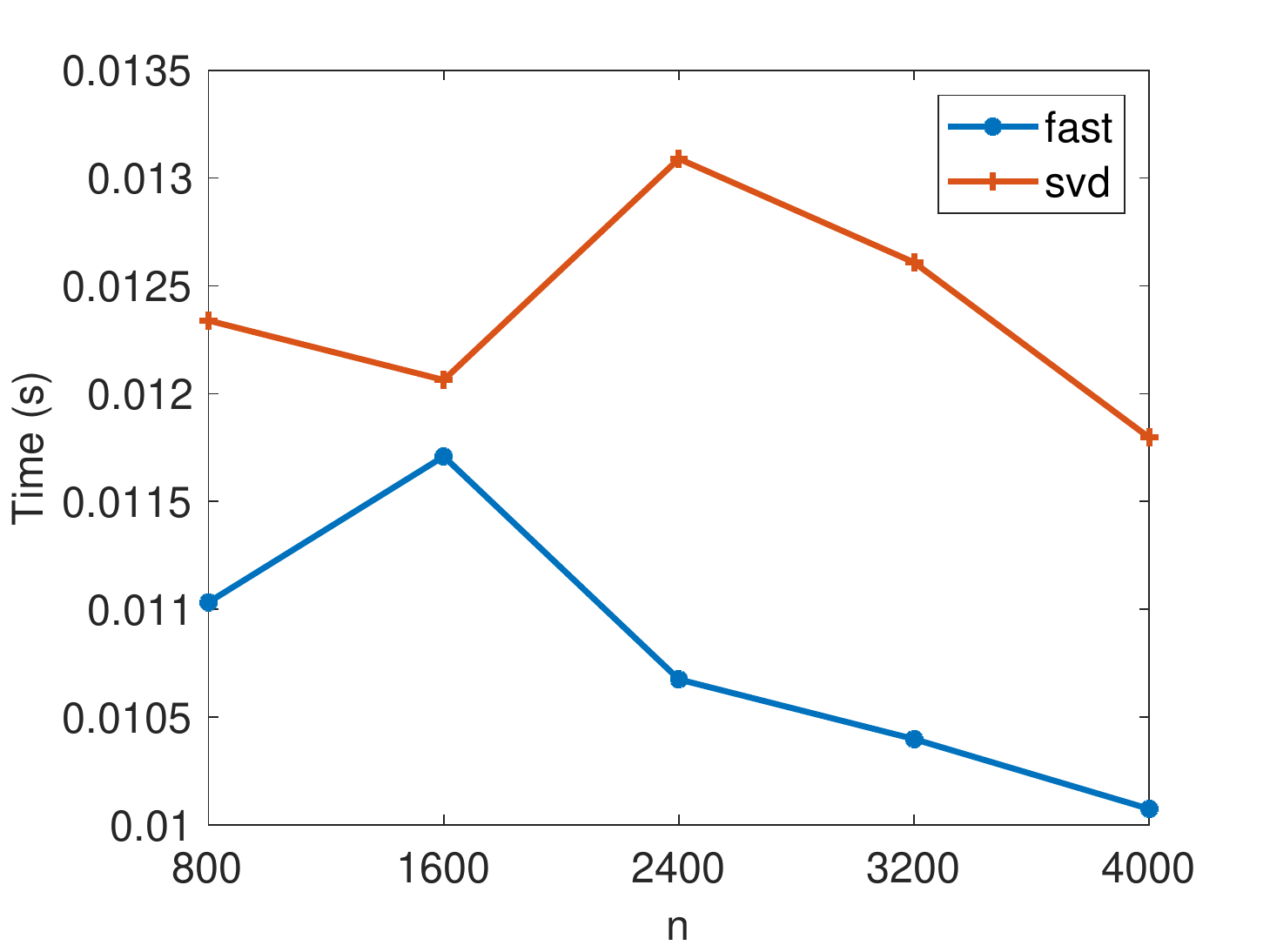}
        \caption{Execution time.}
        \label{fig:merge_speed}
    \end{subfigure}
    \caption{Illustration of the benefits of~\Cref{algorithm:fastest_subspace_apx}, in of errors of subspace (\cref{fig:merge_u_errors}), singular values (\cref{fig:merge_sigma_errors}), and its execution speed (\cref{fig:merge_speed}).}
    \label{fig:subspace_merge_evaluation}
\end{figure*}

\subsection{Federated Error Analysis}

In this section we will give a lower and a upper bound of our federated approach.
This is also  based on the mathematical toolbox we previously used~\cite{iwen2016distributed} but is adapted in the case of streaming block matrices. 

\begin{lemma}
  Let $\Y^{i}_{t} \in \mathbb{R}^{d\times tMb}, i=[M]$ for a any time $t$ and a fixed update chunk size $b$.
  Furthermore, suppose matrix $\Y^{i}_{t}$ at time $t$ has block matrices defined as
  $\Y^{i}_{t}=\left[\Y^{1}_{t}| \Y^{2}_{t}|\cdots|\Y^{M}_{t}\right]$, and $\B{Z_{t}}$ at the same time has blocks defined as $\B{Z}_{t}=\left[(\Y^{1}_{t})_r | (\Y^{2}_{t})_r |\cdots|(\Y^{M}_{t})_r \right]$, where $r \le d$.  
  Then, $\|(\B{Z}_{t})_r - \Y_{t} \|_{\rm F} \leq \| (\B{Z})_r - \B{Z}_{t} \|_{\rm F} +  \| \B{Z}_{t} - \Y_{t} \|_{\rm F} \leq 3 \| (\Y_{t})_r - \Y_{t} \|_{\rm F}$ holds for all $r \in [d]$.
  \label{lemma:low_ranking_merge}
\end{lemma}

\begin{proof}
We base our proof on an invariant at each time $t$ the matrix $\Y_{t}$, although not kept in memory, due to the approximation described in~\cref{apx:local_update_guarantees} can be treated as such for the purposes of this proof. 
Thus, we have the following:
\begin{align*}
\| (\B{Z}_{t})_r - \Y_{t} \|_{\rm F} &\leq \| (\B{Z}_{t})_r - \B{Z_{t}} \|_{\rm F} +  \| \B{Z}_{t} - \Y_{t} \|_{\rm F}\\ &\leq \| (\Y_{t})_r - \B{Z}_{t} \|_{\rm F} +  \| \B{Z}_{t} - \Y_{t} \|_{\rm F}\\
&\leq \| (\Y_{t})_r - \Y_{t} \|_{\rm F} + 2 \| \B{Z}_{t} - \Y_{t} \|_{\rm F}.
\end{align*}

We let $(\Y_{t}^{i})_{r} \in \mathbb{R}^{d\times tMb}, i=1,2,\ldots,M$ denote the $i^{\rm th}$ block of $(\Y_{t})_{r}$, we can see that

\begin{align*}
\| \B{Z}_{t} - \Y_{t} \|^2_{\rm F} = \sum^M_{i=1} \| (\Y_{t}^i)_d - \Y_{t}^i \|^2_{\rm F} \leq \sum^M_{i=1} \| (\Y_{t}^i)_r - \Y_{t}^i \|^2_{\rm F} = \| (\Y_{t})_r - \Y_{t} \|^2_{\rm F} .
\end{align*}
Hence, if we combine these two estimates we complete our proof.
\end{proof}

To bound the error of the federated algorithm, we use \Cref{lemma:low_ranking_merge} to derive a lower and an upper bound of
the error.
Suppose that we choose a $r\le d$ which is a truncated version of $\Y_{t}$ while also having the depth equal to $1$.  
We can improve over \Cref{lemma:low_ranking_merge} in this particular setting by requiring no access on the right singular vectors of any given block - \textit{e.g.} the $\B{V_{t}^{i}}^T$.
Furthermore, it is possible to also show that this method is stable with respect to (small) additive errors. We represent this mathematically with a noise matrix $\Psi$.

\vspace{12pt}

\begin{thm}
Let $\Y_{t} \in \mathbb{R}^{d\times tMb}$ at time $t$ has its blocks defined as $\Y^i_{t} \in \R^{d \times tMb}, i=[M]$, so that $\Y_{t}=\left[\Y_{t}^1|\Y_{t}^2|\cdots|\Y_{t}^M\right]$.
Now, also let $\B{Z_{t}}=\left[\overline{(\Y_{t}^1)_r} ~\big|~ \overline{(\Y_{t}^2)_r} ~\big|~ \cdots~\big|~ \overline{(\Y_{t}^M)_r} \right]$, $\Psi_{t} \in \R^{d \times tMb}$, and $\B{Z_{t}}' = \B{Z_{t}} + \Psi_{t}$.  Then, there exists a unitary matrix $\B{B}_{t}$ such that
$$\left\| \overline{\left( \B{Z_{t}}' \right)_r} - \Y_{t}\B{B_{t}}_{t} \right\|_{\rm F} \leq 3\sqrt{2} \| (\Y_{t})_r - \Y_{t} \|_{\rm F} + \left(1+ \sqrt{2} \right) \| \Psi_{t} \|_{\rm F}$$
holds for all $r \in [d]$.
\label{thm:low_rank_merge}
\end{thm}

\begin{proof}
Let $\Y_{t}' = \left[\overline{\Y_{t}^1} ~\big|~ \overline{\Y_{t}^2} ~\big|~\cdots ~\big|~ \overline{\Y_{t}^M}\right]$.  Note that $\overline{\Y_{t}'} = \overline{\Y_{t}}$ by~\Cref{lemma:svd_partial_lemma}.  
Thus, there exists a unitary matrix $\B{B_{t}}''$ such that $\Y_{t}' = \overline{\Y_{t}} \B{B_{t}}''$.  
Using this fact in combination with the unitary invariance of the Frobenius norm, one can now see that
$$\left\| \left( \B{Z_{t}}' \right)_r - \Y_{t}' \right\|_{\rm F} ~=~ \left\| \left( \B{Z_{t}}' \right)_r - \overline{\Y_{t}} \B{B_{t}}'' \right\|_{\rm F} ~=~\left\| \overline{\left( \B{Z_{t}}' \right)_r} - \overline{\Y_{t}}\B{B_{t}}' \right\|_{\rm F} = \left\| \overline{\left( \B{Z_{t}}' \right)_r} - \Y_{t}\B{B_{t}} \right\|_{\rm F}$$
for some (random) unitary matrices $\B{B_{t}}'$ and $\B{B_{t}}$.  
Hence, it suffices to bound the norm of $\left\| \left( \B{Z_{t}}' \right)_r - \Y_{t}' \right\|_{\rm F}$.  

Having said that, we can now do
\begin{align*}
\left\| \left( \B{Z_{t}}' \right)_r - \Y_{t}' \right\|_{\rm F} ~&\leq~ \left\| \left( \B{Z_{t}}' \right)_r - \B{Z_{t}}' \right\|_{\rm F} + \left\| \B{Z_{t}}'  - \B{Z_{t}} \right\|_{\rm F} + \left\| \B{Z_{t}} - \Y_{t}' \right\|_{\rm F}\\
~&=~ \sqrt{\sum^d_{j = r+1} \sigma^2_j(\B{Z_{t}} + \Psi_{t})} ~+~ \left\| \Psi_{t} \right\|_{\rm F} +  \left\| \B{Z_{t}} - \Y_{t}' \right\|_{\rm F}\\
~&=~ \sqrt{\sum^{\left\lceil \frac{d-r}{2} \right\rceil}_{j = 1} \sigma^2_{r+2j-1}(\B{Z_{t}} + \Psi_{t}) + \sigma^2_{r+2j}(\B{Z_{t}} + \Psi_{t})} ~+~ \left\| \Psi_{t} \right\|_{\rm F} +  \left\| \B{Z_{t}} - \Y_{t}' \right\|_{\rm F}\\
~&\leq~\sqrt{\sum^{\left\lceil \frac{d-r}{2} \right\rceil}_{j = 1} \left(\sigma_{r+j}(\B{Z_{t}}) + \sigma_{j}(\Psi_{t}) \right)^2 + \left( \sigma_{r+j}(\B{Z_{t}}) + \sigma_{j+1}(\Psi_{t})\right)^2} ~+~ \left\| \Psi_{t} \right\|_{\rm F} +  \left\| \B{Z_{t}} - \Y_{t}' \right\|_{\rm F}
\end{align*}

the result follows from applying Weyl's inequality in the first term~\cite{horn1994topics}.  

By the application of the triangle inequality on the first term we now have the following 
\begin{align*}
\left\| \left( \B{Z_{t}}' \right)_r - \Y_{t}' \right\|_{\rm F} ~&\leq~ \sqrt{\sum^d_{j = r+1} 2 \sigma^2_j(\B{Z_{t}}) } ~+~ \sqrt{\sum^d_{j = 1} 2 \sigma^2_j(\Psi_{t}) }  + \left\| \Psi_{t} \right\|_{\rm F} +  \left\| \B{Z_{t}} - \Y_{t}' \right\|_{\rm F}\\
&\leq~ \sqrt{2} \left( \| (\B{Z_{t}})_r - \B{Z_{t}} \|_{\rm F} +  \| \B{Z_{t}} - \Y_{t}' \|_{\rm F} \right)+ \left(1+ \sqrt{2} \right) \| \Psi_{t} \|_{\rm F}.
\end{align*}
Finally, \Cref{lemma:low_ranking_merge} for bounding the first two terms concludes the proof if we note that $\| (\Y_{t}')_r - \Y_{t}' \|_{\rm F} =  \| (\Y_{t})_r - \Y_{t} \|_{\rm F}$.
\end{proof}

Now, we introduce the final theorem which bounds the general error of $\FPCA{}$ with respect to the data matrix $\Y_{t}$ and up to multiplication by a unitary matrix.

\vspace{12pt}

\begin{thm}
  Let $\Y_{t} \in \R^{d\times tMb}$ and $q \geq 1$.  
  Then, $\FPCA{}$ is guaranteed to recover an
  $\Y_{t}^{q+1,1} \in \R^{d \times tMb}$ for any $t$ such that
  $\overline{\left(\Y_{t}^{q+1,1} \right)_r} = \Y_{t} \B{B_{t}} + \Psi_{t}$, where $\B{B_{t}}$ is a
  unitary matrix, and $\| \Psi_{t} \|_{\rm F} \leq \left( \left( 1 +
  \sqrt{2} \right)^{q+1} - 1 \right) \| (\Y_{t})_r - \Y_{t} \|_{\rm F}$.
  \label{thm:low_rank_recovery}
\end{thm}

\begin{proof}
  For the purposes of this proof we will refer to the approximate subspace result for $\Y_{t}^{p+1,i}$ from the merging chunks as $$\B{Z_{t}}^{p+1,i} := \left[
    \overline{\left( \B{Z_{t}}^{p,(i-1)tMb+1} \right)_r} ~\Big| \cdots \Big|~
    \overline{\left( \B{Z_{t}}^{p,itMb} \right)_r} \right],$$ for $p \in [q]$, and $i \in [M/(tMb)^p]$.  
  Which, as previously proved is equivalent to $\Y_{t}$, for any $t$ and up to a unitary transform.
  Moreover, $\Y_{t}$ will refer to the original - and, potentially full rank - matrix with block components defined as
  $\Y_{t}=\left[\Y_{t}^1|\Y_{t}^2|\cdots|\Y_{t}^M\right]$, where $M = (tMb)^q$.  
  Additionally, $\Y_{t}^{p,i}$ will refer to the respective uncorrupted block part of the original matrix $\Y_{t}$ whose values correspond to the ones of $\B{Z_{t}}^{p,i}$.
  \footnote{Meaning, $\B{Z_{t}}^{p,i}$ is used to estimate the approximate singular
    values and left singular vectors of $\Y_{t}^{p,i}$ for all $p \in[q+1]$, and $i \in [M/(tMb)^{p-1}]$} 
    
Hence, $\Y_{t} =
  \left[\Y_{t}^{p,1}|\Y_{t}^{p,2}|\cdots|\Y_{t}^{p,M/(tMb)^{(p-1)}}\right]$ holds for all
  $p \in [q+1]$, in which $$\Y_{t}^{p+1,i} := \left[ \Y_{t}^{p,(i-1)tMb+1}
    ~\Big| \cdots \Big|~ \Y_{t}^{p,itMb} \right]$$ for all $p \in [q]$,
  and $i \in [M/(tMb)^p]$.  For $p=1$ we have $\B{Z_{t}}^{1,i} = \Y_{t}^i =
  \Y_{t}^{1,i}$ for $i \in [M]$ by definition.
  Our target is to bound $\overline{\left(\B{Z_{t}}^{q+1,1} \right)_d}$ matrix with respect to the original matrix $\Y_{t}$, which can be done by induction on the level $p$.
  Concretely, we have to formally prove the following for all $p \in [q+1]$, and $i \in [M/(tMb)^{(p-1)}]$
\begin{enumerate}
\item $\overline{\left( \B{Z_{t}}^{p,i} \right)_r} = \Y_{t}^{p,i} W^{p,i} + \Psi_{t}^{p,i}$, where 
\item $\B{B_{t}}^{p,i}$ is always a unitary matrix, and 
\item $\| \Psi_{t}^{p,i} \|_{\rm F} \leq \left( \left( 1 + \sqrt{2} \right)^{p} - 1 \right) \left \| (\Y_{t}^{p,i})_d - \Y_{t}^{p,i} \right \|_{\rm F}$.
\end{enumerate}

Notably, requirements $1-3$ are always satisfied when $p=1$ since $\B{Z_{t}}^{1,i} = \Y_{t}^i = \Y_{t}^{1,i}$ for all $i \in [M]$ by definition.  
Hence, we can claim that a unitary matrix $\B{B_{t}}^{1,i}$ for all $i \in [M]$ satisfying
$$\overline{ \left( \B{Z_{t}}^{1,i} \right)_d} ~=~ \overline{\left(  \Y_{t}^{1,i} \right)_r} ~=~ \left(  \Y_{t}^{1,i} \right)_r \B{Z_{t}}^{1,i} ~=~ \Y_{t}^{1,i} \B{B_{t}}^{1,i} + \left(  \left(  \Y_{t}^{1,i} \right)_r - \Y_{t}^{1,i} \right) \B{B_{t}}^{1,i},$$
where $\Psi^{1,i} :=  \left(  \left(  \Y_{t}^{1,i} \right)_r - \Y_{t}^{1,i} \right) W^{1,i}$ has 

\begin{equation}
\| \Psi_{t}^{1,i} \|_{\rm F} = \left \| \left(  \Y_{t}^{1,i} \right)_r - \Y_{t}^{1,i} \right \|_{\rm F} \leq \sqrt{2} \left \| \left(  \Y_{t}^{1,i} \right)_r - \Y_{t}^{1,i} \right \|_{\rm F}.
\end{equation}

Moreover, let's assume that conditions $1-3$ hold for some $p \in [q]$.  
In which case, we can see see from condition 1 that

\begin{align*}
\B{Z_{t}}^{p+1,i} ~&:=~ \left[ \overline{\left( \B{Z_{t}}^{p,(i-1)tMb+1} \right)_r} ~\Big| \cdots \Big|~ \overline{\left( \B{Z_{t}}^{p,itMb} \right)_r} \right]\\
&=~\left[ \Y_{t}^{p,(i-1)tMb+1} \B{B_{t}}^{p,(i-1)tMb+1} + \Psi_{t}^{p,(i-1)tMb+1}  ~\Big| \cdots \Big|~ \Y_{t}^{p,itMb} \B{B_{t}}^{p,itMb} + \Psi_{t}^{p,itMb}  \right]\\
&=~\left[ \Y_{t}^{p,(i-1)tMb+1} \B{B_{t}}^{p,(i-1)tMb+1} ~\Big| \cdots \Big|~ \Y_{t}^{p,itMb} \B{B_{t}}^{p,itMb} \right] + \left[  \Psi_{t}^{p,(i-1)tMb+1}  ~\Big| \cdots \Big|~ \Psi_{t}^{p,itMb}  \right] \\
&=~\left[ \Y_{t}^{p,(i-1)tMb+1} ~\Big| \cdots \Big|~ \Y_{t}^{p,itMb}  \right]  \tilde{\B{B_{t}}} + \tilde{\Psi_{t}},
\end{align*}

where $\tilde{\Psi_{t}} := \left[  \Psi_{t}^{p,(i-1)tMb+1}  ~\Big| \cdots \Big|~ \Psi_{t}^{p,itMb)}  \right] $, and 
\[ 
  \arraycolsep=1.4pt\def\arraystretch{0.3}
  \! \! \tilde{\B{B_{t}}} := \! \! \! = \! \! \!
    \left(  \! \! 
    \begin{array}{c;{2pt/2pt}c;{2pt/2pt}c;{2pt/2pt}c}
        & & & \\
        & & & \\
        ~\B{B_{t}}^{p,(i-1)tMb+1} & {\bf 0} & {\bf 0} &  {\bf 0} \\ 
        \vspace{0.05in} \\ \hdashline[2pt/2pt] %
        {\bf 0} & \B{B_{t}}^{p,(i-1)tMb+2}  & {\bf 0} &  {\bf 0} \\ 
        \vspace{0.05in} \\ \hdashline[2pt/2pt]
        {\bf 0} & {\bf 0} & ~\ddots~ &  {\bf 0} \\ 
        \vspace{0.05in} \\ \hdashline[2pt/2pt]
        {\bf 0} & {\bf 0} & {\bf 0} &  \B{B_{t}}^{p,i(tMb)}~
            \end{array}  \! \!   \right) \! \!.    \] 
Of note is that $\tilde{\B{B_{t}}}$ is always unitary due to its diagonal blocks all being unitary by condition 2 (and hence, by construction).  
Hence, we can claim that $\B{Z_{t}}^{p+1,i} = \Y_{t}^{p+1,i} \tilde{\B{B_{t}}} + \tilde{\Psi_{t}}.$  

Following this, we can now bound $\left \| \left( \B{Z_{t}}^{p+1,i} \right)_r - \Y_{t}^{p+1,i} \tilde{\B{B_{t}}}\right \|_{\rm F}$ by the use of similar argument to that we employed during the the proof of Theorem~\ref{thm:low_rank_merge}.

\begin{align}
\left \| \left( \B{Z_{t}}^{p+1,i} \right)_r - \Y_{t}^{p+1,i} \tilde{\B{B_{t}}}\right \|_{\rm F} ~&\leq~ \left\| \left( \B{Z_{t}}^{p+1,i} \right)_r - \B{Z_{t}}^{p+1,i} \right\|_{\rm F} + \left\| \B{Z_{t}}^{p+1,i}  - \Y_{t}^{p+1,i} \tilde{\B{B_{t}}} \right\|_{\rm F} \nonumber\\
~&=~ \sqrt{\sum^d_{j = r+1} \sigma^2_j \left(\Y_{t}^{p+1,i} \tilde{\B{B_{t}}} + \tilde{\Psi_{t}} \right)} ~+~ \| \tilde{\Psi_{t}} \|_{\rm F} \nonumber \\
~&\leq~ \sqrt{\sum^d_{j = r+1} 2 \sigma^2_j \left(\Y_{t}^{p+1,i} \tilde{\B{B_{t}}} \right) } ~+~ \sqrt{\sum^d_{j = 1} 2 \sigma^2_j( \tilde{\Psi_{t}}) }  + \| \tilde{\Psi_{t}} \|_{\rm F} \nonumber \\
~&=~ \sqrt{2} \left \| \Y_{t}^{p+1,i} - \left( \Y_{t}^{p+1,i} \right)_r \right \|_{\rm F} + \left(1+ \sqrt{2} \right) \| \tilde{\Psi_{t}} \|_{\rm F} \label{eq:thm_ref_first_bit}.
\end{align}

Appealing to condition 3 in order to bound  $\| \tilde{\Psi_{t}} \|_{\rm F}$ we obtain

\begin{align*}
\| \tilde{\Psi_{t}} \|^2_{\rm F} ~=~ \sum^{tMb}_{j=1} \| \Psi_{t}^{p,(i-1)tMb+j} \|^2_{\rm F} &\leq \left( \left( 1 + \sqrt{2} \right)^{p} - 1 \right)^2 \sum^{tMb}_{j=1} \left \| (\Y_{t}^{p,(i-1)tMb+j})_r - \Y_{t}^{p,(i-1)tMb+j} \right \|^2_{\rm F}\\
&\leq \left( \left( 1 + \sqrt{2} \right)^{p} - 1 \right)^2 \sum^{tMb}_{j=1} \left \| (\Y_{t}^{p+1,i})^j_d - \Y_{t}^{p,(i-1)n+j} \right \|^2_{\rm F},
\end{align*}
where $(\Y_{t}^{p+1,i})^j_r$ denotes the block of $( \Y_{t}^{p+1,i} )_d$ corresponding to $\Y_{t}^{p,(i-1)n+j}$ for $j \in [tMb]$.  
Hence,

\begin{align}
\| \tilde{\Psi_{t}} \|^2_{\rm F} &\leq \left( \left( 1 + \sqrt{2} \right)^{p} - 1 \right)^2 \sum^{tMb}_{j=1} \left \| (\Y_{t}^{p+1,i})^j_d - \Y_{t}^{p,(i-1)tMb+j} \right \|^2_{\rm F} \nonumber \\ 
&= \left( \left( 1 + \sqrt{2} \right)^{p} - 1 \right)^2 \left \| (\Y_{t}^{p+1,i})_r - \Y_{t}^{p+1,i} \right \|^2_{\rm F} \label{eq:thm_ref_last_bit}.
\end{align}

By using both \eqref{eq:thm_ref_first_bit} and \eqref{eq:thm_ref_last_bit} we can claim that

\begin{align}
\left \| \left( \B{Z_{t}}^{p+1,i} \right)_r - \Y_{t}^{p+1,i} \tilde{\B{B_{t}}}\right \|_{\rm F} &\leq \left[ \sqrt{2} + (1+ \sqrt{2} ) \left( \left( 1 + \sqrt{2} \right)^{p} - 1 \right) \right] \left \| \left( \Y_{t}^{p+1,i} \right)_r - \Y_{t}^{p+1,i}  \right \|_{\rm F} \nonumber\\
&= \left( \left( 1 + \sqrt{2} \right)^{p+1} - 1 \right) \left \| \left( \Y_{t}^{p+1,i} \right)_r - \Y_{t}^{p+1,i} \right \|_{\rm F} \label{equ:LastthmRef3}.
\end{align}

In the above, of note is that $\left \| \left( \B{Z_{t}}^{p+1,i} \right)_r - \Y_{t}^{p+1,i} \tilde{\B{B_{t}}}\right \|_{\rm F} = \left \| \overline{ \left( \B{Z_{t}}^{p+1,i} \right)_r } - \Y_{t}^{p+1,i} \B{B_{t}}^{p+1,i}  \right \|_{\rm F}$ where $\B{B_{t}}^{p+1,i}$ is always unitary.  
Hence, we can see that conditions 1 - 3 hold at any $t$ and any $p+1$ with $\Psi_{t}^{p+1,i} := \overline{ \left( \B{Z_{t}}^{p+1,i} \right)_r } - \Y_{t}^{p+1,i} \B{B_{t}}^{p+1,i}$.
\end{proof}

\Cref{thm:low_rank_recovery} proves that at any given time $t$, $\FPCA{}$ will accurately compute low rank approximations $\overline{\Y_{t}}$ of the data seen so up to time $t$ so long as the depth of the tree is relatively small. 
This is a valid assumption in our setting since we expect federated deployments to be shallow and have a large fanout. That is, we expect that the depth of the tree will be low and that many nodes will be using the same aggregator for their merging procedures. 
It is also worth mentioning that the proof of Theorem~\ref{thm:low_rank_recovery} can tolerate small additive noise (e.g. round-off and approximation errors) in the input matrix $\Y_{t}$ at time $t$.
Finally, we fully expect that, at any $t$, the resulting error will be no higher than $\min \rank(\Y_t^i)$ $\forall i \in [M]$ and no lower than $ \max \rank(\Y_t^i)$ $\forall i \in [M]$

\section{Further Evaluation Details}
\label{apx:further_dataset_eval}

In addition to the traditional MNIST results presented in the main paper, we further evaluate $\FPCAC$ against other competing methods which show that it performs favourably both in terms of accuracy and time when using synthetic and real datasets.

\subsection{Synthetic Datasets}
\label{apx:synthetic_dataset_eval}

For the tests on synthetic datasets, the vectors $\{\mathbf{y}_t\}_{t=1}^\tau$ are drawn independently from a zero-mean Gaussian distribution with the covariance matrix $\B{\Xi}=\B{S}\B{\Lambda}\B{S}^{T}$, where $\B{S}\in \mathcal{O}(d)$ is a generic basis obtained by orthogonalising a standard random Gaussian matrix. The entries of the diagonal matrix $\B{\Lambda}\in\mathbb{R}^{d \times d}$ (the eigenvalues of the covariance matrix $\B{\Xi}$) are selected according to the power law, namely,  $\lambda_i =  i^{-\alpha}$, for a positive $\alpha$.
To be more succinct, wherever possible we employ {MATLAB}'s notation for specifying the value ranges in this section. 

To assess the performance of $\FPCA$, we let $\Y_t=[\y_1,\cdots,\y_t]\in \mathbb{R}^{d \times t}$ be the data received by time $t$ and  
$\widehat{\Y}^{\FPCAC}_{t,r}$ be the output of $\FPCAC$ at time $t$.
\footnote{Recall, since \textit{block}-based algorithms like $\FPCA$, 
do not update their estimate after receiving feature vector but per each 
block for convenience in with respect to the evaluation against other 
algorithms (which might have different block sizes or singular updates), 
we properly {\em interpolate} their outputs over time.} 
Then, the error incurred by $\FPCAC$ is 
\begin{equation}
\frac{1}{t}\|\Y_t - \widehat{\Y}^{\FPCAC}_{t,r}\|_F^2, 
\end{equation} 
Recall, that the above error is always larger than the residual of $\Y_t$, namely,
\begin{equation}
\|\Y_t - \widehat{\Y}^{\FPCAC}_{t,r}\|_F^2 \ge \|\Y_t - {\Y}_{t,r}\|_F^2 = \rho_r^2(\Y_t).
\end{equation}
In the expression above, ${\Y}_{t,r}=\SVD_r(\Y_t)$ is a rank-$r$ truncated $\SVD$ of $\Y_t$ and $\rho_r^2(\Y_t)$ is the corresponding residual. 

Additionally, we compare $\FPCA$ against {GROUSE}~\citep{balzano2013grouse}, FD~\citep{desai2016improved}, PM~\citep{mitliagkas2013memory} and a version of PAST~\citep{papadimitriou2005streaming,yang1995projection}. 
Interestingly and contrary to $\FPCAC$, the aforementioned algorithms are \textit{only} able to estimate the principal components of the data and \textit{not} their projected data on-the-fly. 
Although, it has to noted that in this setup we are only interested in the resulting subspace $\U$ along with its singular values $\Sigma$ but is worth mentioning that the projected data, if desired, can be kept as well. 
More specifically, let $\widehat{\SU}_{t,r}^g \in \text{G}(d,r)$ be the span of the output of GROUSE, with the outputs of the other algorithms defined similarly. 
Then, these algorithms incur errors
\begin{equation*}
\frac{1}{t} \| \Y_t - \P_{\widehat{\SU}_{t,r}^{v}} \Y_t \|_F^2, \; v\in{g,f,p,\FPCAC},
\end{equation*}
where we have used the notation $\P_{\mathcal{A}}\in \mathbb{R}^{d\times d}$ to denote the orthogonal projection onto the subspace $\mathcal{A}$. 
Even though robust FD \citep{luo2017robust} improves over FD in the quality of matrix sketching, since the subspaces produced by FD and robust FD coincide, there is no need here for computing a separate error for robust FD.

Throughout our synthetic dataset experiments we have used an ambient dimension $d=400$, and for each $a \in (0.001, 0.1, 0.5, 1, 2, 3)$ generated $N=4000$ feature vectors in $\R^{d}$ using the method above. 
This results in a set of with four datasets of size $\R^{d \times N}$. 
Furthermore, in our experiments we used a block size of $b=50$ for $\FPCAC$, while for PM we chose $b=d$.
FD \& GROUSE perform singular updates and do not 
need a block-size value.
Additionally, the step size for GROUSE was set to $2$ and the total sketch size for FD was set $2r$. 
In all cases, unless otherwise noted 
in the respective graphs the starting rank for all methods in the synthetic dataset experiments was set to $r=10$.

We evaluated our algorithm using the aforementioned error metrics on a set of datasets generated as described above. 
The results for the different $a$ 
values are shown 
in~\autoref{fig:more_experiments_path}, which shows $\FPCAC$ can achieve an error that is significantly smaller than SP while maintaining a small number of principal components throughout the evolution of the algorithms in the absence
of a forgetting factor $\lambda$. When a forgetting
factor is used, as is shown 
in~\ref{fig:more_experiments_normal} then the 
performance of the two methods is similar.
This figure was produced on pathological datasets generated with an adversarial spectrum.
It can be seen that in SPIRIT the need for PC's increases dramatically for no apparent reason, whereas $\FPCA$ behaves favourably. 

\begin{figure*}[htb!]
    \centering
    \begin{subfigure}{.48\textwidth}
        \centering
        \includegraphics
        [scale=.28]
        {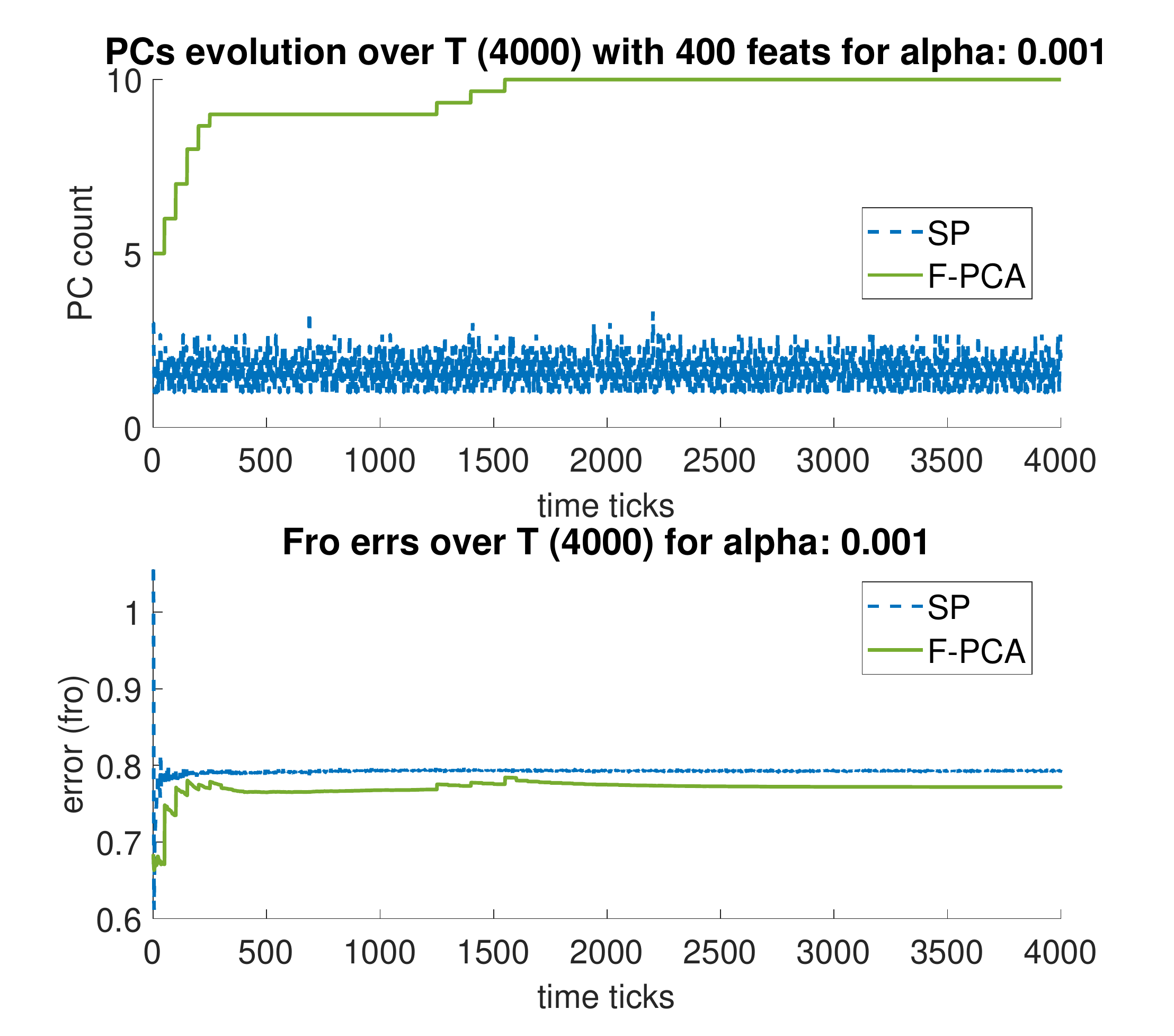}
        \caption{$\alpha=0.001$.}
        \label{fig:normal_fro_ext_errors_feat_400_T_4000_alpha_0_001}
    \end{subfigure}
    \begin{subfigure}{.48\textwidth}
        \centering
        \includegraphics
        [scale=.28]
        {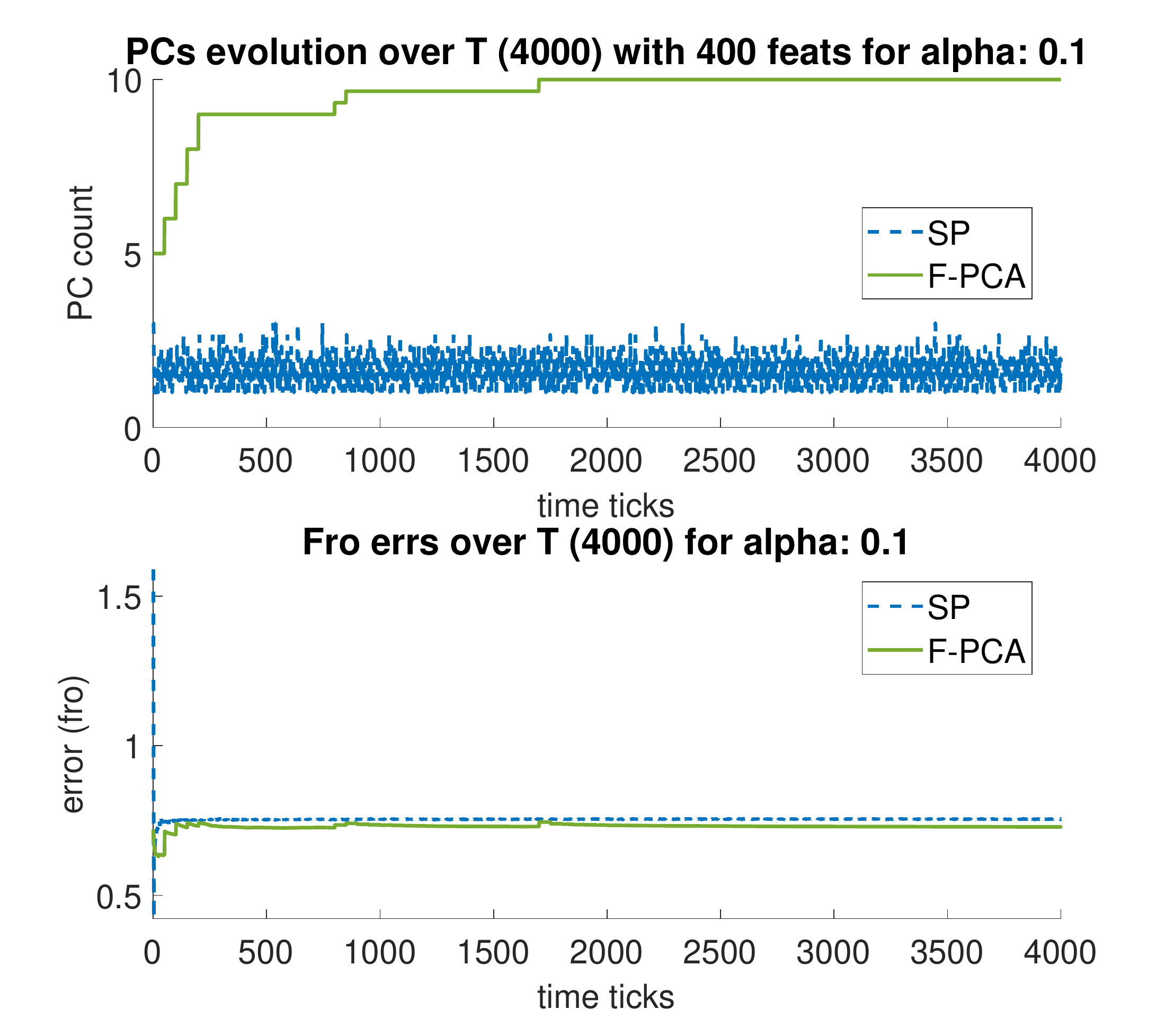}
        \caption{$\alpha=0.1$.}
        \label{fig:normal_fro_ext_errors_feat_400_T_4000_alpha_0_01}
    \end{subfigure}
    \begin{subfigure}{.48\textwidth}
        \centering
        \includegraphics
        [scale=.28]
        {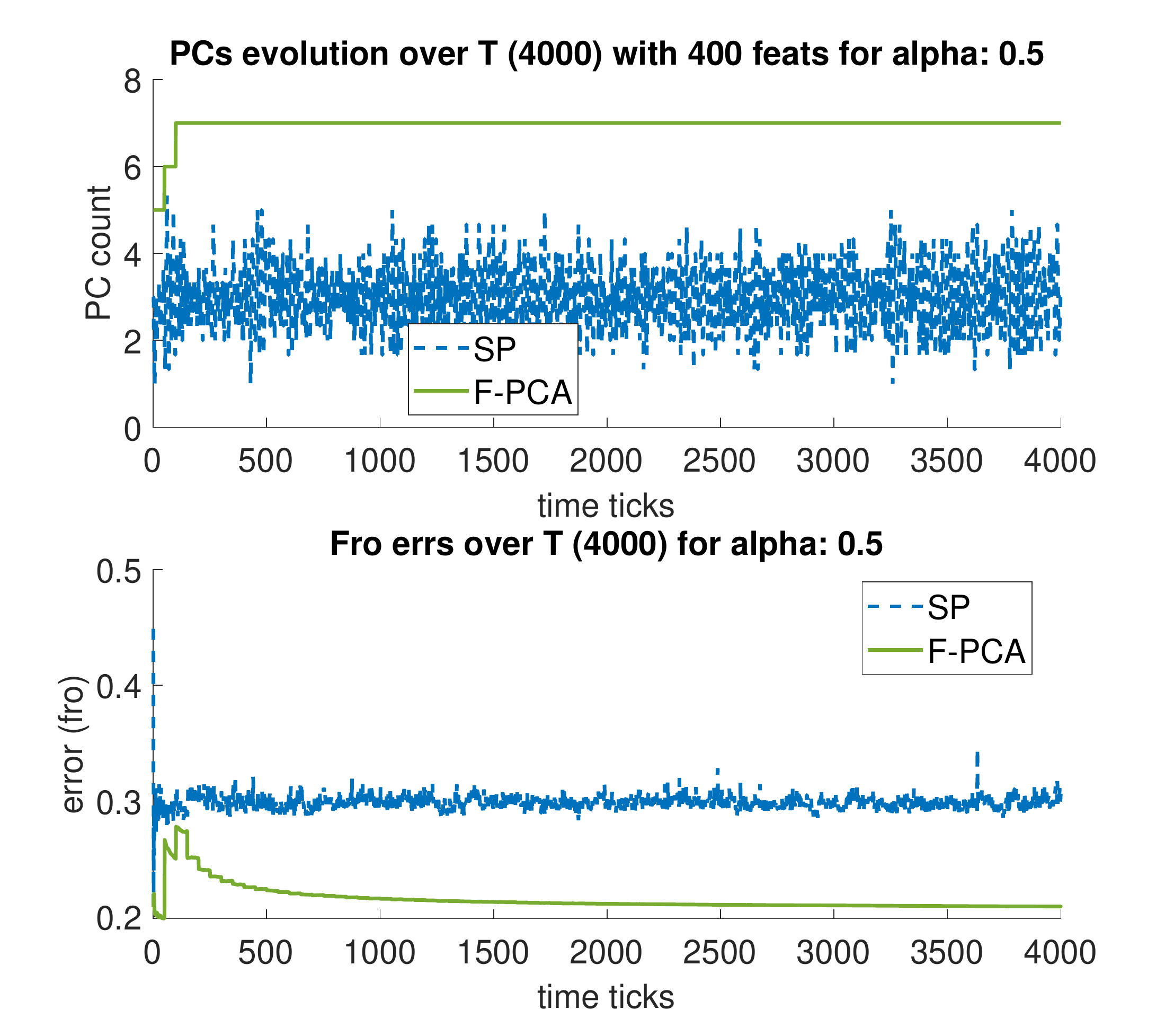}
        \caption{$\alpha=0.5$.}
        \label{fig:normal_fro_ext_errors_feat_400_T_4000_alpha_0_5}
    \end{subfigure}
    \begin{subfigure}{.48\textwidth}
        \centering
        \includegraphics
        [scale=.28]
        {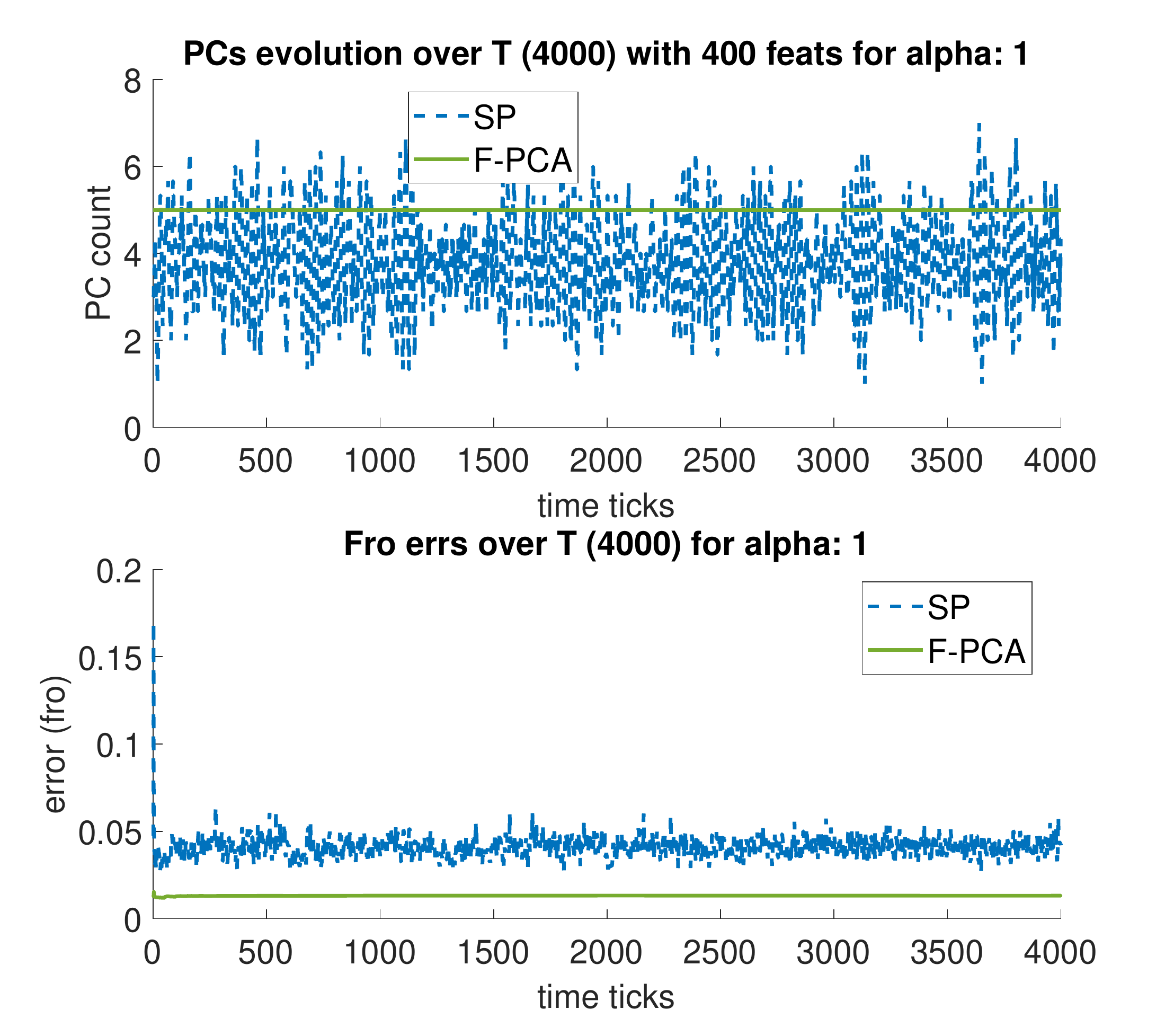}
        \caption{$\alpha=1$.}
        \label{fig:normal_fro_ext_errors_feat_400_T_4000_alpha_1}
    \end{subfigure}
    \begin{subfigure}{.48\textwidth}
        \centering
        \includegraphics
        [scale=.28]
        {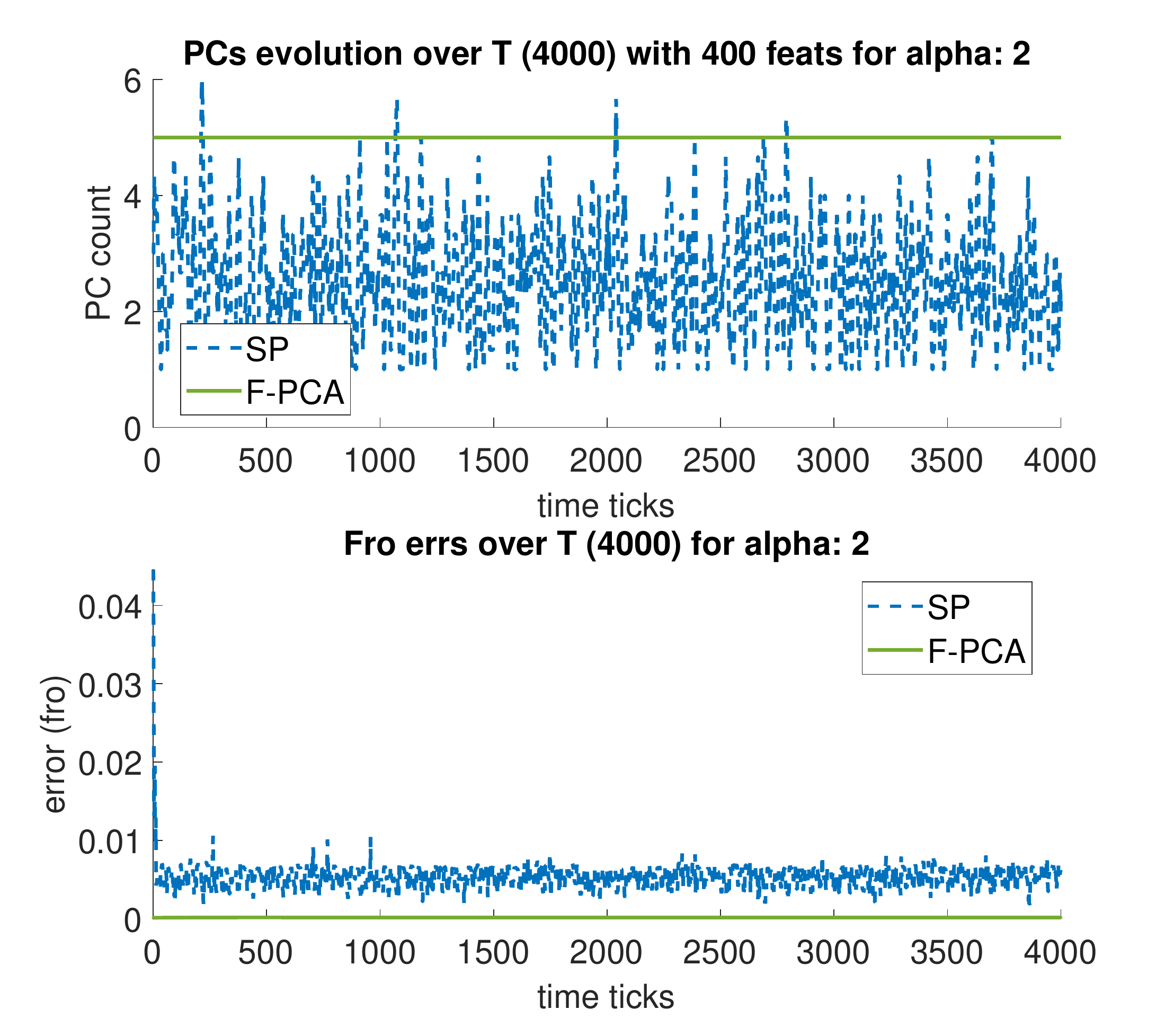}
        \caption{$\alpha=2$.}
        \label{fig:normal_fro_ext_errors_feat_400_T_4000_alpha_2}
    \end{subfigure}
    \begin{subfigure}{.48\textwidth}
        \centering
        \includegraphics
        [scale=.28]
        {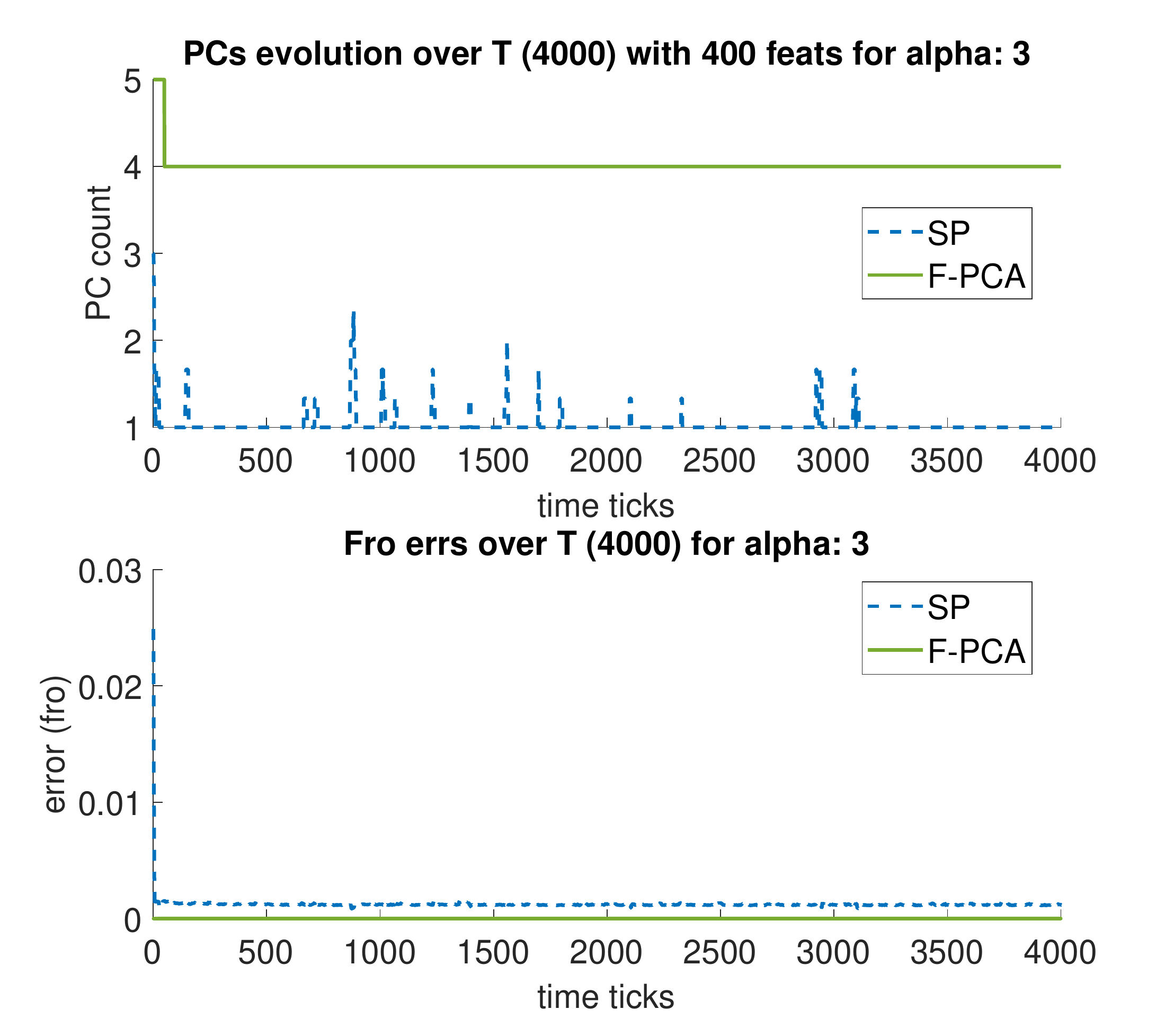}
        \caption{$\alpha=3$.}
        \label{fig:normal_fro_ext_errors_feat_400_T_4000_alpha_3}
    \end{subfigure}
    \caption{Performance measurements across the spectrum (when using forgetting factor $\lambda=0.9$).}
    \label{fig:more_experiments_normal}
\end{figure*}

\vfill
\pagebreak

Additionally, in order to bound our algorithm in terms of the expected error, we used a \textit{fixed rank} version with a low and high bound which fixed its rank value $r$ to the lowest and highest estimated $r$-rank during its normal execution.
We fully expect the incurred error of our adaptive scheme to fall within these bounds.
On the other hand, ~\autoref{fig:more_experiments_normal} shows that a drastic performance improvement occurs when using an exponential forgetting factor for SPIRIT with value $\lambda=0.9$, but the generated subspace is of inferior quality when compared to the one produced by $\FPCAC$.

\begin{figure*}[htb!]
    \centering
    \begin{subfigure}{.48\textwidth}
        \centering
        \includegraphics
        [scale=0.28]
        {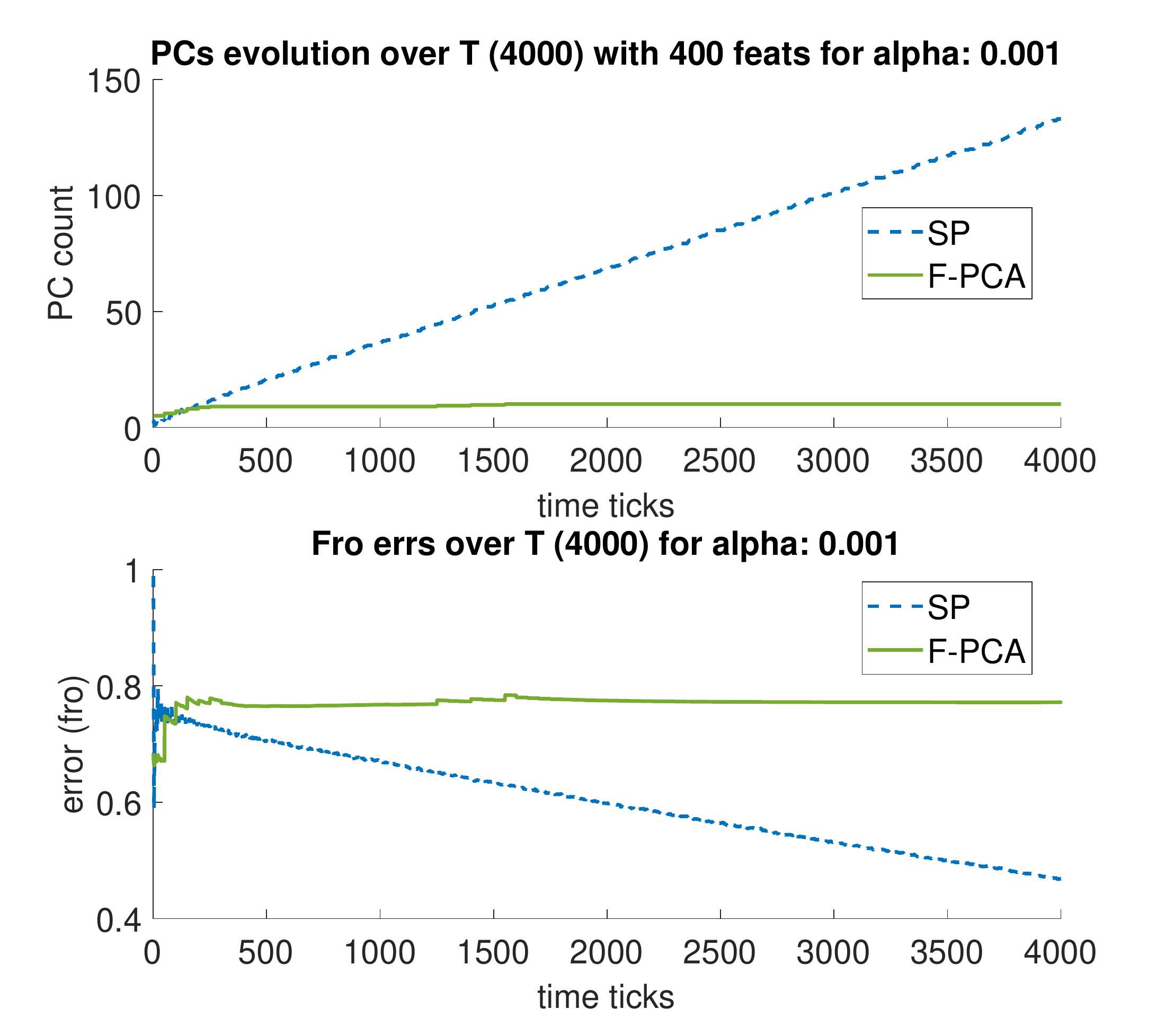}
        \caption{$\alpha=0.001$.}
        \label{fig:fro_ext_err_alpha_0_001}
    \end{subfigure}
    \begin{subfigure}{.48\textwidth}
        \centering
        \includegraphics
        [scale=0.28]
        {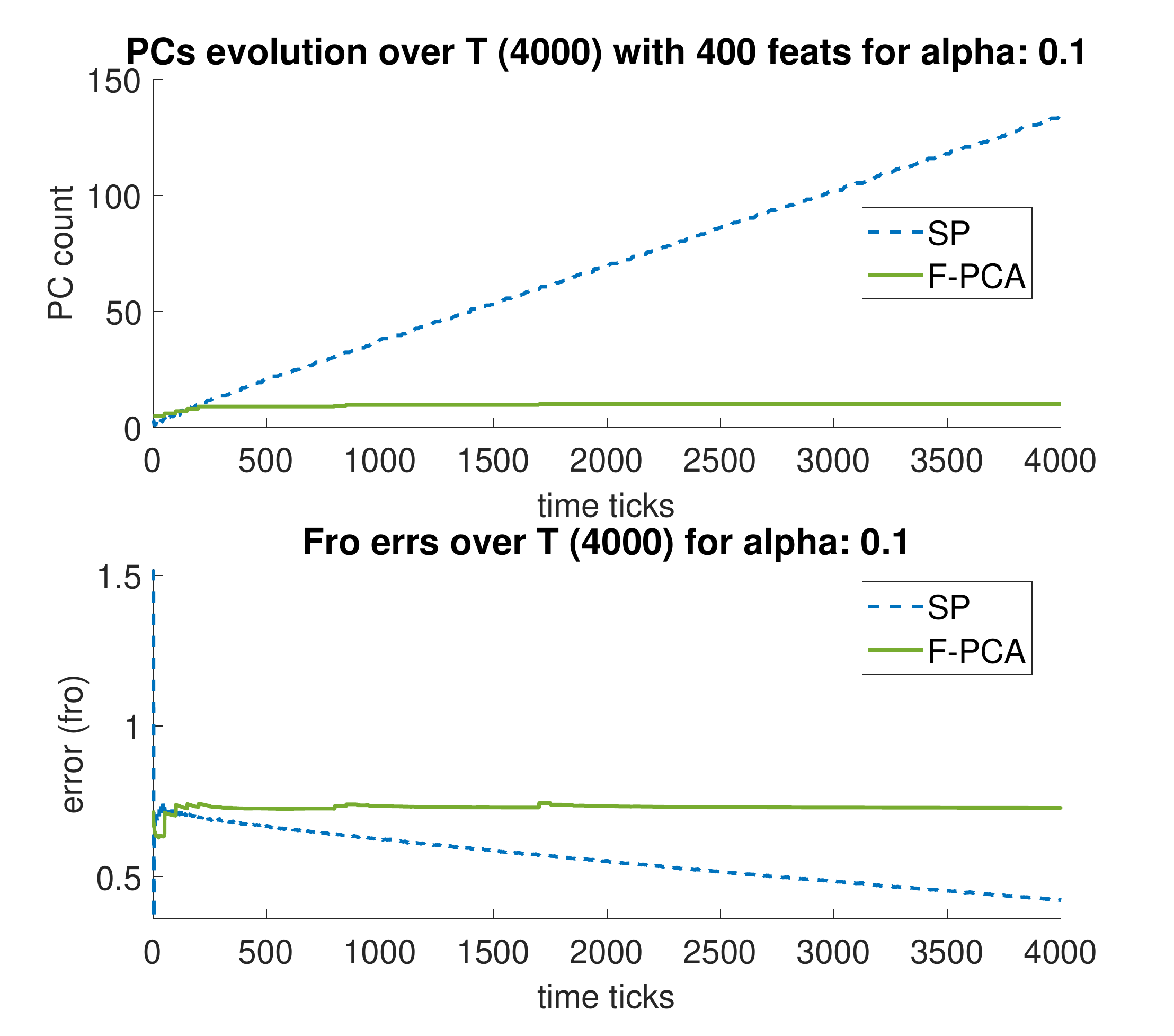}
        \caption{$\alpha=0.1$.}
        \label{fig:fro_ext_err_alpha_0_1}
    \end{subfigure}
    \begin{subfigure}{.48\textwidth}
        \centering
        \includegraphics
        [scale=0.28]
        {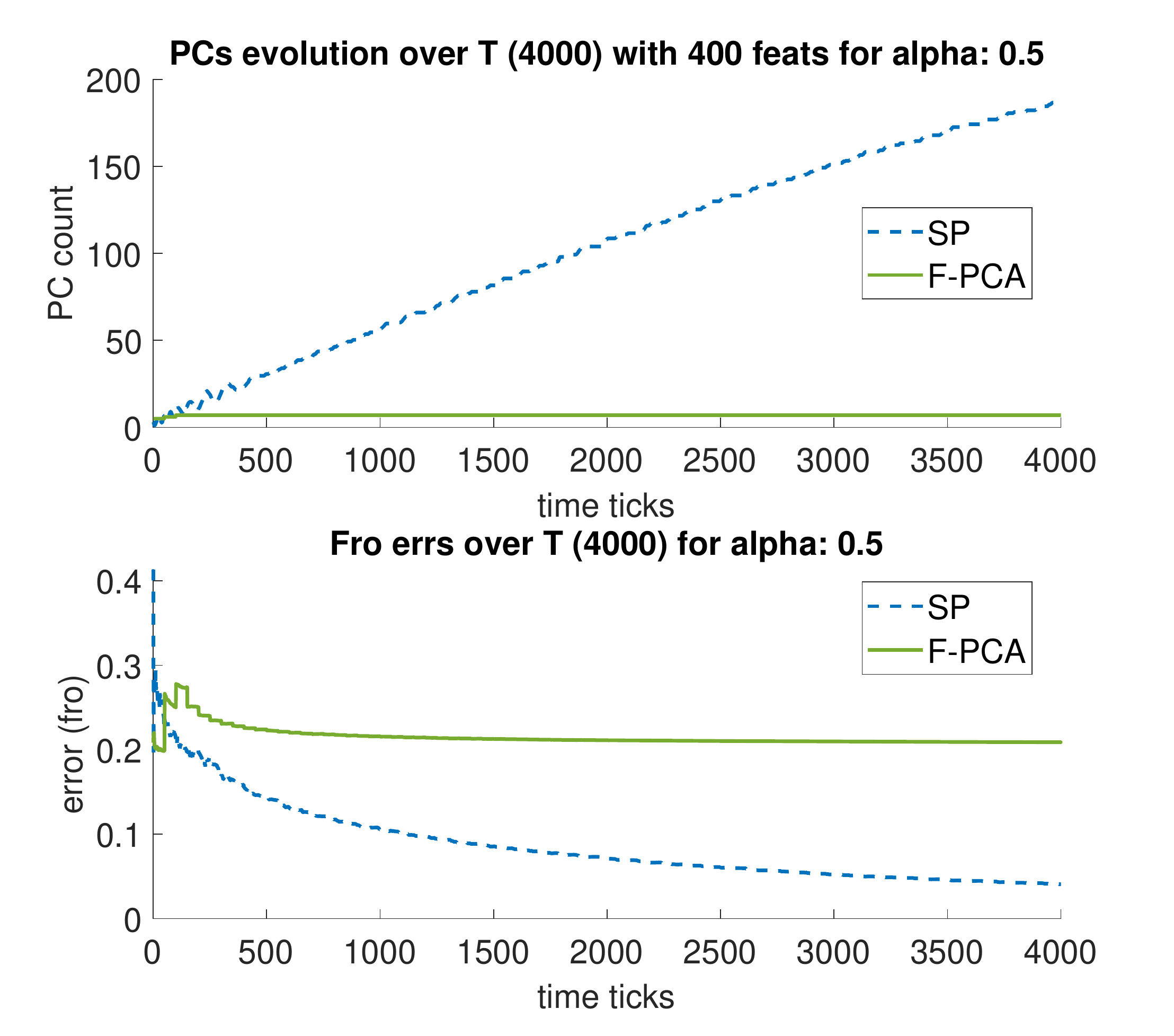}
        \caption{$\alpha=0.5$.}
        \label{fig:path_bottom_alpha_0_5}
    \end{subfigure}
    \begin{subfigure}{.48\textwidth}
        \centering
        \includegraphics
        [scale=0.28]
        {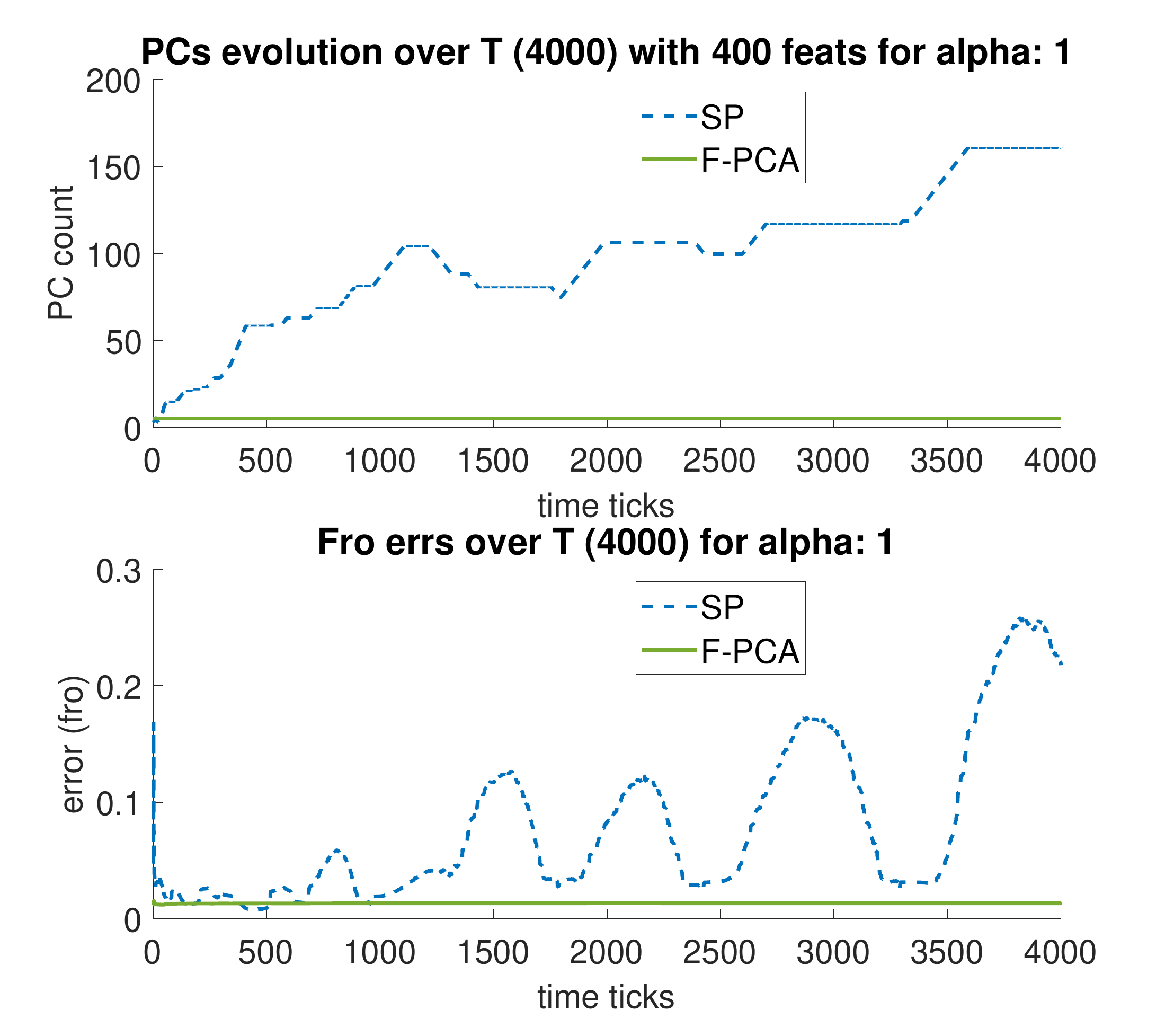}
        \caption{$\alpha=1$.}
        \label{fig:path_bottom_alpha_1}
    \end{subfigure}
    \begin{subfigure}{.48\textwidth}
        \centering
        \includegraphics
        [scale=.28]
        {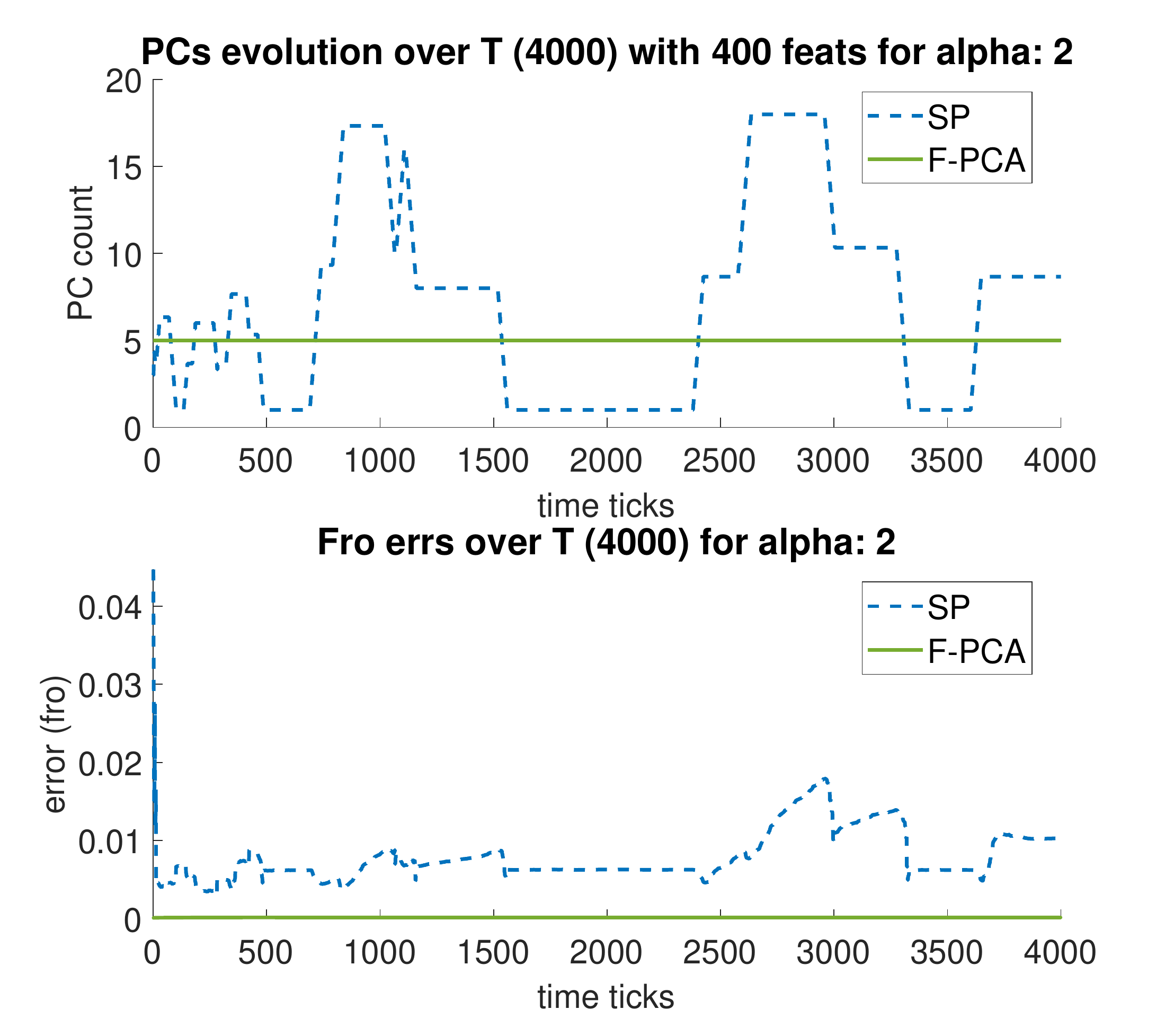}
        \caption{$\alpha=2$.}
            \label{fig:path_bottom_alpha_2}
    \end{subfigure}
    \begin{subfigure}{.48\textwidth}
        \centering
        \includegraphics
        [scale=0.28]
        {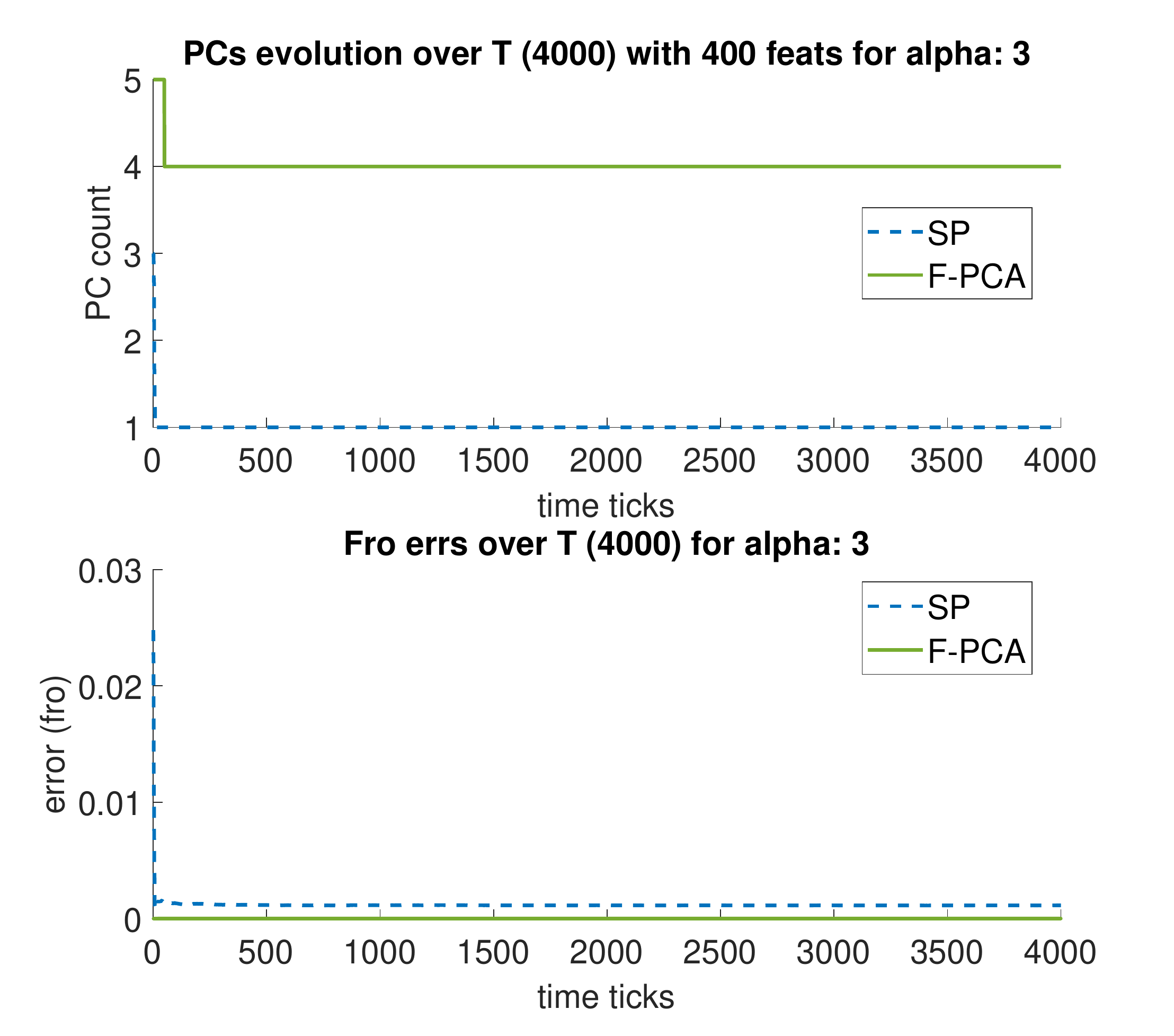}
        \caption{$\alpha=3$.}
        \label{fig:path_bottom_alpha_3}
    \end{subfigure}
    \caption{Pathological examples for adversarial Spectrums.}
    \label{fig:more_experiments_path}
\end{figure*}

\vfill
\pagebreak

Figures~\ref{fig:mse_subspace_normal} and \ref{fig:mse_subspace_pathological} show the results of our experiments on synthetic data $\text{Synth}(\alpha)^{d \times n} \subset \R^{d \times n}$ with $(d,n) =(400, 4000)$ generated as described above.
In the experiments, we let $\lambda$ be the forgetting factor of SP.
Figure~\ref{fig:more_experiments_normal} compares $\FPCAC$ with SP when $(\alpha, \lambda)=(1, 0.9)$ and Figure \ref{fig:more_experiments_path} when $(\alpha, \lambda) = (2, 1)$.
While $\FPCA$ exhibits relative stability in both cases with respect to the incurred $||\cdot||_{F}$ error, $SP$ exhibits a monotonic increase in the number of principal components estimated, in most cases, when $\lambda=1$. 
This behaviour is replicated in Figures~\ref{fig:mse_subspace_normal} and \ref{fig:mse_subspace_pathological} where \texttt{RMSE} subspace error is computed across the evaluated methods; thus, we can see while SP has better performance when $\lambda=1$ the number of principal components kept in most cases is unusually high.

\begin{figure}[htb!]
    \centering
    \begin{subfigure}{.48\linewidth}
    \centering
    \includegraphics[scale=0.4]{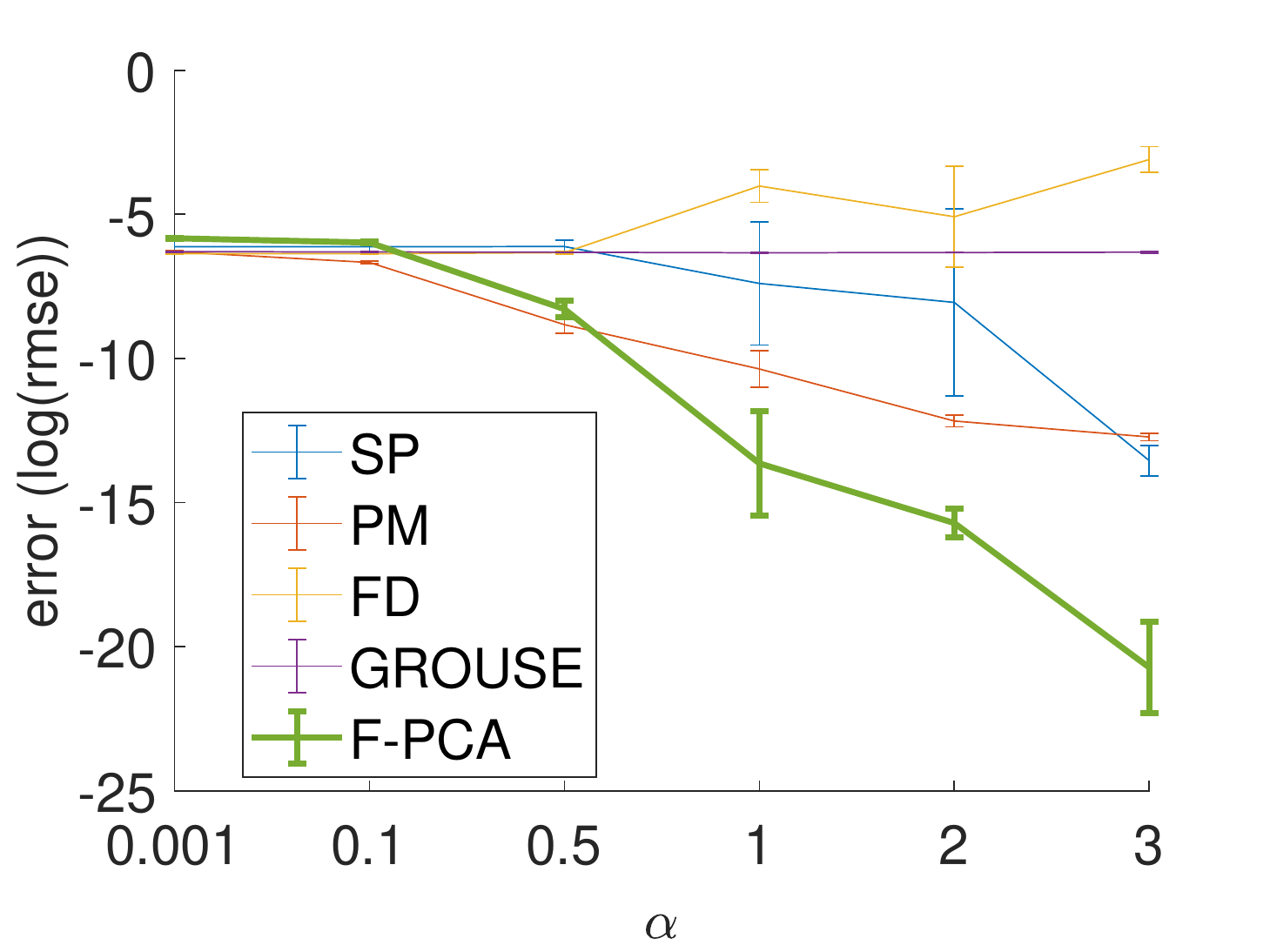}
    \caption{$\lambda=0.9$}
    
    \label{fig:mse_subspace_normal}
    \end{subfigure}
    \begin{subfigure}{.48\linewidth}
    \centering
    \includegraphics
    [scale=0.4]
    {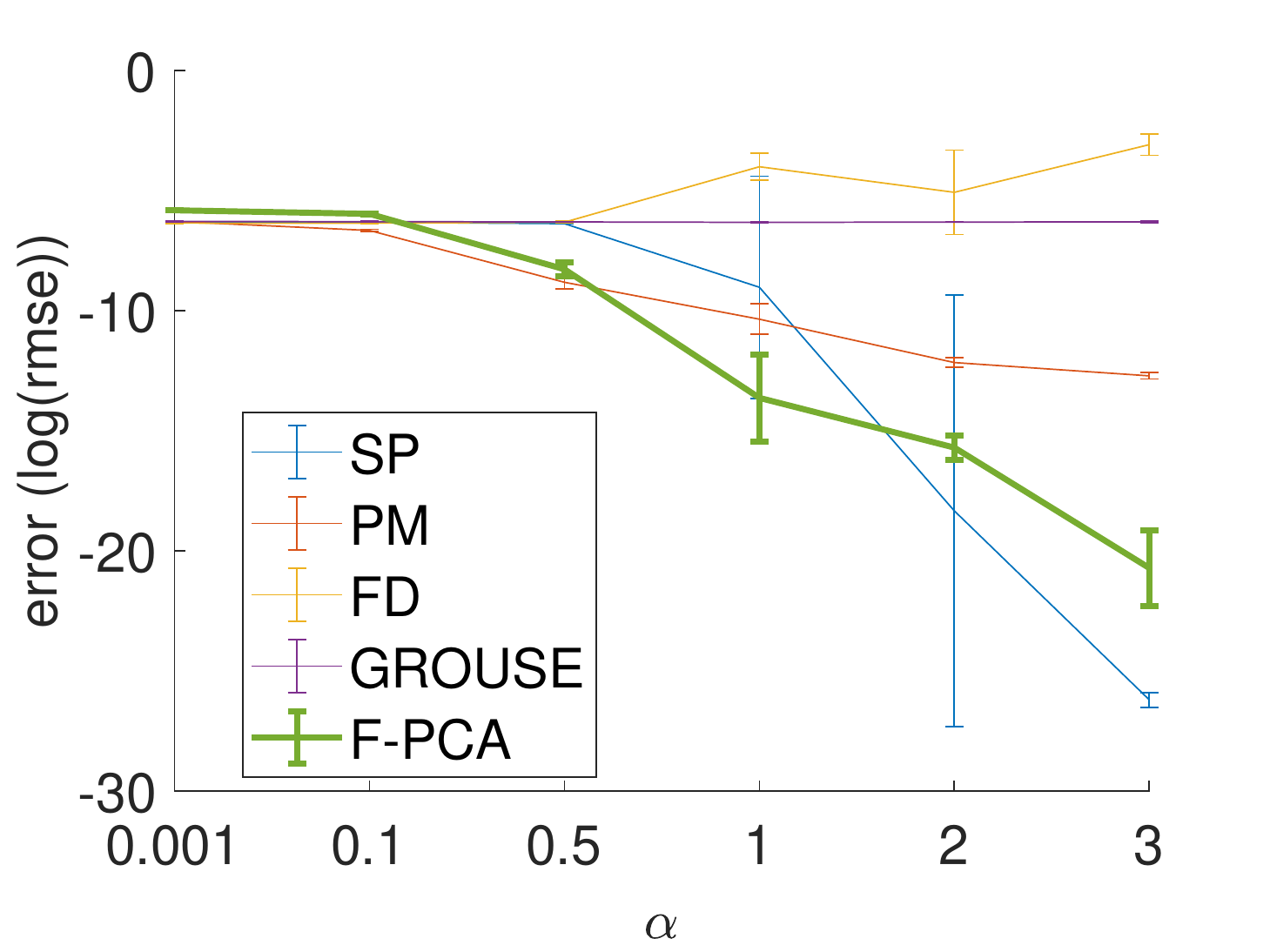}
    \caption{$\lambda=1$}
    \label{fig:mse_subspace_pathological}
    \end{subfigure}
    \caption{Resulting subspace $\mathbf{U}$ comparison across different spectrums generated using different $\alpha$ values.}
    \label{fig:mse_subspace}
\end{figure}

\vfill
\pagebreak

\subsection{Real Datasets}
\label{apx:real_dataset_eval}

To further evaluate our method against real datasets we also report in addition to the final subspace errors the Frobenious norm errors over time for all datasets and methods we used in the main paper. 
Namely, we used one that contains {\em light}, {\em volt}, and {\em temperature} readings gathered over a significant period of time, each of which exhibiting different noteworthy characteristics\footnote{Source of data: \url{https://www.cs.cmu.edu/afs/cs/project/spirit-1/www/data/Motes.zip}}. 
These datasets are used in addition to the MNIST and Wine quality datasets discussed in the main paper.
As with the synthetic datasets, across all real dataset experiments we used an ambient dimension $d$ and $N$ equal to the dimensions of each dataset. 
For the configuration parameters we elected to use a block size of $b=50$ for $\FPCAC$ and $b=d$ for PM.
The step size for GROUSE was again set to $2$ and the total sketch size for FD equal to $2r$. 
Additionally, we used the same bounding technique as with the synthetic datasets to bound the error of $\FPCAC$ using a fixed $r$ with lowest and highest estimation of the $r$-rank and note that we fully expect $\FPCAC$ to fall again within these bounds.
Note, that most reported errors are logarithmic; this was done in order for better readability and to be able to fit in the same plot most methods - of course, this is also reflected on the $y$-axis label as well.
We elected to do this as a number of methods, had errors orders of magnitude higher which posed a challenge when trying to plot them in the same figure.

\subsubsection{Motes datasets}

In this we elaborate on the findings with respect to the Motes dataset; below we present each of the measurements included along with discussion on the findings.

\paragraph{Humidity readings sensor node dataset evaluation.} 
Firstly, we evaluate against the motes dataset which has an ambient dimension $d=48$ and is comprised out of $N=7712$ total feature vectors thus its total size being $\R^{48 \times 7712}$.
This dataset is highly periodic in nature and has a larger lower/higher value deltas when compared to the other datasets. 
The initial rank used for all algorithms was $r=10$. 
The errors are plotted in logarithmic scale and can be seen in~\autoref{fig:eval_real_humid} and we can clearly see that $\FPCAC$ outperforms the competing algorithms while being within the expected $\FPCAC_{\text{(low)}}$ \& $\FPCAC_\text{(high)}$ bounds.

\paragraph{Light readings sensor node dataset evaluation.} 
Secondly, we evaluate against a motes dataset that has an ambient dimension $d=48$ and is comprised out of  $N=7712$ feature vectors thus making its total size $\R^{48 \times 7712}$. 
It contains mote light readings can be characterised as a much more volatile dataset when compared to the Humidity one as it contains much more frequent and rapid value changes while also having the highest value delta of all mote datasets evaluated. 
Again, as with Humidity dataset we used an initial seed rank $r=10$ while keeping the rest of the parameters as described  above, the errors over time for all algorithms is shown in~\autoref{fig:eval_real_light} plotted logarithmic scale.
As before, $\FPCAC$ outperforms the other algorithms while being again 
within the expected $\FPCAC_{\text{(low)}}$ \& $\FPCAC_{\text{(high)}}$ bounds.

\paragraph{Temperature readings sensor node dataset evaluation.} 
The third motes dataset we evaluate contains temperature readings from the mote sensors and has an 
ambient dimension $d=56$ containing $N=7712$ feature vectors thus making its total size $\R^{56 \times 7712}$. 
Like the humidity dataset the temperature readings exhibit periodicity in their value change and rarely have spikes. 
As previously we used a seed rank of $r=20$ and the rest of the parameters as described in the synthetic comparison above, the errors over time for all algorithms is shown in~\autoref{fig:eval_real_temp} plotted in logarithmic scale. 
It is again evident that $\FPCAC$ outperforms the other algorithms while being within the $\FPCAC_{\text{(low)}}$ \& $\FPCAC_{\text{(high)}}$ bounds.

\paragraph{Voltage readings sensor node dataset evaluation.} 
Finally, the fourth and final motes dataset we consider has an ambient dimension of $d=46$ contains $N=7712$ feature vectors thus making its size $\R^{46 \times 7712}$. 
Similar to the Light dataset this is an contains very frequent value changes, has large value delta which can be expected during operation of the nodes due to various reasons (one being duty cycling). 
As with the previous datasets we use a seed rank of $r=10$ and leave the rest of the parameters as described previously. 
Finally, the errors over time for all algorithms is shown in~\autoref{fig:eval_real_volt} and are plotted in logarithmic scale. 
As expected, $\FPCA$ here outperforms the competing algorithms while being within the required error bounds.

\subsubsection{MNIST}

To evaluate more concretely the performance of our algorithm in a streaming setting and how the errors evolve over time rather than just reporting the result we plot the logarithm of the frobenious norm error over time while using the MNIST dataset used in the main manuscript.
From our results as can be seen from~\autoref{fig:eval_real_mnist} $\FPCA$ consistently outperforms competing methods and exhibits state of the art performance throughout.

\subsubsection{Wine}

The final real dataset we consider to evaluate and plot the evolving errors is the (red) Wine quality dataset, in which we also used in the main manuscript albeit, as with MNIST, we only reported the resulting subspace quality error.
Again, as we can see from~\autoref{fig:eval_real_wine_red} $\FPCA$ performs again remarkably, besting all other methods in this test as well.

\subsubsection{Real dataset evaluation remarks}

One strength of our algorithm is that it has the flexibility of not having its incremental updates to be bounded by the ambient dimension $d$ - {\it i.e.} its merges. 
This is especially true when operating on a memory limited scenario as the minimum number of feature vectors that need to be kept has to be a multiple of the ambient dimension $d$ in order to provide their theoretical guarantees (such as in~\citep{mitliagkas2014streaming}).
Moreover, in the case of having an adversarial spectrum (\textit{e.g.} $\alpha>1$), energy thresholding can quickly overestimates the number of required principal components, unless a forgetting factor is used, but at the cost of approximation quality and robustness as it can be seen through our experiments. 
Notably, in a number of runs SP ended up with linearly dependent columns in the generated subspace and failed to complete. 
This is an inherent limitation of Gram-Schmidt orthonormalisation procedure used in the reference implementation and substituting it with a more robust one (such as $\QR$) decreased its efficiency throughout our experiments.

\begin{figure*}[ht]
    \centering
    \begin{subfigure}{.48\textwidth}
        \centering
        \includegraphics
        [scale=.25]
        {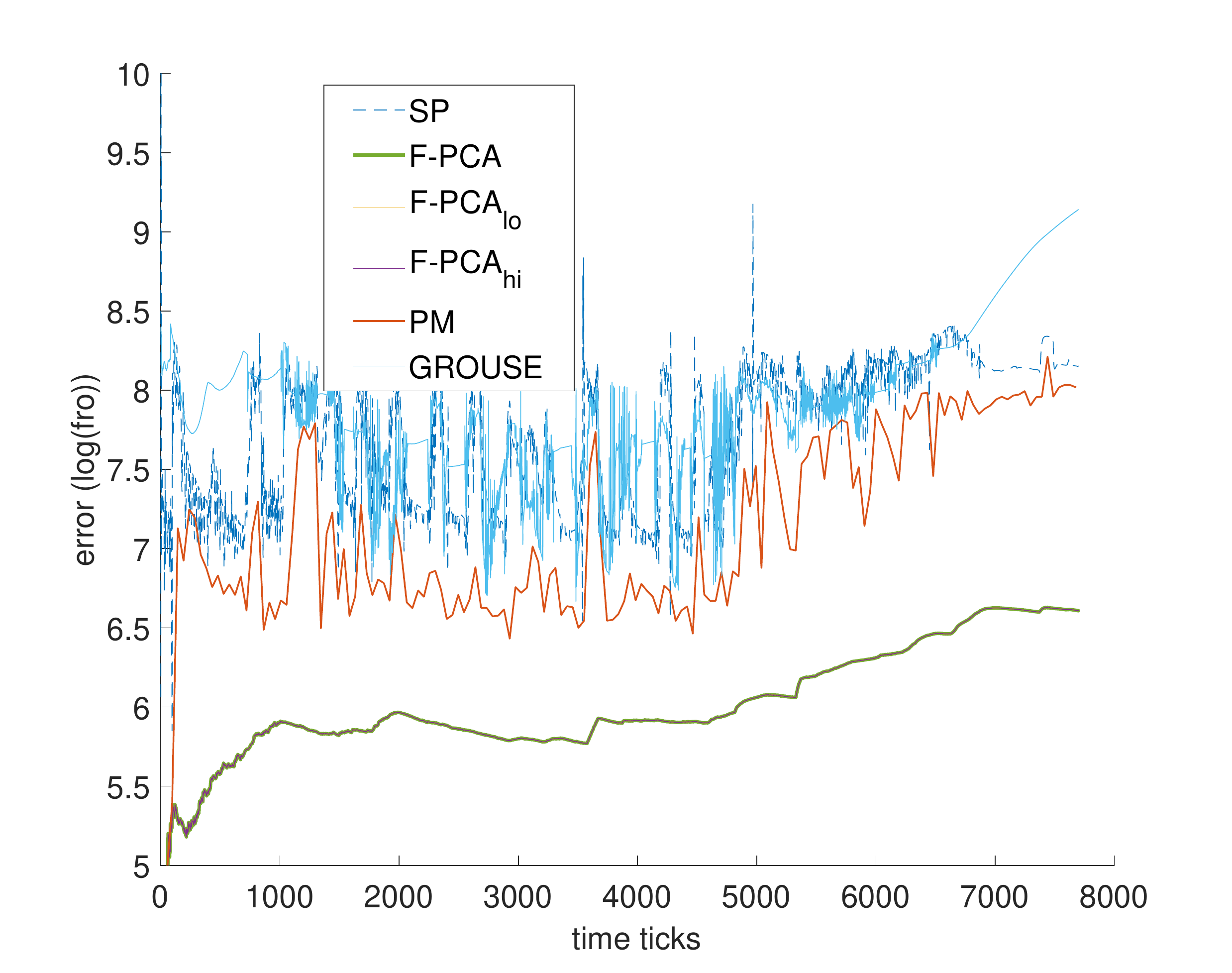}
        \caption{Humidity.}
        \label{fig:eval_real_humid}
    \end{subfigure}
    \begin{subfigure}{.48\textwidth}
        \centering
        \includegraphics
        [scale=.25]
        {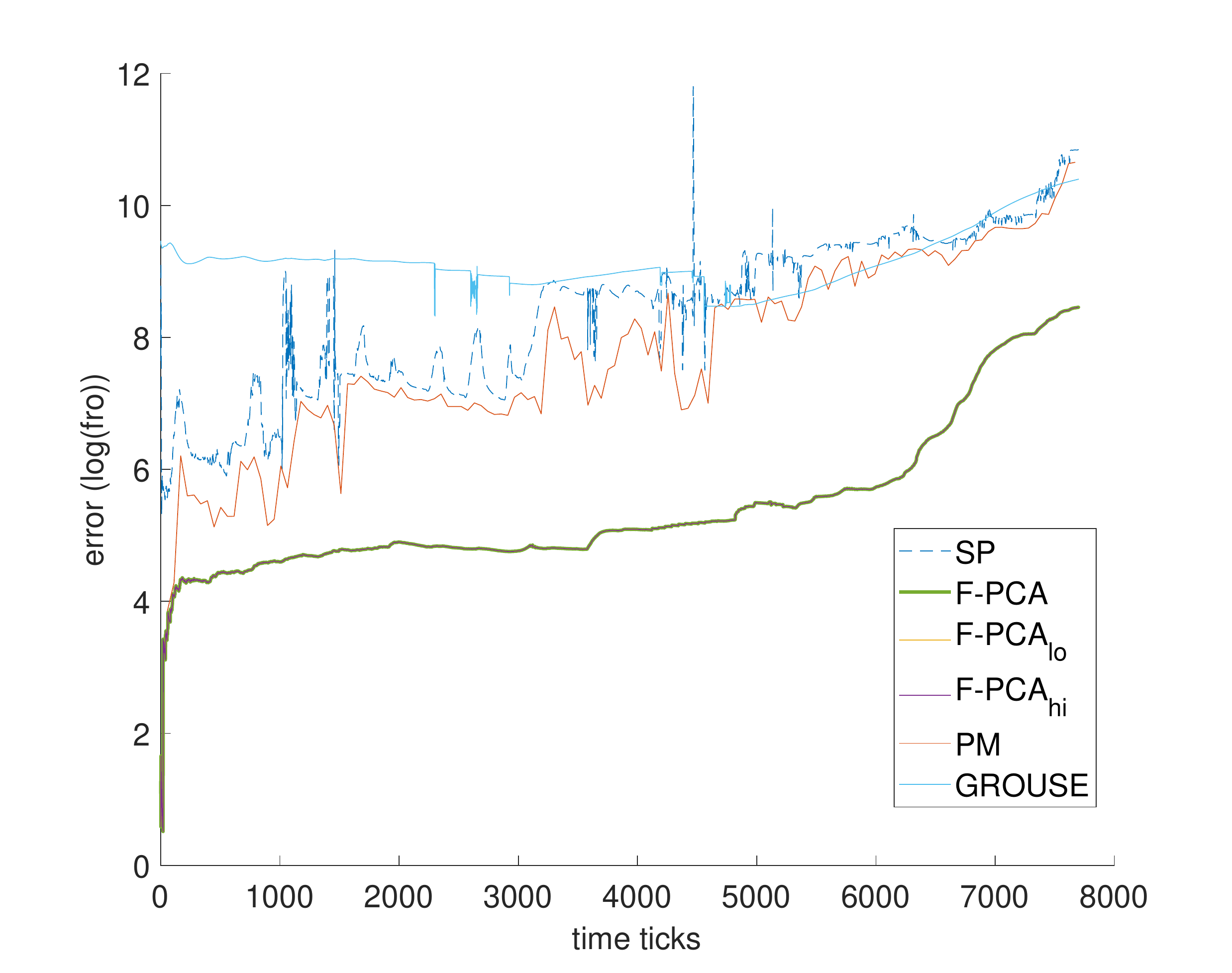}
        \caption{Temperature.}
        \label{fig:eval_real_temp}
    \end{subfigure}
    \begin{subfigure}{.48\textwidth}
        \centering
        \includegraphics
        [scale=.25]
        {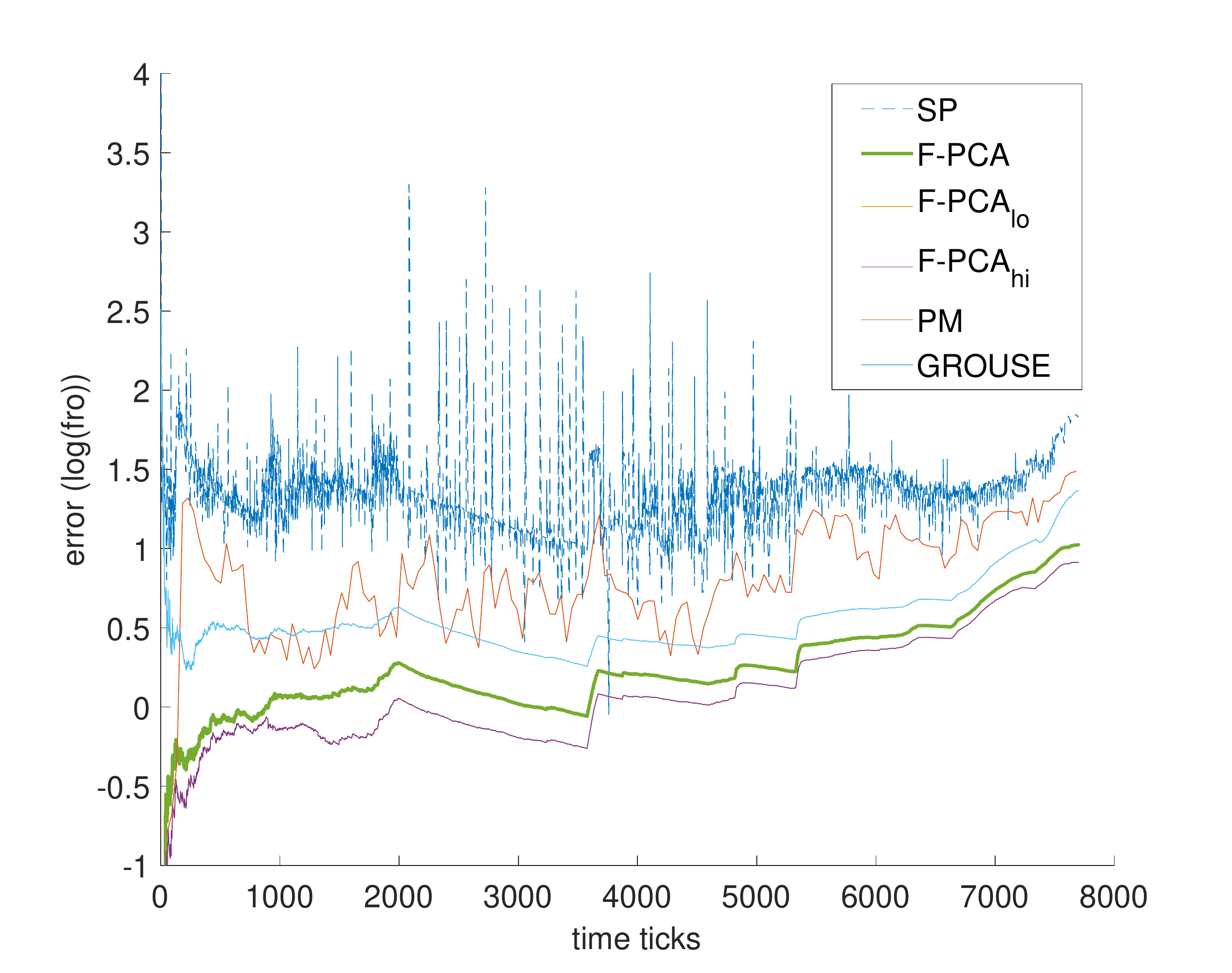}
        \caption{Volt.}
        \label{fig:eval_real_volt}
    \end{subfigure}
    \begin{subfigure}{.48\textwidth}
        \centering
        \includegraphics
        [scale=.25]
        {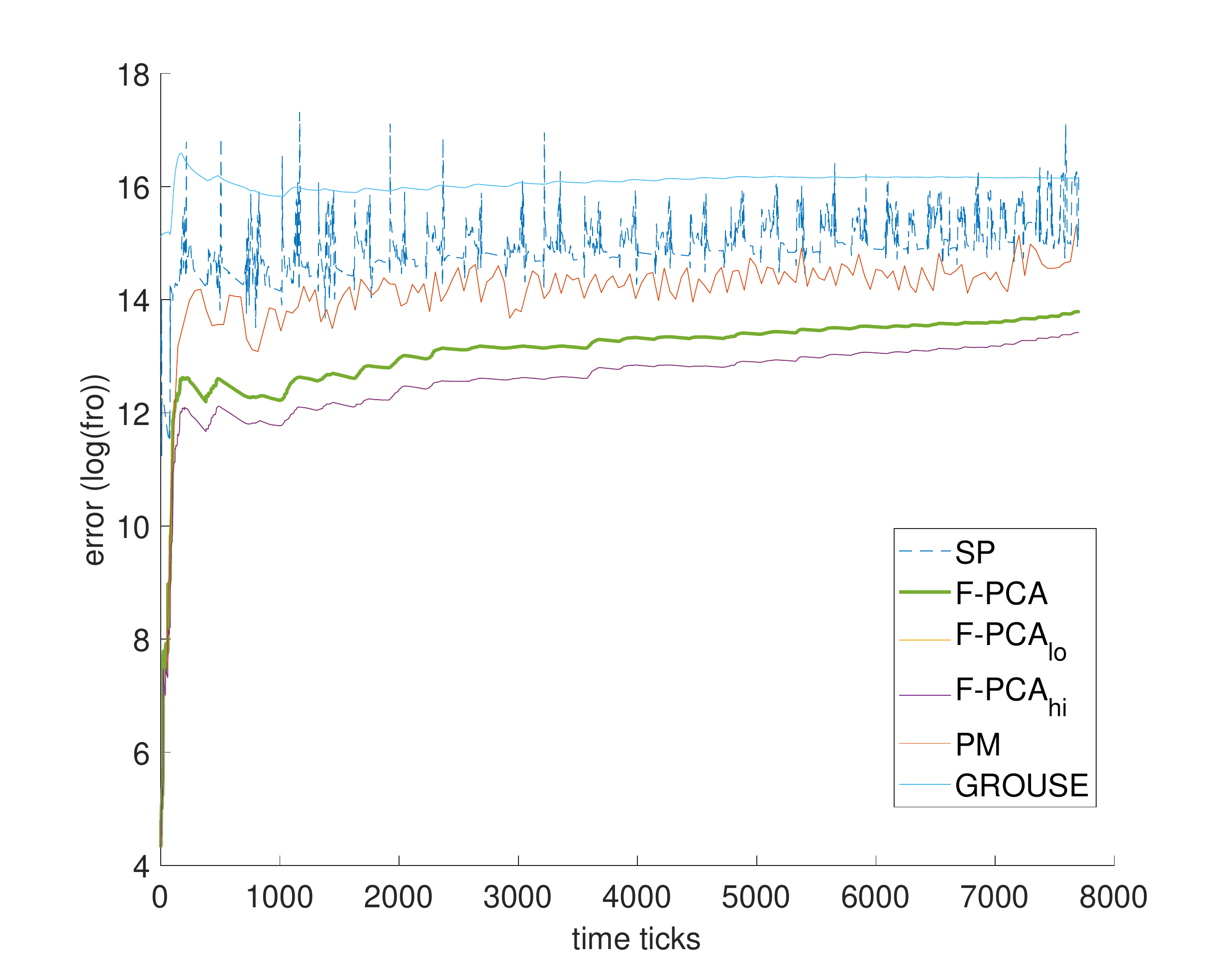}
        \caption{Light.}
        \label{fig:eval_real_light}
    \end{subfigure}
    \begin{subfigure}{.48\textwidth}
        \centering
        \includegraphics
        [scale=.25]
        {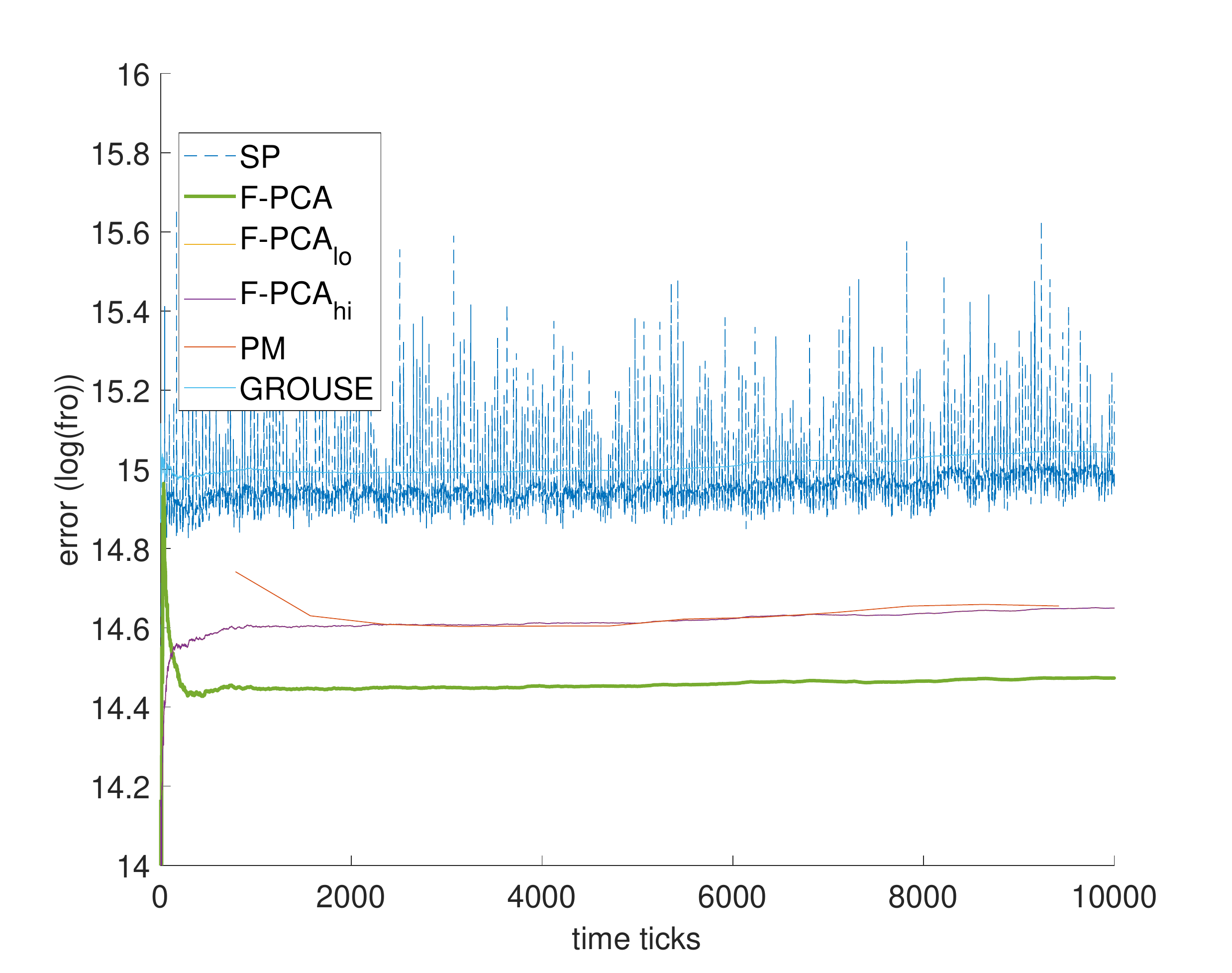}
        \caption{MNIST.}
        \label{fig:eval_real_mnist}
    \end{subfigure}
    \begin{subfigure}{.48\textwidth}
        \centering
        \includegraphics
        [scale=.25]
        {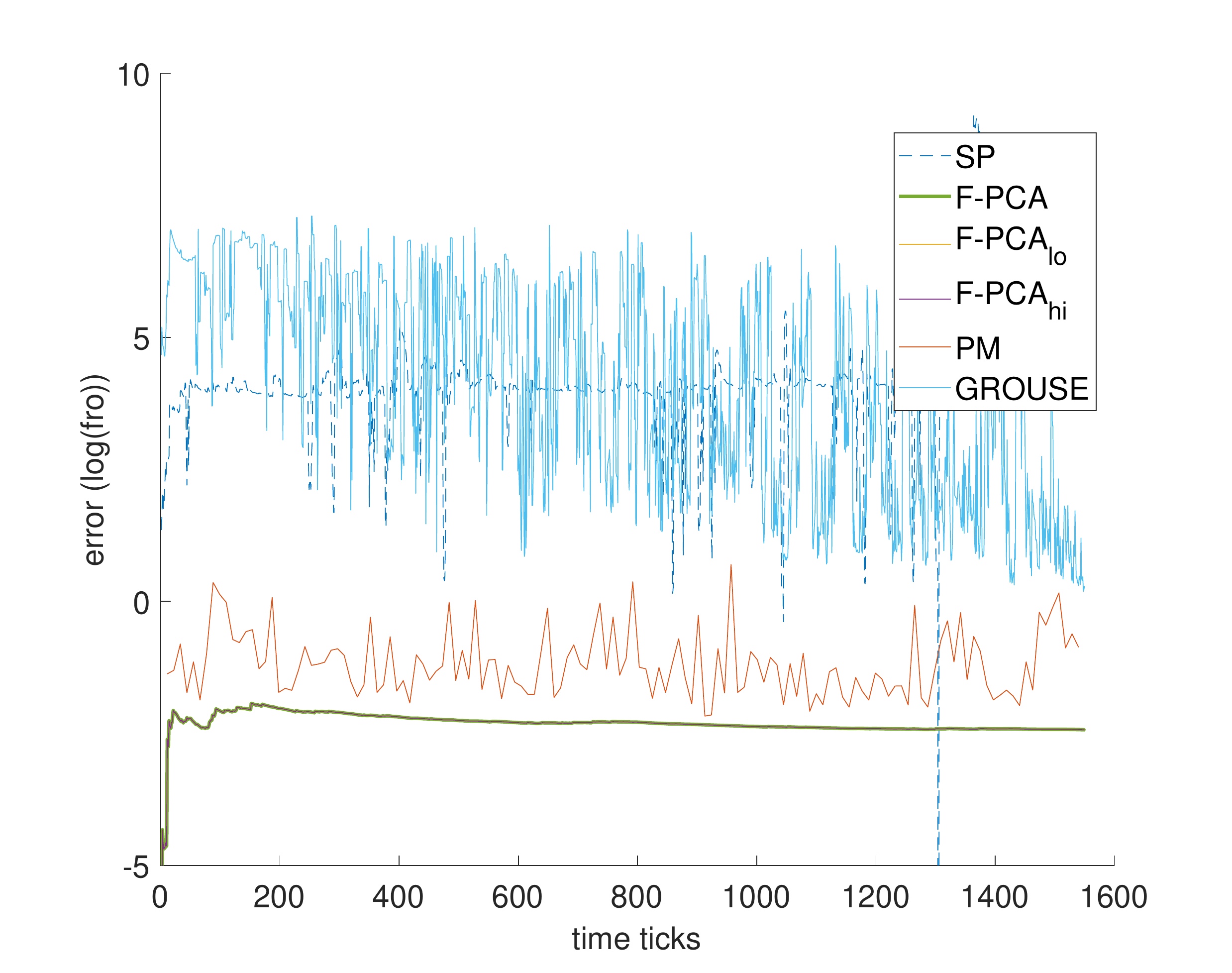}
        \caption{(red) Wine Quality.}
        \label{fig:eval_real_wine_red}
    \end{subfigure}
    \caption{Comparisons against the Motes dataset containing Humidity (\cref{fig:eval_real_humid}), Temperature (\cref{fig:eval_real_temp}), Volt (\cref{fig:eval_real_volt}), and Light (\cref{fig:eval_real_light}) datasets with respect to the Frobenious norm error over time; further, we compare the same error over time for the MNIST (\cref{fig:eval_real_mnist}) and (red) Wine quality (\cref{fig:eval_real_wine_red}) datasets. We compare against SPIRIT (SP), $\FPCAC$, non-adaptive $\FPCAC$ (low/high bounds), PM, \& GROUSE; Frequent directions was excluded due to exploding errors.}
    \label{fig:more_experiments_real}
\end{figure*}

\vfill
\pagebreak

\clearpage

\subsection{Differential Privacy}
\label{apx:diff_eval_details}

Due to spacing limitation we refrained from showing the projections using a variety of differential privacy budgets for the evaluated datasets; in this section we will show how the projections behave for two additional DP budgets, namely for: $\varepsilon\in\{0.6, 1\}$ and $\delta=0.1$ for both datasets. 
The projections for MNIST can be seen in~\autoref{fig:dp_appx_mnist_eval}; the quality of the projections produced by $\FPCA$ appear to be \textit{closer} to the offline ones~\autoref{fig:dp_appx_offline_pca_1_mnist} than the ones produced by $\modsulq$ for both DP budgets considered.

\begin{figure*}[htb!]
    \centering
    \begin{subfigure}{.90\textwidth}
        \centering
        \includegraphics
        [scale=.3]
        {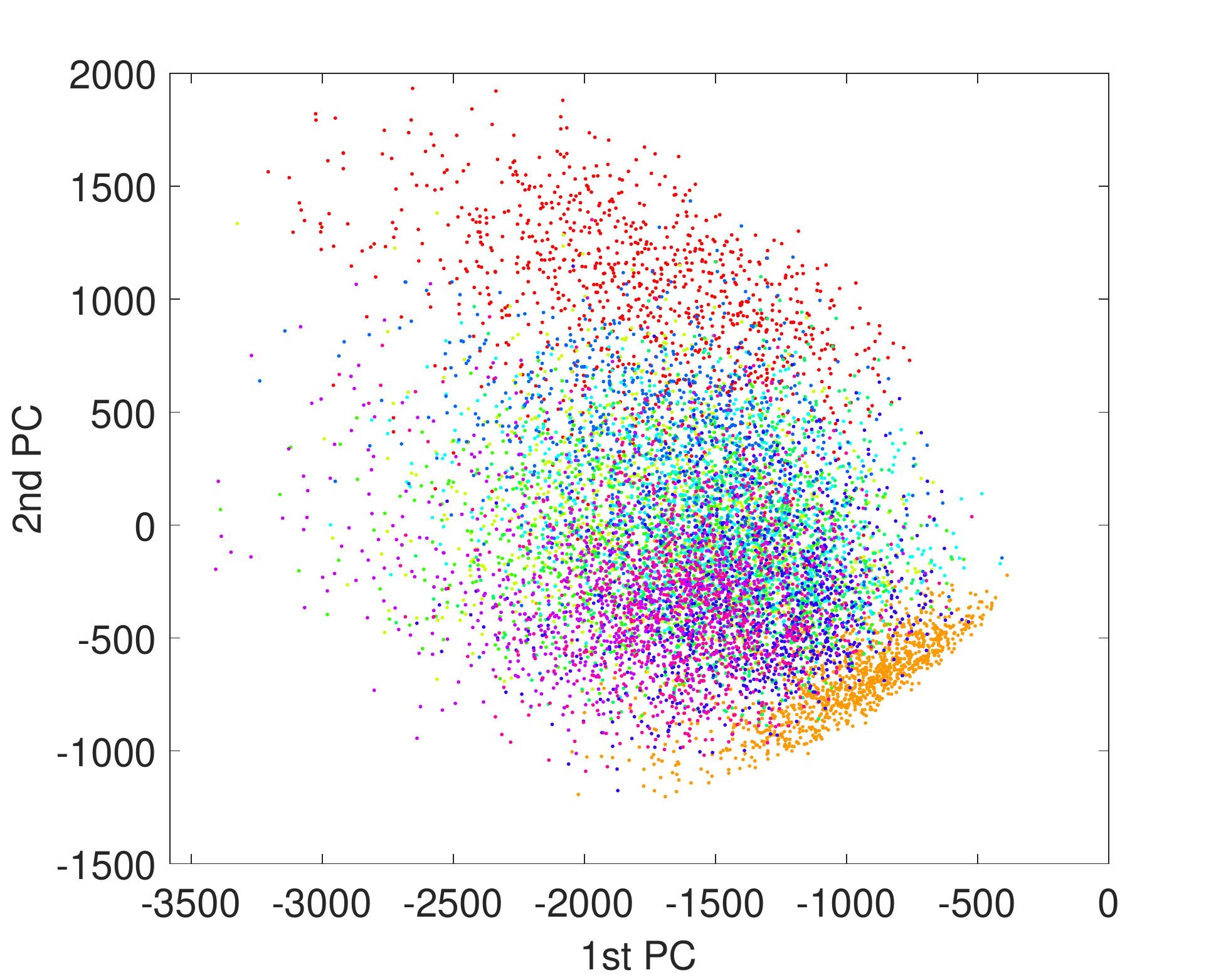}
        \caption{Offline.}
        \label{fig:dp_appx_offline_pca_1_mnist}
    \end{subfigure}
    \begin{subfigure}{.48\textwidth}
        \centering
        \includegraphics
        [scale=.3]
        {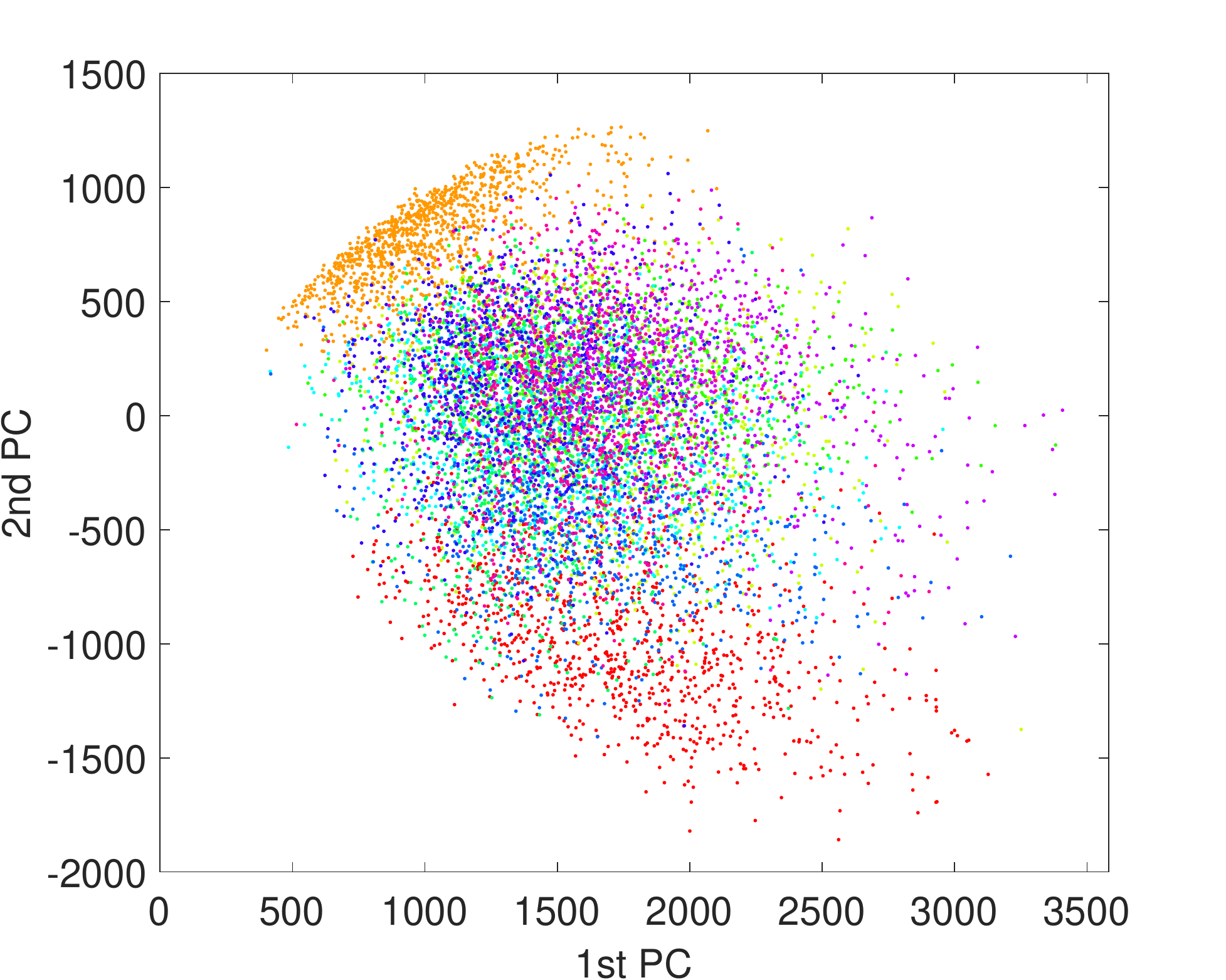}
        \caption{$\FPCAC$ (with masks), $(\varepsilon, \delta)=(0.6, 0.1)$.}
        \label{fig:dp_appx_fpca_0_60_mnist}
    \end{subfigure}
    \begin{subfigure}{.48\textwidth}
        \centering
        \includegraphics
        [scale=.3]
        {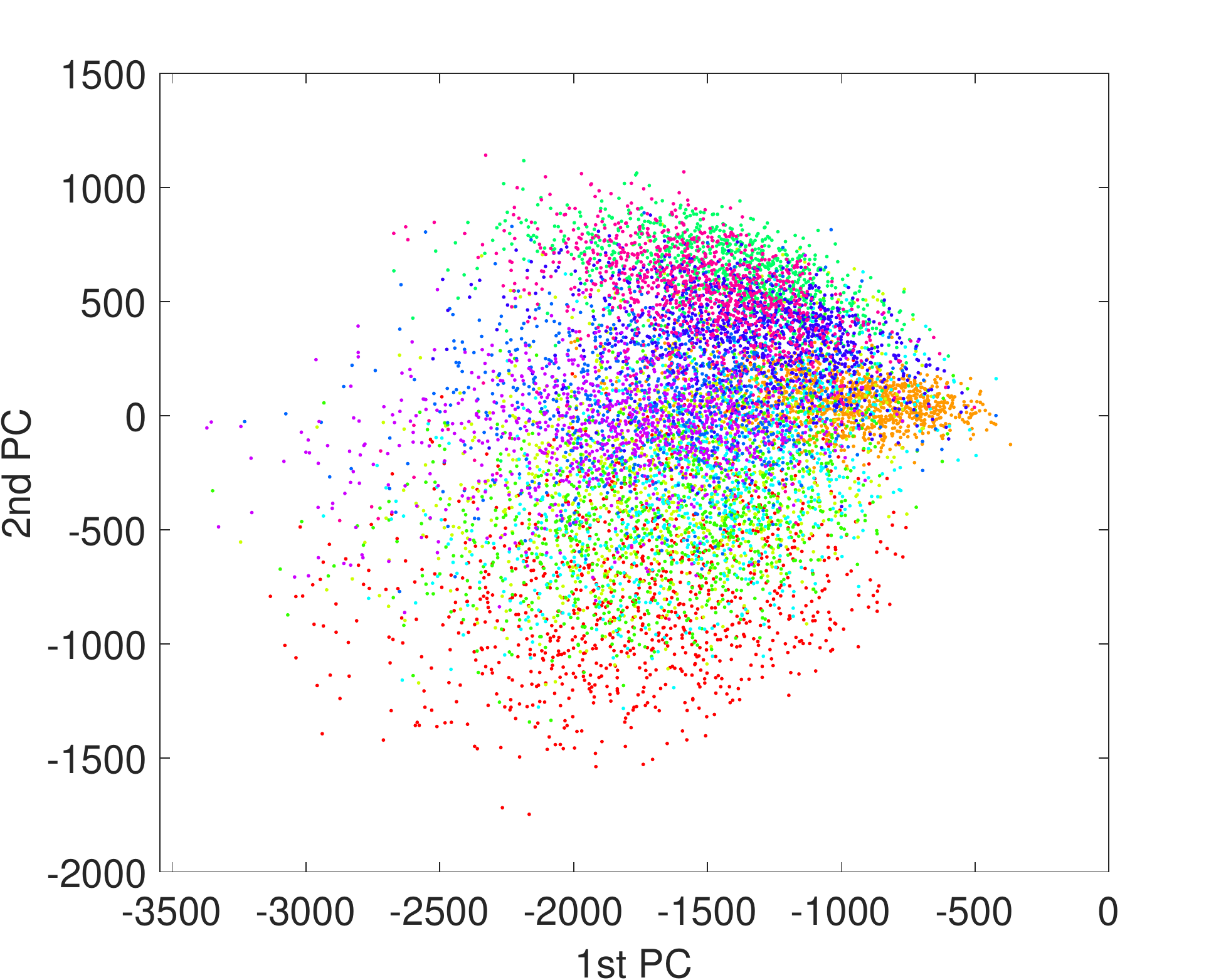}
        \caption{$\modsulq$, $(\varepsilon, \delta)=(0.6, 0.1)$.}
        \label{fig:dp_appx_modsulq_0_60_mnist}
    \end{subfigure}
    \begin{subfigure}{.48\textwidth}
        \centering
        \includegraphics
        [scale=.3]
        {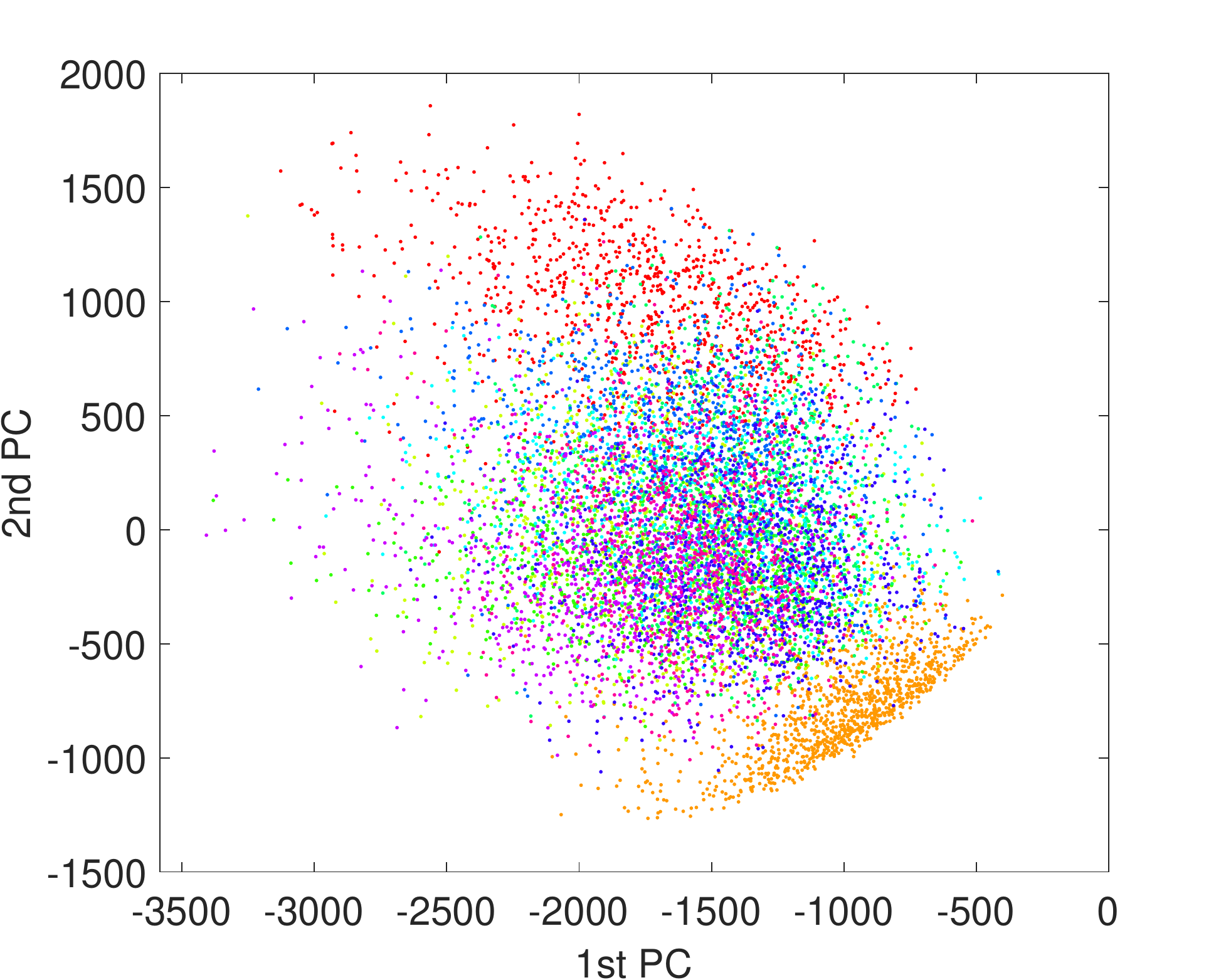}
        \caption{$\FPCAC$ (with masks), $(\varepsilon, \delta)=(1, 0.1)$.}
        \label{fig:dp_appx_fpca_1_mnist}
    \end{subfigure}
    \begin{subfigure}{.48\textwidth}
        \centering
        \includegraphics
        [scale=.3]
        {figs/dp-appx/mnist/e_0_60_delta_0_1_mod_sulq.pdf}
        \caption{$\modsulq$, $(\varepsilon, \delta)=(1, 0.1)$.}
        \label{fig:dp_appx_modsulq_1_mnist}
    \end{subfigure}
    \caption{MNIST projections using different differential privacy budgets, at the top (\cref{fig:dp_appx_offline_pca_1_mnist}) is the full rank PCA while on the left column is $\FPCA$ with perturbation masks and on the right column $\modsulq$ using DP budget of $\varepsilon\in\{0.6, 1\}$ and $\delta=0.1$ while starting from a recovery rank of $6$. Note here that $\FPCA$ exhibits remarkable performance producing higher quality projections than $\modsulq$ in both cases.}
    \label{fig:dp_appx_mnist_eval}
\end{figure*}

However, on the Wine quality dataset projections seen in~\autoref{fig:dp_appx_wine_eval} it seems that $\modsulq$ can produce projection that are \textit{closer} to the offline ones than $\FPCA$ but not too far apart. 
Notably, this can be attributed to the higher sample complexity required by $\FPCA$ as it is an inherently \textit{streaming} method and the (red) Wine dataset is \textit{considerably} smaller than MNIST.

\begin{figure*}[htb!]
    \centering
    \begin{subfigure}{.90\textwidth}
        \centering
        \includegraphics
        [scale=.3]
        {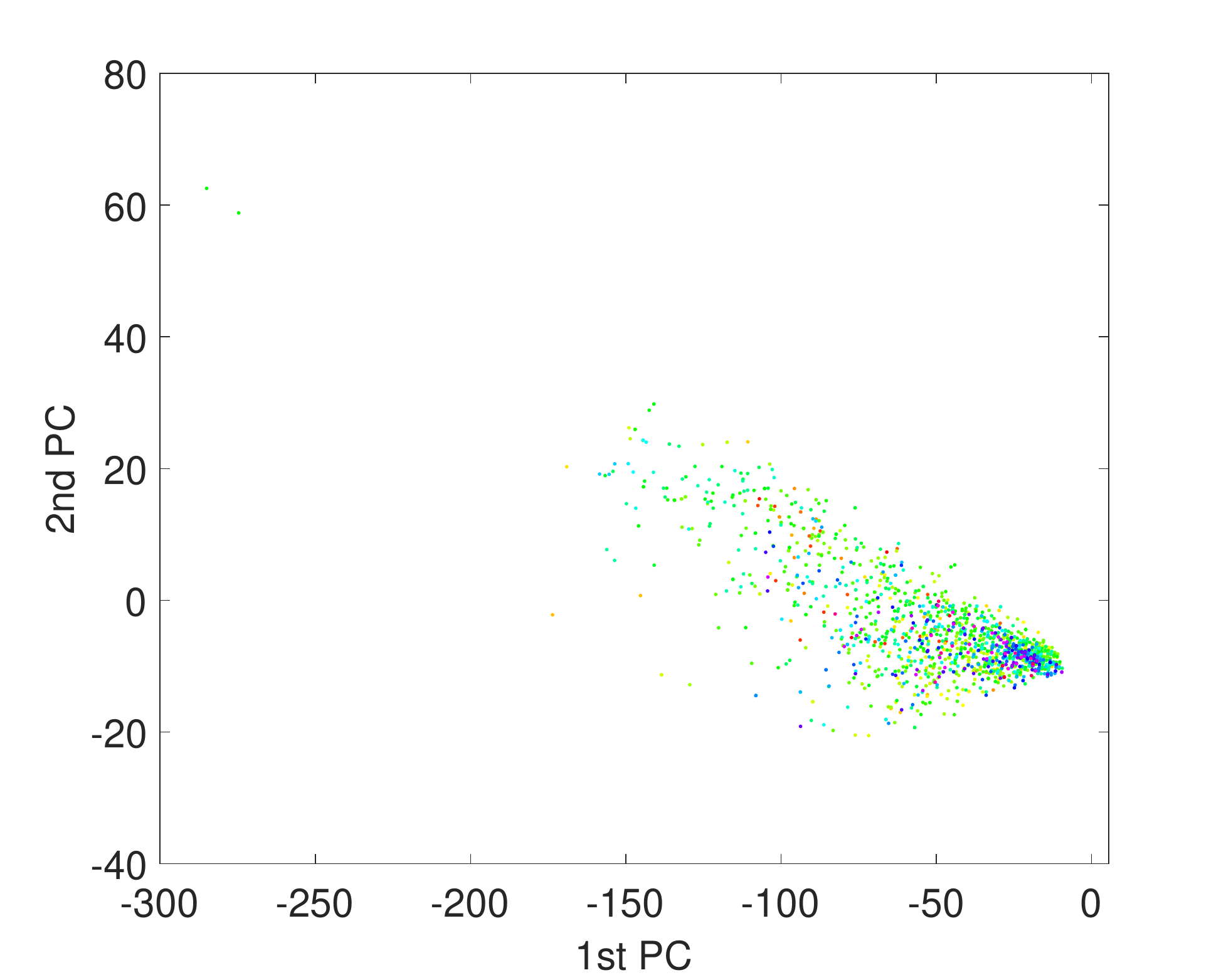}
        \caption{Offline.}
        \label{fig:dp_appx_offline_pca_1_wine}
    \end{subfigure}
    \begin{subfigure}{.48\textwidth}
        \centering
        \includegraphics
        [scale=.3]
        {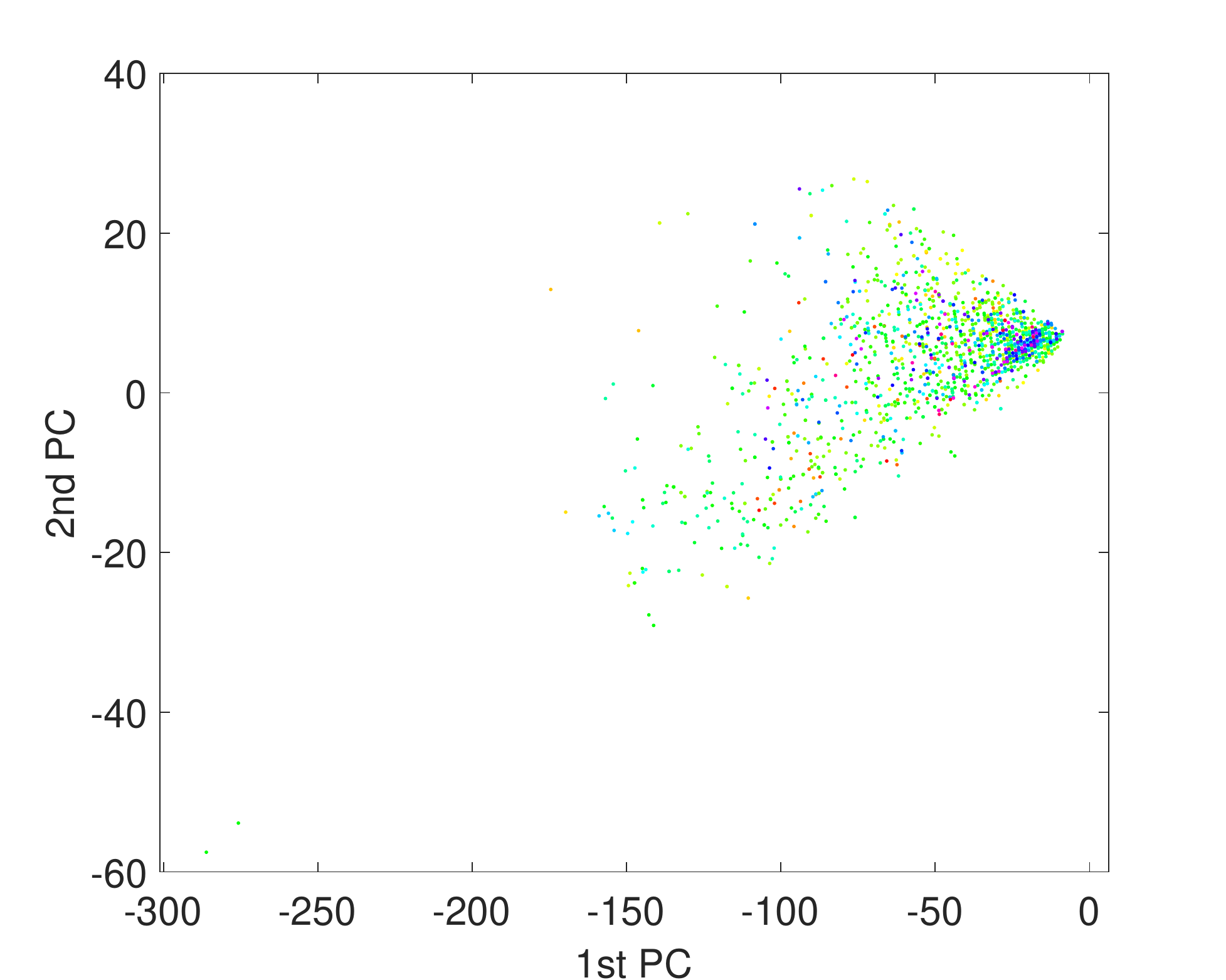}
        \caption{$\FPCAC$ (with masks), $(\varepsilon, \delta)=(0.6, 0.1)$.}
        \label{fig:dp_appx_fpca_0_60_wine}
    \end{subfigure}
    \begin{subfigure}{.48\textwidth}
        \centering
        \includegraphics
        [scale=.3]
        {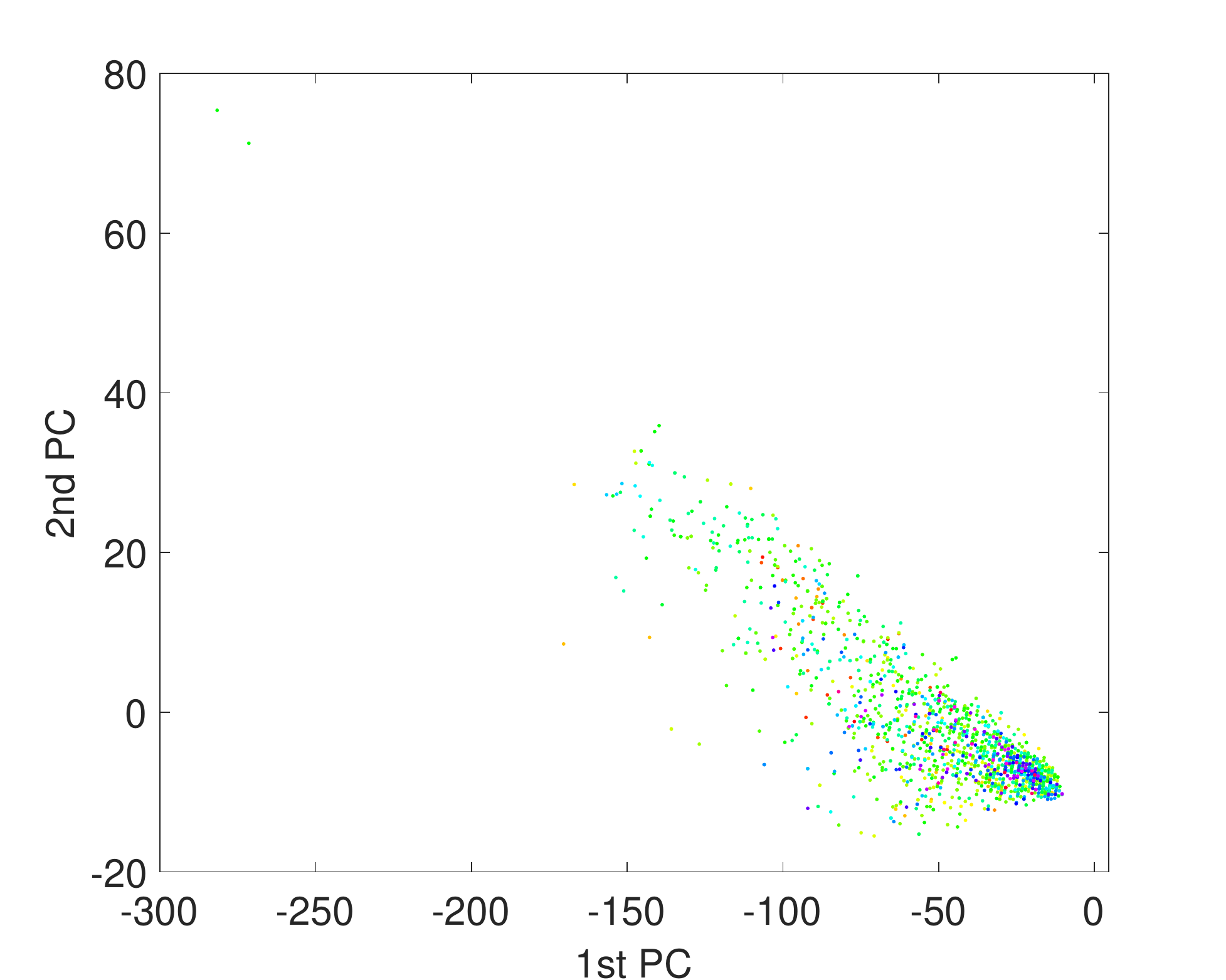}
        \caption{$\modsulq$, $(\varepsilon, \delta)=(0.6, 0.1)$.}
        \label{fig:dp_appx_modsulq_0_60_wine}
    \end{subfigure}
    \begin{subfigure}{.48\textwidth}
        \centering
        \includegraphics
        [scale=.3]
        {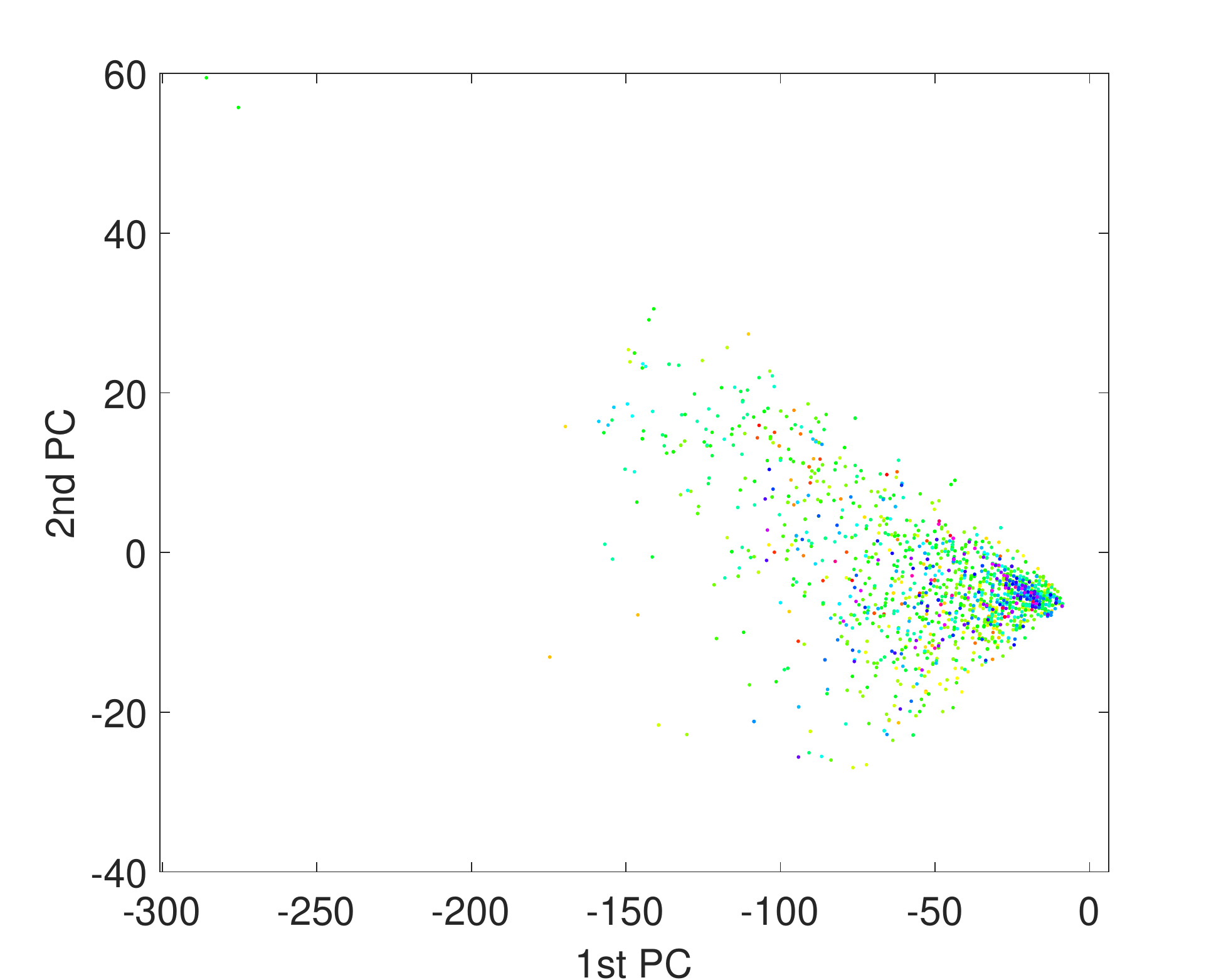}
        \caption{$\FPCAC$ (with masks), $(\varepsilon, \delta)=(1, 0.1)$.}
        \label{fig:dp_appx_fpca_1_wine}
    \end{subfigure}
    \begin{subfigure}{.48\textwidth}
        \centering
        \includegraphics
        [scale=.3]
        {figs/dp-appx/wine/e_0_60_delta_0_1_mod_sulq.pdf}
        \caption{$\modsulq$, $(\varepsilon, \delta)=(1, 0.1)$.}
        \label{fig:dp_appx_modsulq_1_wine}
    \end{subfigure}
    \caption{(red) Wine quality projections using different differential privacy budgets, at the top (\cref{fig:dp_appx_offline_pca_1_wine}) is the full rank PCA while on the left column is $\FPCA$ with perturbation masks and on the right column $\modsulq$ using DP budget of $\varepsilon\in\{0.6, 1\}$ and $\delta=0.1$ while starting from a recovery rank of $6$. Note here that due to the higher sample complexity requirements of $\FPCA$ the projections appear slighly worse.}
    \label{fig:dp_appx_wine_eval}
\end{figure*}

\vfill

\clearpage

\pagebreak

\subsection{Federated Evaluation}
\label{apx:fed_eval_details}

To provide additional information with respect to the evaluation we also report the amortised execution times per number of workers, as if the workers exceed the number of available compute nodes in our workstation then computation cannot be completed in parallel thus hindering the potential speedup.
In \autoref{fig:fed_compute_apx} we show the amortised total (\cref{fig:fed_apx_amo_total_time}), PCA (\cref{fig:fed_apx_amo_pca_time}), and merge (\cref{fig:fed_apx_amo_merge_time}) times respectively - these results, as in the main text, use $\FPCA$ \textit{without} perturbation masks but a similar result would apply to this case as well.
These results indicate, that in the presence of enough resources, $\FPCA$ exhibits an extremely favourable scalability curve emphasising the practical potential of the method if used in conjunction with thin clients ({\it i.e.} mobile phones).

\begin{figure*}[htb!]
    \centering
    \begin{subfigure}{.31\textwidth}
        \centering
        \includegraphics[scale=0.31]{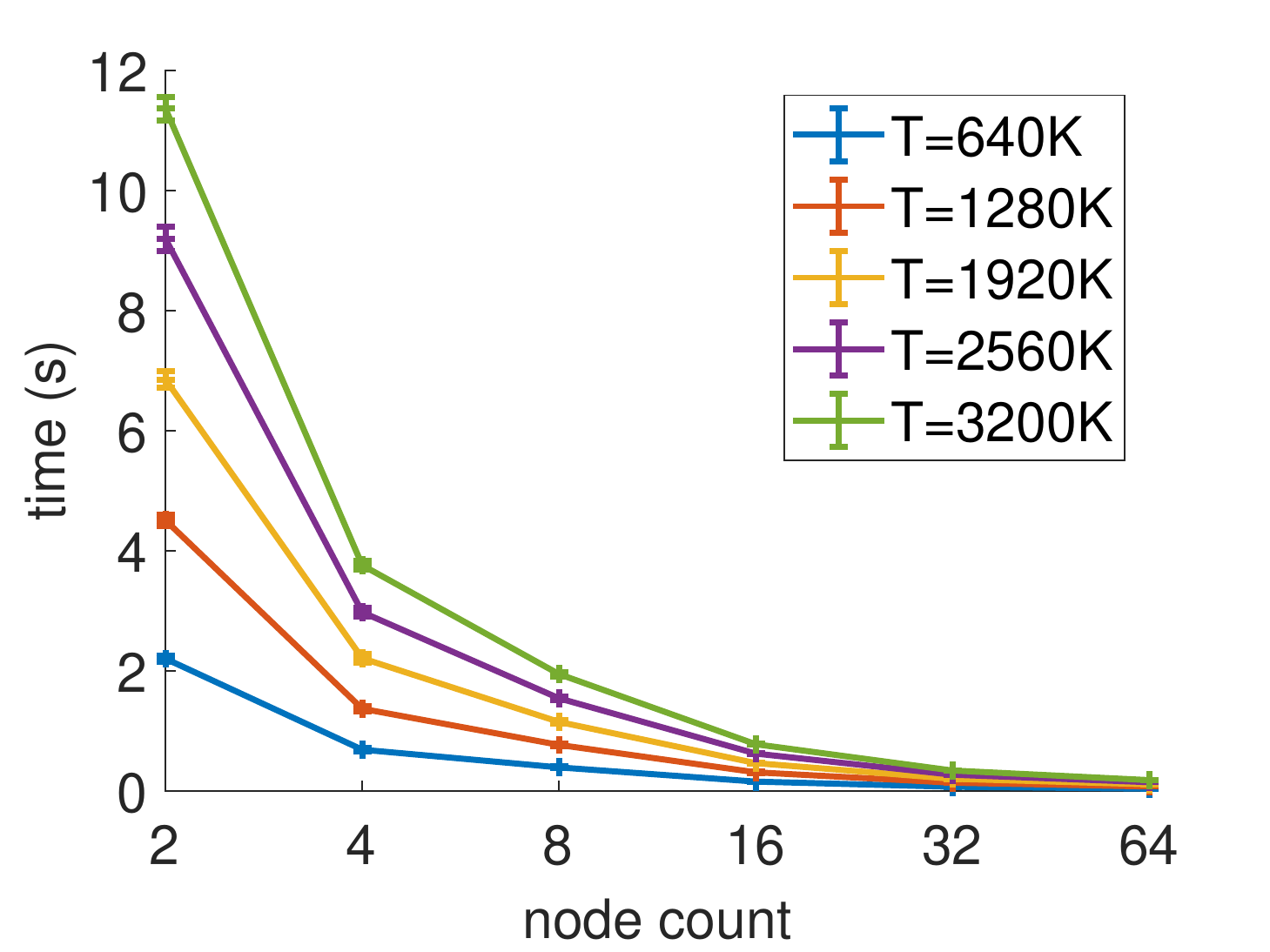}
        \caption{Amortised execution time.}
        \label{fig:fed_apx_amo_total_time}
    \end{subfigure}
    \centering
    \begin{subfigure}{.31\textwidth}
        \centering
        \includegraphics[scale=0.31]{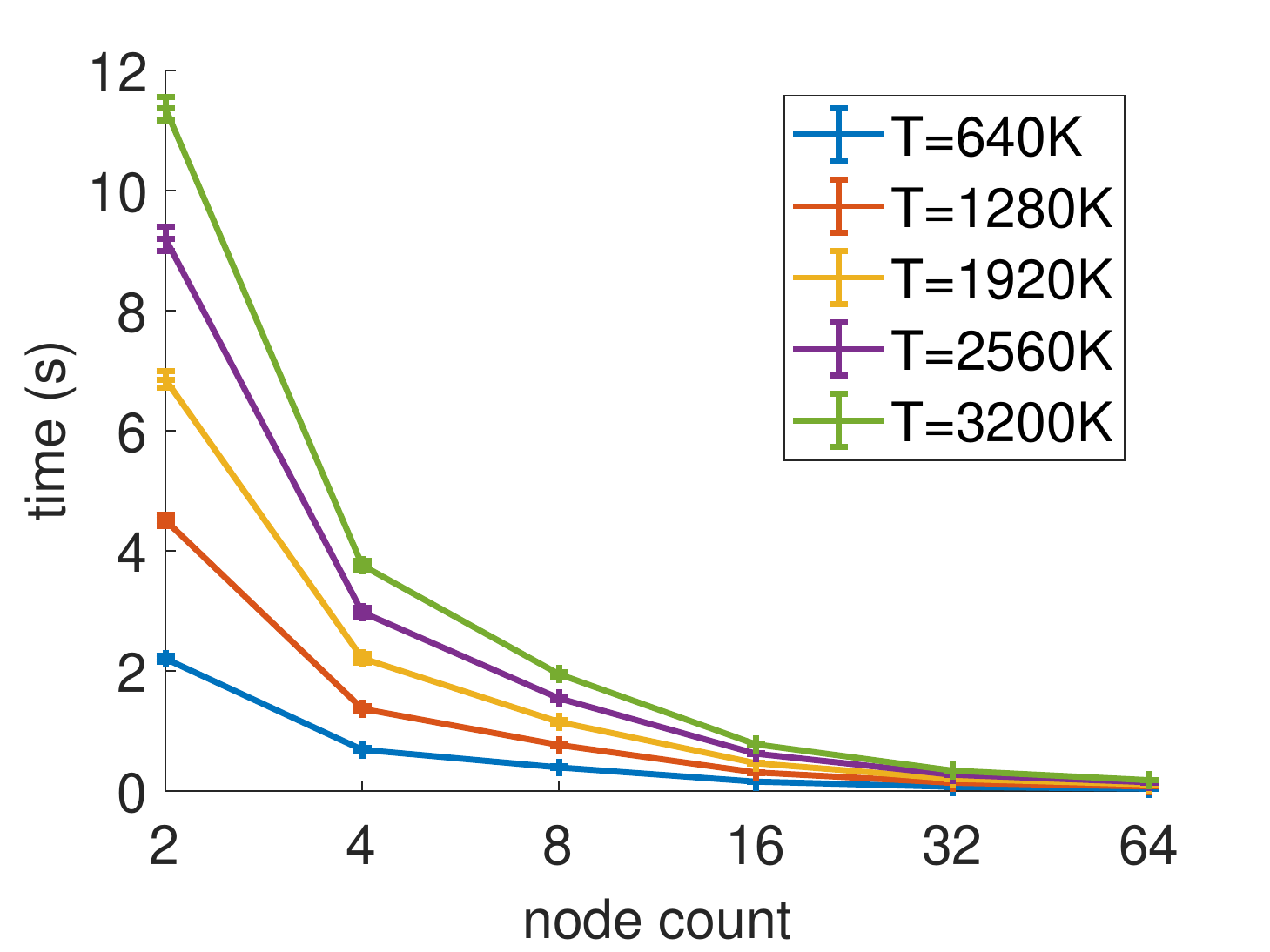}
        \caption{Amortised PCA time.}
        \label{fig:fed_apx_amo_pca_time}
    \end{subfigure}
    \centering
    \begin{subfigure}{.31\textwidth}
        \centering
        \includegraphics[scale=0.31]{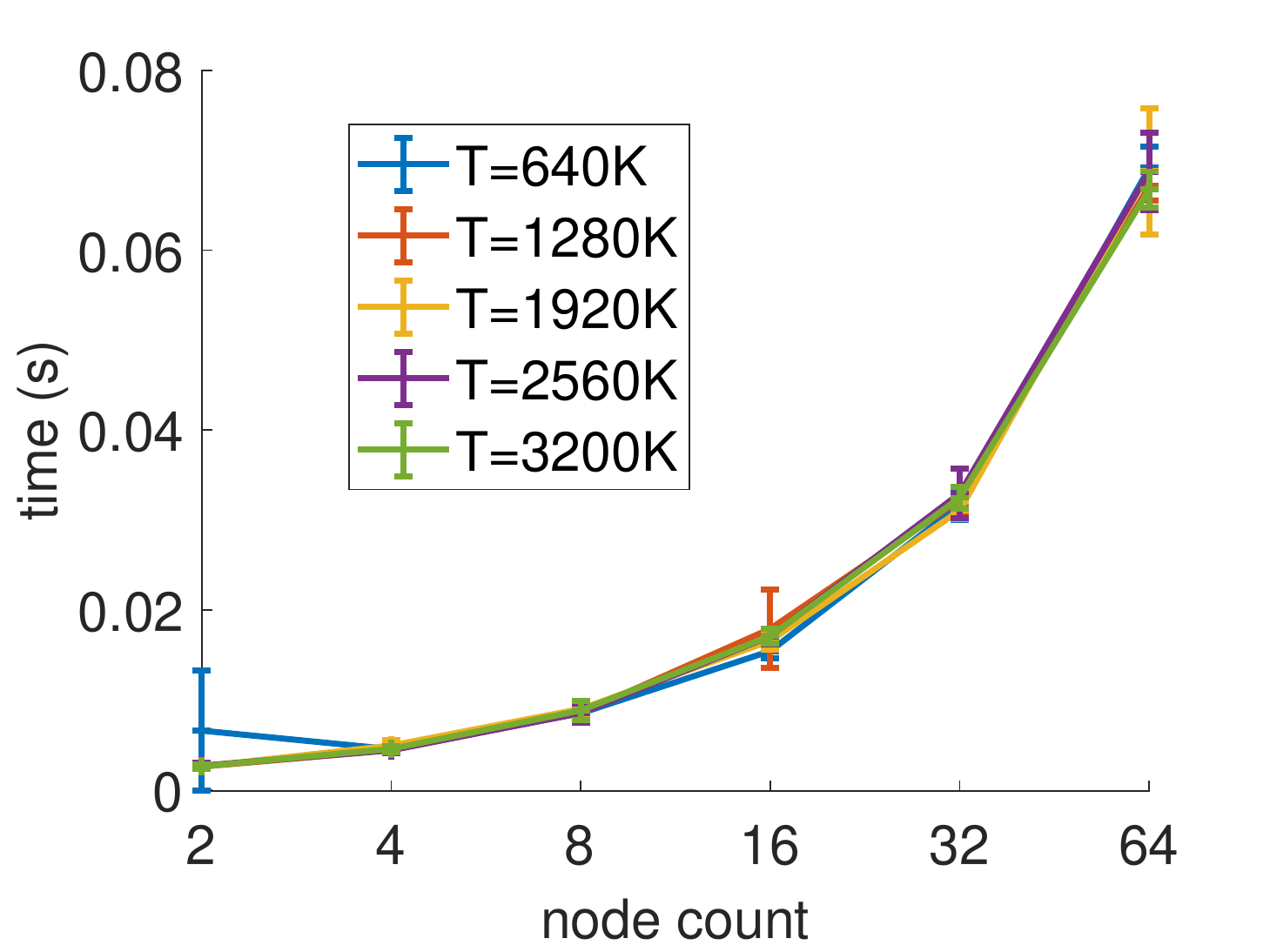}
        \caption{Amortised time spent merging.}
        \label{fig:fed_apx_amo_merge_time}
    \end{subfigure}
    \caption{Amortised execution times for total (\cref{fig:fed_apx_amo_total_time}), PCA (\cref{fig:fed_apx_amo_pca_time}), and merge (\cref{fig:fed_apx_amo_merge_time}) operations respectively.}
    \label{fig:fed_compute_apx}
\end{figure*}

\subsection{Memory Evaluation}
\label{apx:mem_discussion}

We benchmarked each of the methods used against its competitors and found that our $\FPCA$ performed favourably. 
With respect to the experiments, in order to ensure accurate measurements, we started measuring after clearing the previous profiler contents. 
The tool used in all profiling instances was \texttt{MATLAB}'s built-in memory profiler which provides a rough estimate about the memory consumption; however, it has been reported that can cause issues in some instances.

These empirical results support the theoretical claims about the storage optimality of $\FPCAC$.
In terms of average and median memory allocations, $\FPCAC$ is most of the times better than the competitors. 
Naturally, since by design, PM requires the materialisation of larger block sizes it requires more memory than both $\FPCAC$ as well as FD. 
Moreover, GROUSE, in its reference implementation requires the instantiation of the whole matrix again; this is because the reference version of GROUSE is expected to run on a subset of a sparse matrix which is copied locally to the function - since in this instance we require the   entirety of the matrix to be allocated and thus results in a large memory overhead. 
An improved, more efficient GROUSE implementation would likely solve this particular issue. 
Concluding, we note that although $\FPCA{}$ when using perturbation masks consumes slightly more memory, this is due to the inherent added for supporting differential privacy; however, this cost appears to be in line with our $\mathcal{O}(db)$ memory bound and not quadratic with respect to $d$, as with competing algorithms.

\begin{table*}[htb!]
    \centering
    \small
     \caption{Average / median memory allocations (Kb) for a set of real-world datasets.}%
\scalebox{0.85}{
    \begin{tabular}{l|llll}
          & Humidity & Light & Voltage & Temperature  \\
         \hline
         $\FPCAC$ (with mask) & $166.57$ / $81.23$ Kb & $172.00$ / $ 99.17$ Kb & $289.02$ / $143.79$ Kb & $257.00$ / $ 195.30$ Kb \\
         \hline
         $\FPCAC$ (no mask) & $ \textbf{138.11}$ / $\textbf{58.99}$ Kb & $\textbf{104.00}$ / $ \textbf{76.03}$ Kb & $ 204.58$ / $ \textbf{23.47}$ Kb & $ \textbf{187.74}$ / $ \textbf{113.28}$ Kb \\
         PM & $ 905.45$ / $ 666.11$ Kb & $ 685.48$   / $ 685.44$ Kb & $ 649.12$  / $ 644.35$ Kb & $ 657.57$ / $ 668.27$ Kb \\
         GROUSE & $ 2896.61$ / $ 2896.62$ Kb & $ 2896.84$ / $ 2896.62$ Kb & $ 2772.86$ / $ 2772.62$ Kb & $3379.62$ / $ 3376.62$ Kb \\
         FD & $ 162.70$ / $ 117.92$ Kb & $ 170.48$  / $ 127.91$ Kb & $ \textbf{114.46}$ / $ 112.66$ Kb & $ 196.11$ / $ 118.59$ Kb \\
         SP & $476.68$ / $ 405.01$ Kb & $1009.03$ / $508.11$ Kb & $ 348.84$ / $351.98$ Kb & $ 541.56$ / $437.61$ Kb \\
         \hline
    \end{tabular}
}
     \label{tab:mem_results_avg_mem}
\end{table*}

\pagebreak

\subsection{Extended Time-Order Independence Empirical Evaluation}

\label{apx:time_order_independence_experiments}

The figures show the errors for recovery ranks $r$  equal to $5$ (\ref{fig:subspace_time_order_r_5}), $20$ (\ref{fig:subspace_time_order_r_20}), $40$ (\ref{fig:subspace_time_order_r_40}), $60$ (\ref{fig:subspace_time_order_r_60}), and $80$ (\ref{fig:subspace_time_order_r_80}). 
It has to be noted, that legends which are subscripted with $s$ (e.g. ${gr}_{s}$) compare against the $\SVD$ output while the others against its own output of the perturbation against the original $\Y$. 
We remark that when trying a full rank recovery (i.e. $r=100$), SPIRIT failed to complete the full run as it ended up in some instances with linearly dependent columns, while the other methods perform similarly to the previous examples. 

\begin{figure*}[htb!]
    \centering
    \begin{subfigure}{.48\textwidth}
    \centering
    \includegraphics
    [scale=0.28]
    {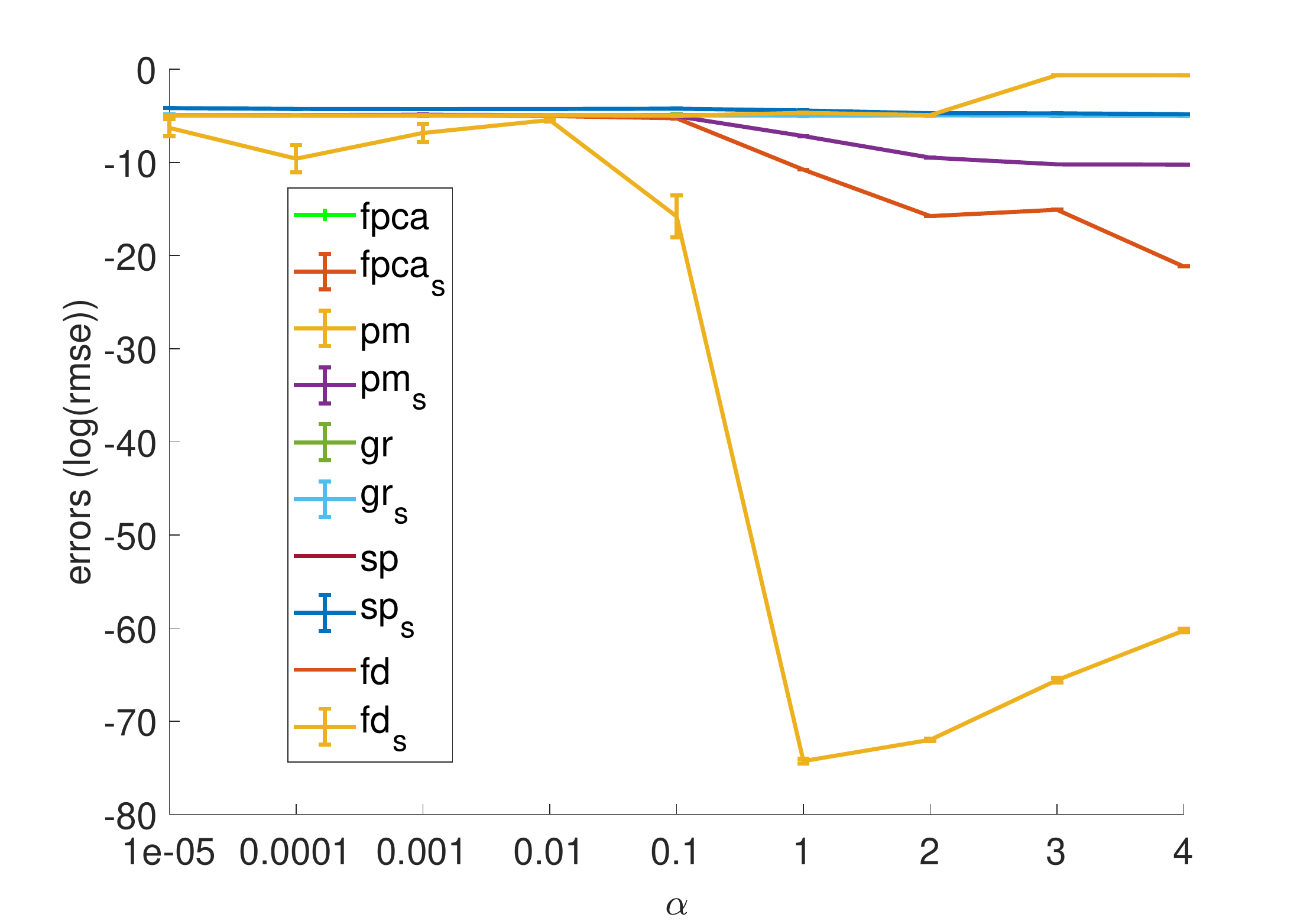}
    \caption{Permutation errors for recovery rank $r=5$.}
    \label{fig:subspace_time_order_r_5}
    \end{subfigure}
    \begin{subfigure}{.48\textwidth}
    \centering
    \includegraphics
    [scale=0.28]
    {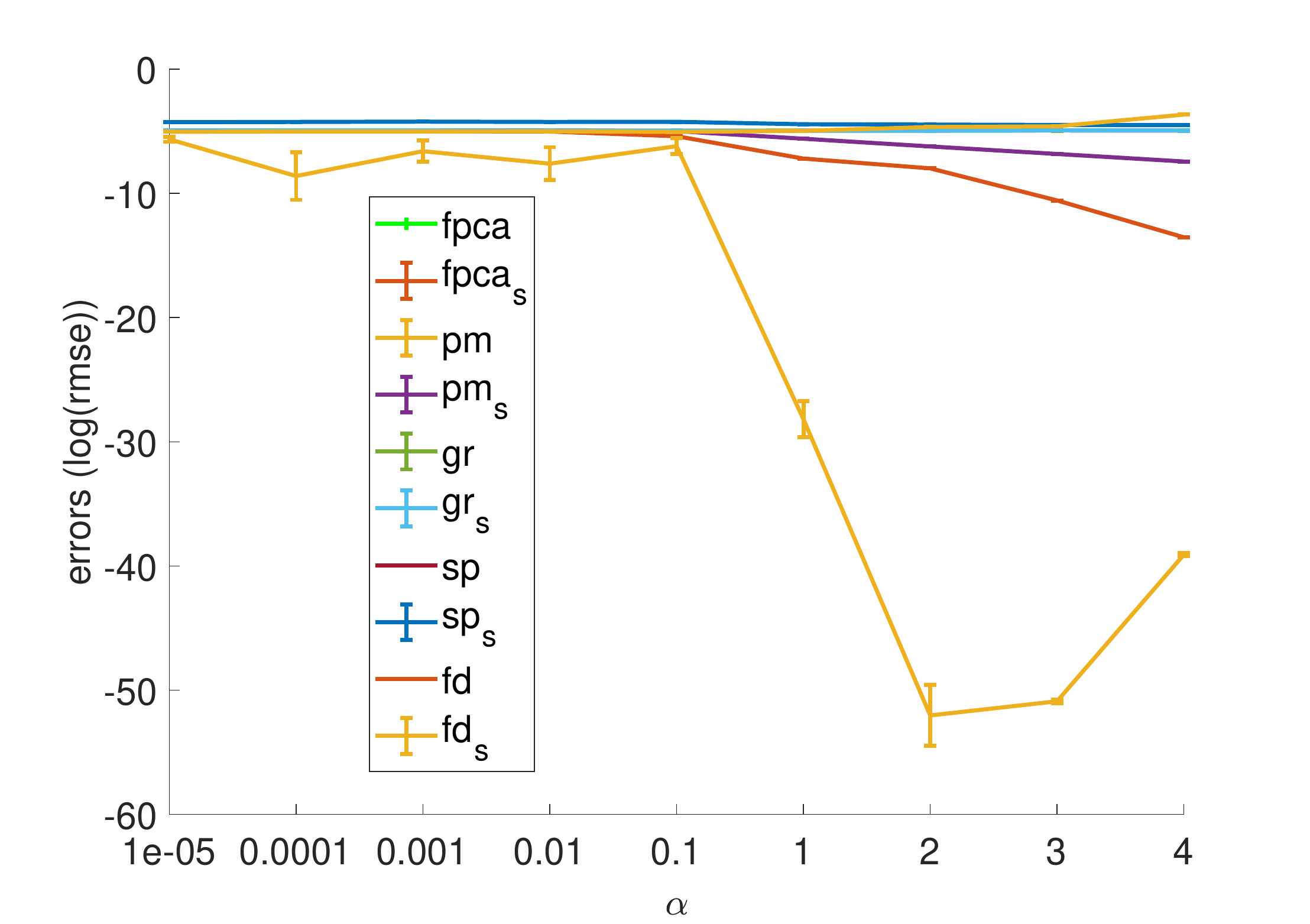}
    \caption{Permutation errors for recovery rank $r=20$.}
    \label{fig:subspace_time_order_r_20}
    \end{subfigure}
    \begin{subfigure}{.48\textwidth}
    \centering
    \includegraphics
    [scale=0.28]
    {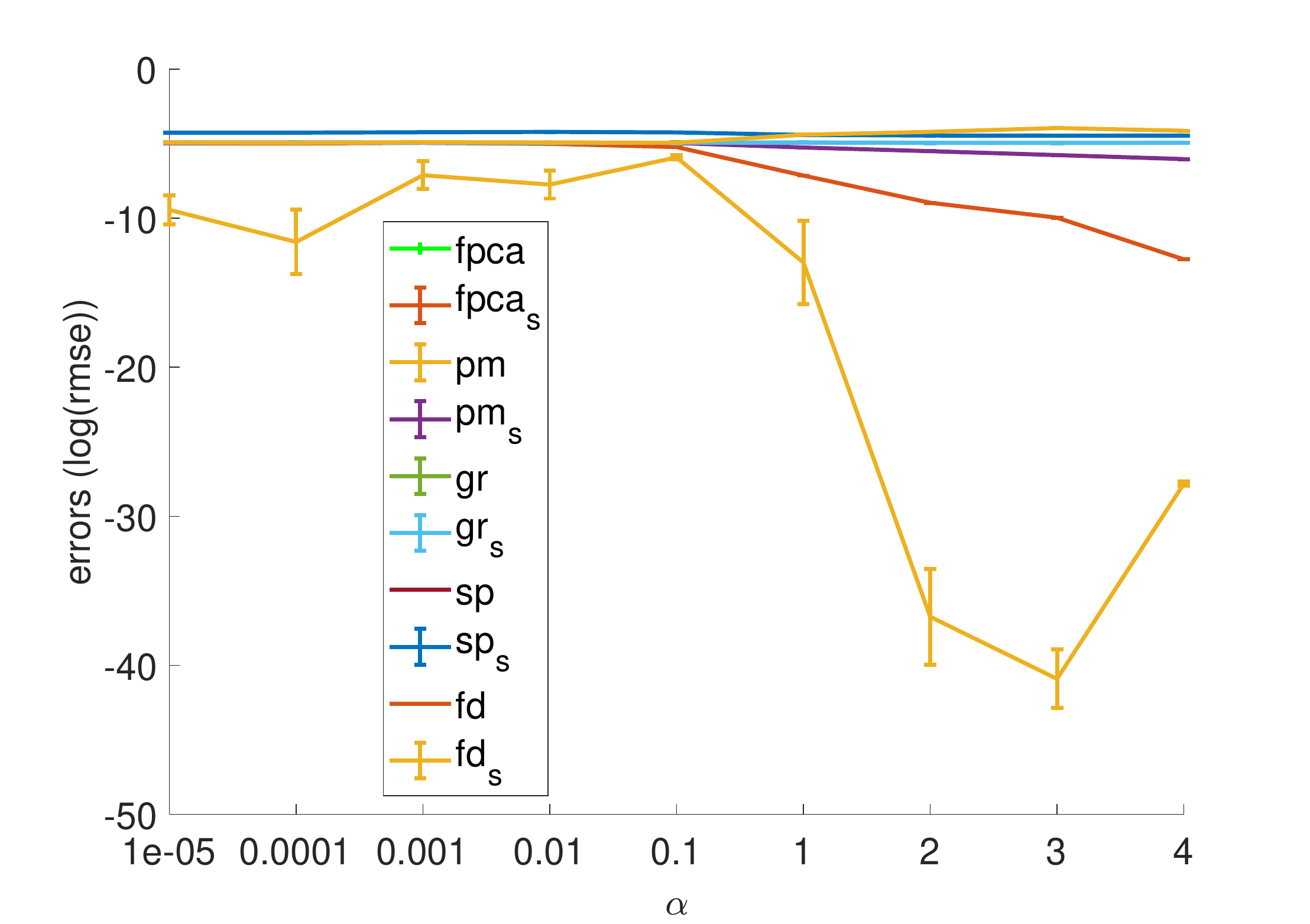}
    \caption{Permutation errors for recovery rank $r=40$.}
    \label{fig:subspace_time_order_r_40}
    \end{subfigure}
    \begin{subfigure}{.48\textwidth}
    \centering
    \includegraphics
    [scale=0.28]
    {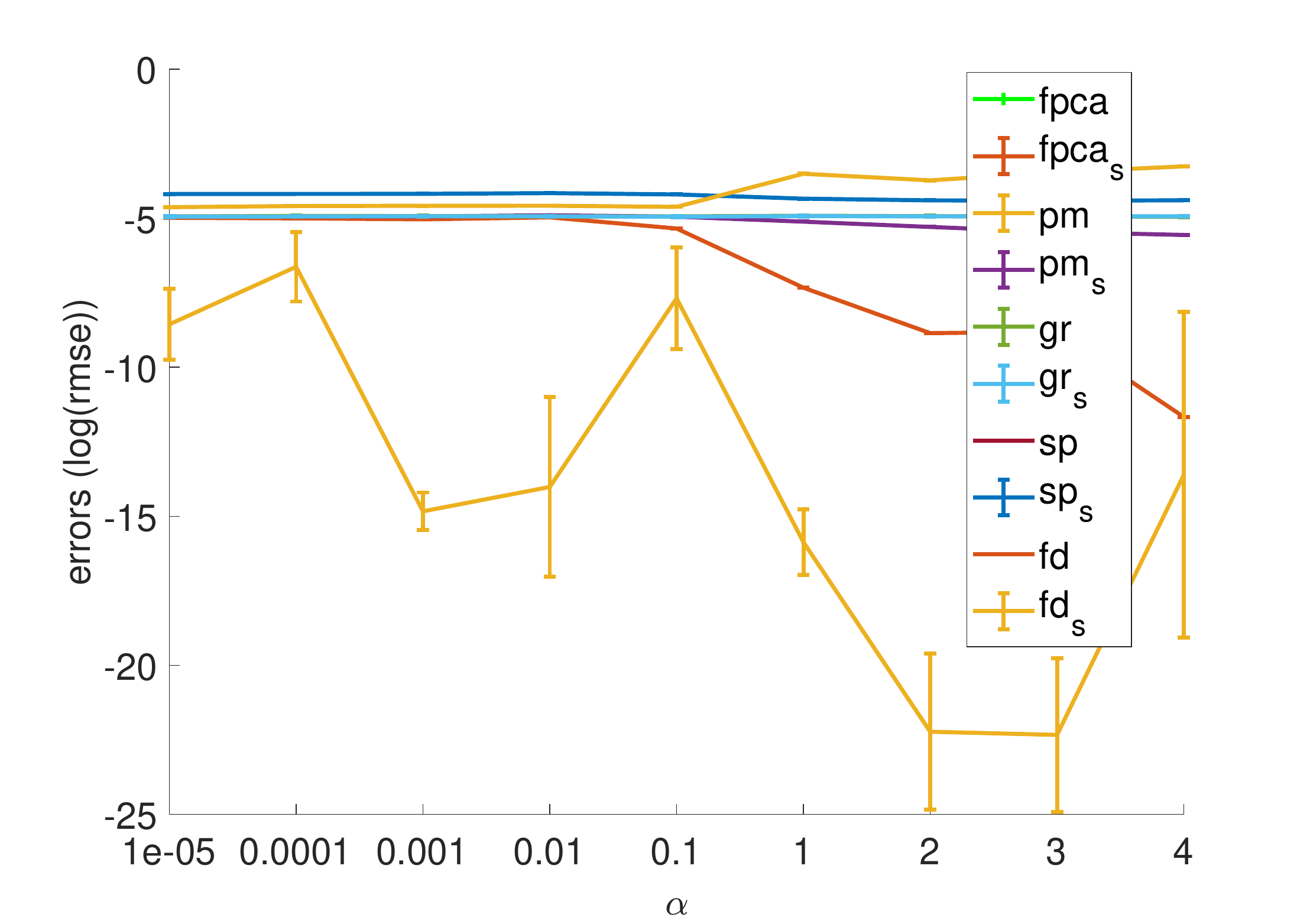}
    \caption{Permutation errors for recovery rank $r=60$.}
    \label{fig:subspace_time_order_r_60}
    \end{subfigure}
    \begin{subfigure}{.48\textwidth}
    \centering
    \includegraphics
    [scale=0.28]
    {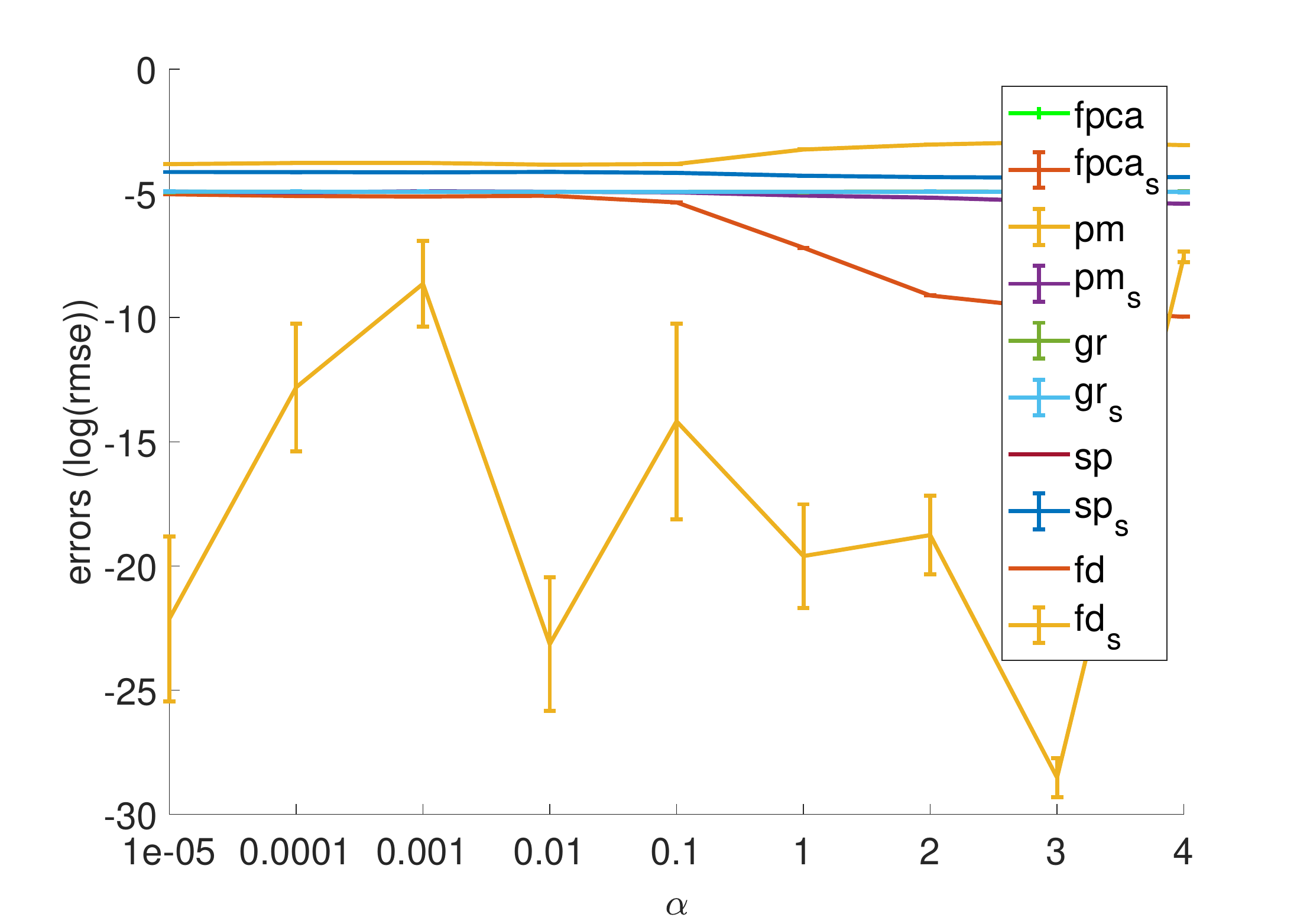}
    \caption{Permutation errors for recovery rank $r=80$.}
    \label{fig:subspace_time_order_r_80}
    \end{subfigure}
    \caption{Mean Subspace errors over $20$ permutations of $Y \in \R^{100 \times 10000}$ for recovery rank $r$ equals $5$ (a), $20$ (b), $40$ (c), $60$ (d), and $80$ (e).}
    \label{fig:subspace_time_order_errors_apx}
\end{figure*}

\vfill

\end{document}